\documentclass[12pt]{article}

\usepackage[onehalfspacing]{setspace}
\usepackage{amsthm}
\usepackage{hyperref}
\usepackage{graphics,graphicx,amssymb,mathtools}
\usepackage{caption}
\usepackage{booktabs}
\usepackage{dsfont}
\usepackage{xcolor}

\usepackage{natbib}

\usepackage{footmisc}

\newtheorem{corollary}{Corollary}

\usepackage{xr}

\usepackage{dsfont}
\usepackage{amssymb}

\newcommand{\uu}{{\mathbf u}}

\newcommand{\YY}{{\mathbf Y}}

\newcommand{\hl}[1]{#1}

\newcommand{\xx}{\mathbf{x}}

\let\hat\widehat
\let\tilde\widetilde

\newtheorem{lem}{Lemma}
\newtheorem{assumption}{Assumption}

\newenvironment{exam1}[1][Normal model.]{\begin{trivlist}
\item[\hskip \labelsep {\bfseries #1}]}{\end{trivlist}}

\newenvironment{exam2}[1][Binomial model.]{\begin{trivlist}
\item[\hskip \labelsep {\bfseries #1}]}{\end{trivlist}}

\newenvironment{exam3}[1][Poisson model.]{\begin{trivlist}
\item[\hskip \labelsep {\bfseries #1}]}{\end{trivlist}}

\newenvironment{CJMLE}[1][Constrained joint maximum likelihood estimator (CJMLE).]{\begin{trivlist}
\item[\hskip \labelsep {\bfseries #1}]}{\end{trivlist}}

\newenvironment{NBE}[1][Nuclear-norm-based estimator (NBE).]{\begin{trivlist}
\item[\hskip \labelsep {\bfseries #1}]}{\end{trivlist}}

\newenvironment{OthE}[1][Other estimators.]{\begin{trivlist}
\item[\hskip \labelsep {\bfseries #1}]}{\end{trivlist}}

\newenvironment{Meth2}[1][Method 2' (A variation of Method 2).]{\begin{trivlist}
\item[\hskip \labelsep {\bfseries #1}]}{\end{trivlist}}

\newtheorem{meth}{Method}
\newtheorem{rmk}{Remark}

\usepackage{dsfont}

\let\hat\widehat
\let\tilde\widetilde

\usepackage{algorithm}
\usepackage{algorithmic}
\usepackage{array}
\newcolumntype{L}[1]{>{\raggedright\let\newline\\\arraybackslash\hspace{0pt}}p{#1}}
\newcolumntype{C}[1]{>{\centering\let\newline\\\arraybackslash\hspace{0pt}}p{#1}}
\newcolumntype{R}[1]{>{\raggedleft\let\newline\\\arraybackslash\hspace{0pt}}p{#1}}

\usepackage{geometry}
\newtheorem{theorem}{Theorem}

\newcommand{\ZZ}{{\mathbf Z}}

\newcommand{\AAA}{{\mathbf A}}
\newcommand{\aaaa}{{\mathbf a}}

\newcommand{\RR}{{\mathbf R}}
\newcommand{\gG}{{\mathcal{G}}}

\newcommand{\ttt}{{\boldsymbol \theta}}
\newcommand{\TTT}{{\boldsymbol \Theta}}

\newcommand{\argmax}{\operatornamewithlimits{arg\,max}}
\newcommand{\argmin}{\operatornamewithlimits{arg\,min}}
\newcommand{\MM}{{\mathbf M}}
\newcommand{\UU}{{\mathbf U}}
\newcommand{\VV}{{\mathbf V}}
\newcommand{\DD}{{\mathbf D}}
\newcommand{\vvv}{{\mathbf v}}
\newcommand{\II}{{\mathcal I}}
\newcommand{\PP}{{\mathbf P}}

\newcommand{\SSS}{{\mathbf S}}

\newcommand{\XX}{{\mathbf X}}
\newcommand{\TT}{{\mathbf T}}

\newcommand{\bt}{{\mathbf t}}

\newcommand{\BB}{{\mathbf B}}

\newcommand{\bR}{\mathbb{R}}

\newcommand{\OOO}{{\boldsymbol \Omega}}
\newcommand{\PPP}{{\boldsymbol \Pi}}

\newcommand{\dOmegai}{{\mbox{diag}(\boldsymbol{\Omega}_{i\cdot})}}

\newcommand{\dOmegaj}{{\text{diag}(\boldsymbol{\Omega}_{\gN_2, j})}}

\newcommand{\pr}{{\mathbb{P}}}
\newcommand{\gN}{{\mathcal{N}}}
\newcommand{\ex}{{\mathbb{E}}}

\title{A Generalized Latent Factor Model Approach to Mixed-data Matrix Completion with Entrywise Consistency}

\author{
 Yunxiao Chen\\
London School of Economics and Political Science\vspace{0.5cm}\\
  Xiaoou Li\\
  University of Minnesota}
\date{}

\begin{document}
\maketitle

\maketitle

\begin{abstract}
Matrix completion is a class of machine learning methods that concerns the prediction of missing entries in a partially observed matrix. This paper studies matrix completion for mixed data, i.e., data involving mixed types of variables (e.g., continuous, binary, ordinal). We formulate it as a low-rank matrix estimation problem under a general family of non-linear factor models and then propose entrywise consistent estimators for estimating the low-rank matrix. Tight probabilistic error bounds are derived for the proposed estimators. The proposed methods are evaluated by simulation studies and real-data applications for collaborative filtering and large-scale educational assessment.   

\end{abstract}

\section{Introduction}

Missing data are commonly encountered when we analyze real-world data, especially for large-scale data involving many observations and variables. Matrix completion refers to a rich family of machine learning methods that concern the prediction of missing entries in a partially observed matrix. Matrix completion methods have received wide applications, such as 
collaborative filtering \citep{goldberg1992,feuerverger2012statistical}, 
social network recovery \citep{jayasumana2019network},  sensor localization \citep{biswas2006},   and educational and psychological measurement \citep{bergner2022multidimensional,chen2021note}. 

Many matrix completion methods consider
real-valued matrices \citep{candes2009exact, candes2010power,  keshavan2010matrix, klopp2014noisy,koltchinskii2011, negahban2012,  chen2020noisy,xia2021statistical}. Their theoretical guarantees are typically established under a linear factor model \citep[e.g.][]{Bartholomew2008AnalysisData},
which says the underlying complete data matrix can be decomposed as the sum of a low-rank signal matrix $\MM$ and a mean-zero noise matrix. Under this statistical model, 
the matrix completion task  becomes to estimate the signal matrix $\MM$ based on the observed data entries. However, 
many real applications of matrix completion involve mixed types of variables (e.g., continuous, count, binary, ordinal), for which the linear factor model may not be suitable. Methods have been developed for matrix completion with specific variable types, such as binary \citep{cai2013max,davenport20141,han2020asymptotic,han2022general}, categorical \citep{bhaskar2016probabilistic,klopp2015adaptive},   count \citep{cao2015poisson,mcrae2021low,robin2019low}, and
mixed data \citep{robin2020main}. Non-linear factor models, which are extensions of the linear factor model,  are typically assumed in these works.

A matrix completion method is typically evaluated by a mean squared error (MSE)
$\sum_{i=1}^n\sum_{j=1}^p (\hat m_{ij} - m_{ij}^*)^2/(np)$, where $n\times p$ is the size of the data matrix, and $\hat \MM =(\hat m_{ij})_{n\times p}$ and $\MM^* = (m_{ij}^*)_{n\times p}$ are  the estimated and  true signal matrices, respectively. Probabilistic error bounds have been established for the MSE in the literature \citep[see][and references therein]{chen2020noisy,Chen2020DeterminingModels,cai2016matrix}. Under suitable conditions, these error bounds imply that the MSE decays to zero when both $n$ and $p$ grow to infinity, which is viewed as a notion of statistical consistency for matrix completion.  However, this notion of consistency slightly differs from that in our traditional sense; that is, the MSE converging to zero does not imply the convergence of each individual entry, which, however, 
may be important in some applications which concern the prediction of individual data entries. Entrywise results for matrix completion have been established under linear factor models \citep{abbe2020entrywise,chen2019inference,chen2020noisy,chernozhukov2021inference}. However, such results are not available for non-linear factor models, and extending these entrywise results to non-linear factor models is non-trivial.

This paper considers mixed-data matrix completion under 
a generalized latent factor model (GLFM) framework \citep{Bartholomew2008AnalysisData,skrondal2004generalized} which includes many widely used non-linear factor models as special cases. 
Under this model framework, we propose two methods that 
ensure entrywise consistency  under dense and sparse missingness settings. Both methods apply to an initial estimate whose MSE converges to zero. They refine the initial estimate by  solving some estimating equations constructed based on the initial estimate. The difference between the two methods is that one involves data splitting while the other does not. 
The two methods have the same asymptotic behavior under a dense setting where the proportion of observed entries does not decay to zero. In that case, their entrywise error rate matches the MSE of the initial estimate up to a logarithm factor, suggesting that there is virtually no loss when performing refinement. However, under a sparse setting where the proportion of observed entries converges to zero, the procedure with data splitting achieves a smaller error rate than the 
one without data splitting, and the error rate of the data splitting procedure matches the MSE of the initial estimate up to a logarithm factor. 
Our theoretical analysis further shows that a constrained joint maximum likelihood estimator \citep{Chen2019StructuredImplications} for the GLFM
automatically performs a refinement procedure without data splitting, which implies that this estimator is minimax optimal in an entrywise sense under a dense setting and a suitable asymptotic regime.
The proposed methods are evaluated by simulation studies and real-data applications to collaborative filtering and large-scale educational assessment.

The rest of the paper is organized as follows. In Section~\ref{sec:setting}, we introduce a generalized latent factor model for matrix completion with mixed data. In Section~\ref{sec:method}, two methods for achieving entrywise consistency are introduced. Theoretical guarantees on the proposed methods are established in Section~\ref{sec:theo}. Simulation studies and real data examples are given  in Sections~\ref{sec:sim} and \ref{sec:real}, respectively. Finally, we conclude with some discussions in Section~\ref{sec:diss}. {Additional theoretical results, proofs of the theorems, and additional simulation results are given in the supplementary material.}

\section{Mixed-data Matrix Completion}\label{sec:setting}

\subsection{Notation}

For a positive integer $n$, let $[n]:=\{1,\cdots,n\}$ be the set containing all the integers 1, ..., $n$. We let $\|\xx\|$ denote the standard Euclidean norm for a vector $\xx = (x_1, ..., x_n)^T$
and $\|\xx\|_{\infty} = \max_{i} |x_i|$ be the infinity norm (also called the maximum norm) of a vector. 
For a matrix $\XX = (x_{ij})_{n\times m}$, let $\|\XX\|_F$, $\|\XX\|_*$ and $\|\XX\|_2$ denote its Frobenius, nuclear and spectral norms, respectively. We use $\|\XX\|_{\max}:=\max_{i\in[n],j\in[m]}|x_{ij}|$ to denote
the matrix maximum norm, and use $\|\XX\|_{2\to\infty}:=\sup_{\|\uu\|=1}\|\XX\uu\|_{\infty}$ to denote the two-to-infinity norm. According to Proposition 6.1, \cite{Cape2019TheStatistics}, the two-to-infinity norm is the same as the maximum matrix row norm  %
$\|\XX\|_{2\to\infty}=\max_{i\in[n]}(\sum_{j\in[p]}x_{ij}^2)^{1/2}$. For two sequences of real numbers, we write $a_{n,p}\ll b_{n,p}$ (or $a_{n,p}=o(b_{n,p})$) if $\lim_{n,p\to\infty} a_{n,p}/b_{n,p}=0$, $a_{n,p}\gg b_{n,p}$ if $\lim_{n,p\to\infty} a_{n,p}/b_{n,p}=\infty$, $a_{n,p}\lesssim b_{n,p}$ (or $a_{n,p}=O(b_{n,p}))$ if there is a positive constant $M$ independent with $n$ and $p$, such that $|a_{n,p}|\leq M|b_{n,p}|$, $a_{n,p}\gtrsim b_{n,p}$ if there is a positive constant $c$ independent with $n$ and $p$, such that $|a_{n,p}|\geq c|b_{n,p}|$, and $a_{n,p}\sim b_{n,p}$ if $b_{n,p}\lesssim a_{n,p}\lesssim b_{n,p}$.
For two real numbers $x$ and $y$, we denote their maximum and minimum as $x\vee y=\max(x,y)$ and $x\wedge y=\min(x,y)$, respectively. 
We use the standard $O_p(\cdot)$ and $o_p(\cdot)$ notation for stochastic boundedness and convergence in probability, respectively. We use ``$\circ$" for the matrix Hadamard (entrywise) product.  

\subsection{Problem Setup}

    Consider an $n\times p$ data matrix $\YY$, with the $(i,j)$th entry denoted by $Y_{ij}$, for $i\in[n]$ and $j\in[p]$. In the rest, we refer to the rows and columns as the observations and variables, respectively. We do not observe the full matrix due to data missingness. The missing pattern is indicated by an $n\times p$ binary matrix $\OOO=(\omega_{ij})_{i\in[n],j\in[p]}$, where $\omega_{ij}=1$ if $Y_{ij}$ is observed and $\omega_{ij}=0$ if $Y_{ij}$ is missing.  Matrix completion  concerns   inferring the value of $Y_{ij}$ for the missing entries, i.e., entries with $\omega_{ij}=0$. We consider variables of mixed types, which occurs in many real-world applications; that is, we allow $Y_{ij}$ in different columns to be of mixed types, such as continuous, binary, ordinal,  and count variables.

\subsection{A Generalized Latent Factor Model Approach}

Additional assumptions are needed for matrix completion, as otherwise, the missing entries can take any feasible values. A typical assumption for matrix completion is a low-rank assumption, i.e., 
$\YY = \MM + \mathbf E,$
where $\MM$ is a low-rank 
 signal matrix, and $\mathbf E$ is the noise matrix whose entries are independent and mean-zero. Let the rank of $\MM$ be $r$. Then we can write $\YY = \TTT \AAA^T + \mathbf E$, where $\TTT$ and $\AAA$ are $n\times r$ and $p\times r$ matrices, respectively. 
This model is typically known as a linear factor model \citep[e.g.][]{Bartholomew2008AnalysisData}, where  $\TTT$ and $\AAA$ are referred to as the factor-score and loading matrices, respectively. The matrix completion task then becomes an estimation problem, i.e., estimating the signal matrix $\MM = \TTT \AAA^T$ based on the observed data entries.

However, the linear factor model may be restricted when 
not all variables are continuous. The GLFM  is an extension of the linear factor model \citep{Bartholomew2008AnalysisData,skrondal2004generalized}. It assumes that entries $Y_{ij}$ are independent, and the probability density function of $Y_{ij}$ (with respect to some baseline measure) takes an exponential family form 
   $ f_j(y_{ij}|m_{ij},\phi_j)=\exp\big[\phi_j^{-1} \big\{y_{ij}m_{ij}-b_j(m_{ij})\big\}+c_j(y_{ij},\phi_j)\big],$
where $b_j$ and $c_j$ are pre-specified functions, 
$m_{ij}$ is the $(i,j)$th entry of a low-rank signal matrix $\MM = \TTT \AAA^T$ and $\phi_j$ is a dispersion parameter. The density function depends on variable $j$ so that the variables can be of different types. We give some examples below. 

\begin{exam1} 
For a continuous variable $j$, we may assume $f_j$ to be a normal density function, where $\phi_j$ is the variance,  $b_j(m_{ij}) = m_{ij}^2/2$ and $c_j(y_{ij},\phi_j) = -y_{ij}^2/(2\phi_j)   - (\log(2\pi \phi_j))/2$. When all the variables follow this normal model, the data matrix follows a linear factor model. 
\end{exam1}

\begin{exam2} 
Consider a binary or ordinal variable $j$ such that $Y_{ij}$ in $\{0, 1, ..., k_j\}$  for some given $k_j \geq 1$, where $k_j = 1$ and $k_j >1$ correspond to binary and ordinal variables, respectively. We can assume $f_j$ to follow a Binomial logistic  model, 
for which $\phi_j=1$, $b_j(m_{ij}) = k_j \log(1+\exp(m_{ij}))$ and $c_j(y_{ij},\phi_j) =  \log(k_j!) - \log(y_{ij}!) - \log((k_j-y_{ij})!)$. This model has been considered in \cite{masters1984essential}  with psychometric applications. 
When all the variables are binary and follow this logistic model, the data matrix is said to follow a multidimensional two-parameter logistic (M2PL) item response theory model \citep{Reckase2009MultidimensionalTheory}. 
This model  has been considered in \cite{davenport20141} for the completion of binary matrices. 
\end{exam2}

\begin{exam3} 
A Poisson model may be assumed for count variables $j$, for which $\phi_j=1$, $b_j(m_{ij}) =  \exp(m_{ij})$ and $c_j(y_{ij},\phi_j) = -\log(y_{ij}!)  $. When all the variables follow this Poisson model, the
joint model for the data matrix 
is known as  a Poisson factor model \citep{wedel2003factor}. This Poisson model has been considered  in \cite{robin2019low} and \cite{robin2020main} for count data with missing values. 
\end{exam3}

Under the GLFM, $\ex\YY = (b'_j(m_{ij}))_{n\times p}$, where $b'_j(\cdot)$ denotes the derivative of the known function $b_j(\cdot)$. Thus, matrix completion under the GLFM again boils down to estimating the signal matrix $\MM = \TTT \AAA^T$. This estimation problem will be investigated in the rest. We note that a similar GLFM framework has been considered in \cite{robin2020main} for analyzing mixed data with missing values. However, they focused on evaluating the estimation accuracy by the MSE, while our main focus is  the entrywise loss.

\section{Refined Estimation for Entrywise Consistency}\label{sec:method}

As pointed out in the Introduction, the accuracy in estimating $\MM$ is typically measured by the MSE, or equivalently, a scaled Frobenius norm  $\Vert \hat \MM -\MM^*  \Vert_F/\sqrt{np}$, {where $\MM^*$ is the underlying true signal matrix.} %
We say an estimator is F-consistent, if  $\Vert \hat \MM -\MM^*  \Vert_F/\sqrt{np} = o_p(1)$. 
As discussed in Section~\ref{subsec:frob} below, a few F-consistent estimators are available under general or specific GLFMs. 
However, the F-consistency only guarantees consistency in an average sense -- the proportion of inconsistently estimated entries decays to zero. It cannot guarantee entrywise consistency, i.e., the consistency of $\hat m_{ij}$ for each individual data entry, which
may be important in some applications concerning the prediction of individual data entries. Entrywise results for matrix completion,
which focus on the loss $\Vert \hat \MM - \MM^*\Vert_{\max}$, have been established under linear factor models \citep{abbe2020entrywise,chen2019inference,chen2020noisy,chernozhukov2021inference} but not under the GLFM. Establishing entrywise consistency is more challenging under the GLFM due to the involvement of non-linear link functions of the exponential family. 
In what follows, we propose  methods that can improve an F-consistent estimator to an entrywise consistent (E-consistent) estimator under the GLFM. 

\subsection{Refinement without Data Splitting}

Let $\hat{\MM}$ be given by an F-consistent estimator based on observed data $(\YY\circ \OOO,\OOO)$; see Section~\ref{subsec:frob} for examples of such estimators. We propose the following refinement procedure that inputs $\hat \MM$   and outputs an E-consistent estimator.
\begin{meth}[Refinement Procedure without Data Splitting]\label{meth:nosplit}
~
\vspace{-0.3cm}
\begin{itemize}
\item[] {\bf Input}:  Observed data $(\YY\circ \OOO,\OOO)$, an initial estimate $\hat \MM$ and a pre-specified constant $C_2$. 
    \item[] {\bf Step 1.}    Perform singular value decomposition (SVD) to $\hat{\MM}$ and
       obtain $\hat{\VV}_r\in\bR^{p\times r} $ which contains 
       the top-$r$ right singular vectors of $\hat{\MM}$.
    \item[] {\bf Step 2.} Calculate $\hat{\AAA}=\textbf{proj}_{\{\AAA\in \bR^{p\times r}:\|\AAA\|_{2\to\infty}\leq C_2\}}(\hat{\VV}_r)$, where $\textbf{proj}_{\{\AAA\in \bR^{p\times r}:\|\AAA\|_{2\to\infty}\leq C_2\}}(\cdot)$ denotes a projection operator that projects a $p\times r$ matrix to satisfy the two-to-infinity norm constraint. 
    \item[] {\bf Step 3.} For each $i\in[n]$, calculate $\tilde{\ttt}_i$ by solving an equation:
\begin{equation}\label{eq:score1}
\sum_{j=1}^p \omega_{ij} \{y_{ij}-b'_j((\hat{\aaaa}_j)^T\tilde{\ttt}_i)\}\hat{\aaaa}_j = \mathbf{0}_r.
\end{equation}
\item[] {\bf Step 4.} For each $j\in[p]$, obtain $\tilde{\aaaa}_j$ by solving the following  equation:
\begin{equation}\label{eq:score2}
\sum_{i=1}^n \omega_{ij} \{y_{ij}-b'_j((\tilde{\aaaa}_j)^T\tilde{\ttt}_i)\}\tilde{\ttt}_i = \mathbf{0}_r.
\end{equation}
\item[] {\bf Output:} $\tilde{\MM}=\tilde{\TTT}(\tilde{\AAA})^T$, where $\tilde{\TTT} = (\tilde{\ttt}_1,\cdots,\tilde{\ttt}_n)^T\in\mathbb{R}^{n\times r}$ and $\tilde{\AAA}=(\tilde{\aaaa}_1,\cdots,\tilde{\aaaa}_p)^{T}\in \mathbb{R}^{p\times r}$ are obtained from Steps 3 and 4, respectively. %
\end{itemize}

\end{meth}

We comment on the implementation. First, the constant $C_2$ depends on the true signal matrix $\MM^*$. Recall that we assume $\MM^*$ to be of rank $r$ under the GLFM. Thus, 
$\MM^*$ can be decomposed as $\MM^*=\UU^*_r\DD^*_r(\VV^*_r)^T$, where $\UU_r^*\in \bR^{n\times r}$  and $\VV_r^*\in \bR^{p\times r}$ are the left and right singular matrices corresponding to the non-zero singular values, and $\DD_r^*\in\bR^{r\times r}$ is a diagonal matrix whose diagonal elements are the singular values $\sigma_1(\MM^*)\geq\cdots\geq \sigma_r(\MM^*)>0$. We require $C_2$ to satisfy $C_2\geq \Vert \VV_r^*\Vert_{2\to\infty}$. On the other hand, $C_2$ should not be chosen too large. As will be shown in Section~\ref{subsec:thmnosplit}, {it is assumed that $C_2$ has the same asymptotic order as $\|\VV_r^*\|_{2\to\infty}$; otherwise, the error bound for $\Vert\tilde\MM - \MM^*\Vert_{\max}$ needs additional modification.}
Second, we note that the projection in Step 2 is very easy to perform. Let $\VV =(\mathbf v_1, ..., \mathbf v_p)^T$ be a $p\times r$ matrix. Then $\textbf{proj}_{\{\AAA\in \bR^{p\times r}:\|\AAA\|_{2\to\infty}\leq C_2\}}(\VV) = (\tilde{\mathbf v}_1, ..., \tilde{\mathbf v}_p)^T$, where $\tilde{\mathbf v}_i = \mathbf v_i$ if $\Vert \mathbf v_i\Vert \leq C_2$ and  $\tilde{\mathbf v}_i =  (C_2/\Vert \mathbf v_i\Vert)\mathbf v_i$ otherwise. Finally, we provide a remark  on solving the equations in Steps 3 and 4. 

\begin{rmk}\label{rmk:lik}
In Steps 3 and 4, we propose to solve some estimating equations. As will be shown in Section~\ref{sec:theo}, these  equations have a unique solution with probability converging to $1$ under a suitable asymptotic regime. These steps are equivalent to performing optimization to certain log-likelihood functions.
Let 
 $   \ell(\MM) = \sum_{i,j: \omega_{ij}=1} \big\{y_{ij}m_{ij}-b_j(m_{ij})\big\}$
be a weighted log-likelihood function based on observed data $(\YY\circ \OOO,\OOO)$, where the individual log-likelihood terms are weighted by the dispersion parameters\footnote{The weighted likelihood is used so that the nuisance parameters $\phi_j$ do not involve in estimating $\MM$, which simplifies the theoretical analysis. We believe that the current analysis can be extended to the unweighted log-likelihood function for the joint estimation of $\MM$ and dispersion parameters $\phi_j$.}. Then, solving the estimating equations \eqref{eq:score1} is equivalent to solving $ \tilde{\TTT} \in\argmax_{\TTT}\ell(\TTT\hat{\AAA}^T)$, and solving the estimating equations \eqref{eq:score2} is equivalent to solving $ \tilde{\AAA} \in\argmax_{\AAA}\ell (\tilde{\TTT}{\AAA}^T)$. This is due to that the estimating equations \eqref{eq:score1} and \eqref{eq:score2} are obtained by taking the partial derivatives of $\ell ({\TTT}{\AAA}^T)$with respect to $\TTT$ and $\AAA$, respectively, and that the objective function $\ell({\TTT}{\AAA}^T)$ is convex  with respect to $\TTT$ and $\AAA$ given the other.

\end{rmk}

We provide an informal theorem under a simplified setting to shed some light on the 
 asymptotic behavior of Method~\ref{meth:nosplit}.
 Its formal version is Theorem~\ref{thm:m-bound-no-splitting}  in Section~\ref{subsec:thmnosplit}, which is established  under a more general setting. 

\begin{theorem}[An informal and simplified version of Theorem~\ref{thm:m-bound-no-splitting}]\label{thm:simp1}
 Assume that $\lim_{n,p\to\infty}\pr(\|\hat{\MM}-\MM^*\|_F\leq e_{M,F})=1$ {and let $\tilde{\MM}$ be obtained by Method~\ref{meth:nosplit}.} 
 Then, under suitable assumptions on $\MM^*$ and the asymptotic regime $\pi_{\min}=\pi_{\max}=\pi$, $r$ is fixed, $p\pi,n\pi\gg (\log(np))^3$, and $\{(n\wedge p)\pi\}^{-1/2}\lesssim  
 (np)^{-1/2}e_{M,F}\ll \pi^{1/2}(\log(n p))^{-2}$, 
we have
 $ \|\tilde{\MM}-\MM^*\|_{\max}
    \lesssim  (\log(np))^2\cdot \pi^{-1/2}(np)^{-1/2}e_{\MM,F}.$
\end{theorem}
We consider the asymptotic regime $\{(n\wedge p)\pi\}^{-1/2}\lesssim (np)^{-1/2}e_{M,F}$ above because $\{(n\wedge p)\pi\}^{-1/2}$ is the minimax error rate of $(np)^{-1/2}\|\hat{\MM}-\MM^*\|_F$; see \cite{Chen2020DeterminingModels}.
\subsection{Refinement with Data Splitting}

From Theorem~\ref{thm:simp1} above, we see that   $\Vert\tilde\MM -\MM^*\Vert_{\max}$ achieves the same error rate as $\Vert\hat\MM -\MM^*\Vert_{F}/\sqrt{np}$ (up to a logarithm factor) when $\pi \sim 1$. However, when $\pi = o(1)$, 
the rate of $\Vert\tilde\MM -\MM^*\Vert_{\max}$ becomes worse than that of $\Vert\hat\MM -\MM^*\Vert_{F}/\sqrt{np}$, due to the factor $\pi^{-1/2}$ in the upper bound.  This term comes from the worst case scenario when $\hat{\AAA}-\AAA^*$ is highly dependent with $(\omega_{ij})_{j\in[p]}$ for some $i$ (e.g., $\hat{\aaaa}_j-\aaaa_j^*\approx \omega_{ij}\mathbf{b}$ for all $j\in[p]$, some $i\in [n]$, and some random vector $\mathbf{b}\in\mathbb{R}^r$).
To obtain a better error rate under the max norm, we propose a new procedure that uses 
a data splitting step to break the dependence between $\hat{\AAA}$ and $\OOO$.
The proposed data splitting method is similar to the one proposed in \cite{chernozhukov2021inference} for linear factor models, where a similar dependence issue exists. 
{However, due to the non-linear link functions involved in the GLFM, the development of our method and its theory faces unique challenges. }

Let $\gN_1\subset [n]$ be a random subset independent of $(\YY,\OOO)$. In particular, we let $I(i\in \gN_1)$ be i.i.d. Bernoulli random variables with $\pr(i\in \gN_1)=1/2$ for $i\in[n]$, where $I(\cdot)$ denotes the indicator function. By the law of large numbers, $\gN_1$ is a subset of $[n]$ with size around $n/2$. We further let $\gN_2 = [n]\setminus \gN_1$. %
\begin{meth}[Refinement Procedure with Data Splitting]\label{meth:split}
~
\vspace{-0.3cm}
\begin{itemize}
\item[] {\bf Input}: Observed data
$(\YY\circ\OOO,\OOO)$, a constraint parameter $C_2$, and 
initial estimates 
{$\hat{\MM}_{\gN_k,\cdot}$ for $ \MM_{\gN_k,\cdot}=( m_{ij})_{i \in \gN_k, j\in [p]}$ obtained based on $(\YY\circ\OOO,\OOO)_{\gN_k,\cdot}=(y_{ij}\omega_{ij}, \omega_{ij})_{i\in\gN_k, j\in [p]}$ for $k=1,2$.}
    \item[] {\bf Step 1.}    Perform SVD to $\hat{\MM}_{\gN_1,\cdot }$ and
   calculate $\hat{\VV}_r^{(1)} \in\bR^{p\times r} $
   which contains the top-$r$ right singular vectors of
    $\hat{\MM}_{\gN_1,\cdot }$. 
    \item[] {\bf Step 2.} 
    Calculate $\hat{\AAA}^{(1)}= (  \hat{\aaaa}_j^{(1)})_{j\in[p]}^T=\textbf{proj}_{\{\AAA\in \bR^{p\times r}:\|\AAA\|_{2\to\infty}\leq C_2\}}(\hat{\VV}_r^{(1)})$. %

\item[] {\bf Step 3.} Calculate $\tilde{\TTT}_{\gN_2}=(\tilde{\ttt}_i)^T_{i\in \gN_2}$, where for each $i \in \gN_2$, $\tilde{\ttt}_i$ is obtained by solving 
the equation
$\sum_{j=1}^p \omega_{ij} \{y_{ij}-b_j'((\hat{\aaaa}_j^{(1)})^T\tilde{\ttt}_i)\}\hat{\aaaa}_j^{(1) }= \mathbf{0}_r.$

\item[] {\bf Step 4.}  Calculate 
$\tilde{\AAA}^{(1)}=(\tilde{\aaaa}_j^{(1)})^T_{j\in [p]}$, 
where for each $j \in [p]$, $\tilde{\aaaa}_j^{(1)}$ is obtained by solving the equation
$\sum_{i\in\gN_2} \omega_{ij} \{y_{ij}-b_j'((\tilde{\aaaa}_j^{(1)})^T\tilde{\ttt}_i)\}\tilde{\ttt}_i = 0_r.$

\item[] {\bf Step 5.}   Swap $\gN_1$ and $\gN_2$ in Steps 1--4, and obtain $\tilde{\TTT}_{\gN_1}$ and $\tilde{\AAA}^{(2)}$ accordingly.
\item[] {\bf Output:} $\tilde{\MM}=(\tilde{m}_{ij})_{i\in[n],j\in[p]}$, where  $ (\tilde m_{ij})_{i \in \gN_1, j\in [p]}= \tilde{\TTT}_{\gN_1}(\tilde{\AAA}^{(2)})^{T}$ and 
 $(\tilde m_{ij})_{i \in \gN_2, j\in [p]}
=  \tilde{\TTT}_{\gN_2}(\tilde{\AAA}^{(1)})^{T}$.

\end{itemize}

\end{meth}
The comments on Method~\ref{meth:nosplit} regarding the choice of $C_2$, the projection operator, and the solutions to the estimating equations   apply similarly to Method~\ref{meth:split}. {As the rows and columns of the data matrix play a similar role, the above method can be modified to split the columns instead of the rows.}
As summarized in Theorem~\ref{thm:simp2}, which is an informal and simplified version of Theorem~\ref{thm:m-bound-same-eta} in Section~\ref{subsec:thmsplit}, Method~\ref{meth:split} improves the error rate of Method~\ref{meth:nosplit}. 
In fact, $\Vert\tilde\MM -\MM^*\Vert_{\max}$ now achieves the same error rate as $\Vert\hat\MM -\MM^*\Vert_{F}/\sqrt{np}$ up to a logarithm factor, regardless of the missing rate $\pi$. 

\begin{theorem}[An informal and simplified version of Theorem~\ref{thm:m-bound-same-eta}]\label{thm:simp2}
Assume that $\lim_{n,p\to\infty}\pr(\|\hat{\MM}_{\gN_k\cdot}-\MM^*_{\gN_k\cdot}\|_{F}\leq e_{\MM,F}) =1$ for  $e_{\MM,F}$ $(k=1,2)$ {and $\tilde{\MM}$ is obtained by Method~\ref{meth:split}}. Then, under suitable assumptions on $\MM^*$ and the asymptotic regime $\pi_{\min}=\pi_{\max}=\pi$, $r$ is fixed, $p\pi,n\pi\gg (\log(np))^3$, and $\{(n\wedge p)\pi\}^{-1/2}\lesssim (np)^{-1/2}e_{M,F}\ll (\log(np))^{-2} $, we have
  $\|\tilde{\MM}-\MM^*\|_{\max}
    \lesssim  (\log(np))^{2} (np)^{-1/2}e_{\MM,F}.$
\end{theorem}

As the data splitting in Method~\ref{meth:split} is random, it may be beneficial to run it multiple times and then aggregate the resulting estimates. We describe this variation of Method~\ref{meth:split} below. For a fixed number of random splittings, the asymptotic behavior of Method 2' is the same as that of Method~\ref{meth:split}. 

\begin{Meth2} 
~
\vspace{-0.3cm}
\begin{itemize}
\item[] {\bf Input}: Observed data
$(\YY\circ\OOO,\OOO)$
a constraint $C_2$ and the number of data splittings $\mbox{tot}$.

    \item[] {\bf Step 1.}    Independently generate index sets $\gN_1^{(k)} and \gN_2^{(k)}$  {and obtain initial estimates 
    $\hat \MM_{\gN_1}^{(k)}$ and $\hat \MM_{\gN_2}^{(k)}$ based on $(\YY\circ\OOO,\OOO)_{\gN_1^{(k)},\cdot}=(y_{ij}\omega_{ij}, \omega_{ij})_{i\in\gN_1^{(k)}, j\in [p]}$ and $(\YY\circ\OOO,\OOO)_{\gN_2^{(k)},\cdot}=(y_{ij}\omega_{ij}, \omega_{ij})_{i\in\gN_2^{(k)}, j\in [p]}$, respectively, for $k=1,2, ..., \mbox{tot}$.}
    \item[] {\bf Step 2.} For $k = 1, ..., \mbox{tot}$, run Method~\ref{meth:split} with  
    data $(\YY\circ\OOO,\OOO)$, initial estimates {$\hat \MM_{\gN_1}^{(k)}$ and $\hat \MM_{\gN_2}^{(k)}$, index sets $\gN_1^{(k)}, \gN_2^{(k)}$ and a constraint parameter $C_2$. Obtain outputs $\tilde \MM^{(k)}$, $k = 1, ..., \mbox{tot}$. }  
 \item[] {\bf Output:} $\tilde{\MM} = (\sum_{k=1}^{\text{tot}} \tilde \MM^{(k)})/\mbox{tot} $ . 

\end{itemize}

\end{Meth2}

\subsection{F-consistent Estimators}\label{subsec:frob}
Our refinement methods require input from an F-consistent estimator. We  give examples of F-consistent estimators.

\begin{CJMLE}

The constrained joint maximum likelihood estimator (CJMLE) 
{solves}
the following optimization problem
\begin{equation}\label{eq:cjmle}
\begin{aligned}
  (\hat \TTT, \hat \AAA) \in \arg\max_{\TTT, \AAA}& ~~\ell(\TTT\AAA^T),~~
\mbox{s.t.} & ~ \TTT \in \mathbb R^{n\times r}, \AAA \in \mathbb R^{p\times r}, \Vert \TTT\Vert_{2\to\infty} \leq C, \Vert \AAA \Vert_{2\to\infty} \leq C.
\end{aligned}    
\end{equation}
The estimate of $\MM$ is then given by $\hat \MM=\hat \TTT \hat \AAA^T$. 
The terminology ``joint likelihood" comes from the latent variable model literature \citep[Chapter 6,][]{skrondal2004generalized}. This literature distinguishes the   joint likelihood from the marginal likelihood, depending on whether entries of $\TTT$ are treated as fixed parameters or random variables, where the marginal likelihood is more commonly adopted in the statistical inference of traditional latent variable models. This estimator was first proposed in \cite{Chen2019StructuredImplications} for the estimation of high-dimensional GLFM, and an error bound on $\Vert\hat \MM -\MM^*\Vert_F$ under a general matrix completion setting can be found in Theorem 2 of \cite{Chen2020DeterminingModels}. 

More specifically, suppose that the true signal matrix has a decomposition $\MM^* = \TTT^*(\AAA^*)^T$, such that 
$\Vert \TTT^*\Vert_{2\to\infty} \leq C$ and $\Vert \AAA^* \Vert_{2\to\infty} \leq C.$
Then, under a similar setting as in Theorems~\ref{thm:simp1} and \ref{thm:simp2}, we have 
$\lim_{n,p\to\infty}\pr(\|\hat{\MM}-\MM^*\|_{F}/\sqrt{np}\leq \kappa^\dagger \{(p\wedge n) \pi\}^{-1/2}) =1$, for some finite positive constant $\kappa^\dagger$. As shown in Proposition 1 of \cite{Chen2020DeterminingModels}, $\{(p\wedge n) \pi\}^{-1/2}$ is also the minimax lower bound for estimating $\MM$ in the scaled Frobenius norm,  {which is why this lower bound is assumed for 
$(np)^{-1/2}e_{\MM,F}$ in Theorems \ref{thm:simp1} and \ref{thm:simp2}.}

We note that $(\hat \TTT, \hat \AAA)$ given by the CJMLE jointly maximizes the weighted likelihood. Following the discussion in Remark~\ref{rmk:lik} on the connection between the estimating equations and the weighted likelihood,  
$\hat \MM = \hat \TTT\hat \AAA^T$
will remain unchanged when input into Method~\ref{meth:nosplit} if $C_2$ is chosen properly according to $C$. Consequently, the CJMLE is automatically entrywise consistent under a suitable asymptotic regime; see Remark~\ref{rmk:cjmle} for a discussion.

\end{CJMLE}   

\begin{NBE}
The CJMLE requires solving a non-convex optimization problem for which convergence to the global optimum is not always guaranteed. The nuclear-norm-based estimator (NBE) is a convex approximation to CJMLE. It solves the following optimization problem 
\begin{equation}\label{eq:nbe}
\begin{aligned}
  \hat \MM \in \arg\max_{\MM}& ~~\ell(\MM),~~
\mbox{s.t.}  ~~ \Vert \MM\Vert_{\max} \leq \rho',  \Vert \MM\Vert_{*} \leq \rho' \sqrt{rnp}.
\end{aligned}    
\end{equation}
The nuclear norm constraint is introduced, since 
$\{\MM \in \mathbb R^{n\times p}: \Vert\MM\Vert_{\max} \leq \rho', \Vert \MM\Vert_{*} \leq \rho' \sqrt{rnp}\}$ is a convex relaxation of $\{\MM \in \mathbb R^{n\times p}: \Vert\MM\Vert_{\max} \leq \rho', \mbox{rank}(\MM)\leq r\}$. 
This estimator has been considered in \cite{davenport20141} for the completion of binary matrices. 
When the true model follows the M2PL model and the true signal matrix $\MM^*$ satisfies $\Vert\MM^*\Vert_{\max}\leq \rho'$, then Theorem~1 of  \cite{davenport20141} implies that {under the same setting of Theorems~\ref{thm:simp1} and \ref{thm:simp2},}
$\lim_{n,p\to\infty}\pr(\|\hat{\MM}-\MM^*\|_{F}/\sqrt{np}\leq \kappa^\ddagger \{(p\wedge n) \pi\}^{-1/4}) =1$, where $\kappa^\ddagger$ is a finite positive constant which depends on the true model parameters. {We believe that the same rate holds for other GLFMs under the simplified setting  of Theorems~\ref{thm:simp1} and \ref{thm:simp2}.}

\end{NBE}

\begin{OthE}
Note that other F-consistent estimators may be available for GLFMs, such as SVD-based methods \citep{chatterjee2015matrix,zhang2020note}, regularized estimators \citep{klopp2014noisy,koltchinskii2011,negahban2012,robin2020main}, 
and methods based on a matrix factorization norm \citep{cai2013max,cai2016matrix}.

\end{OthE}

\section{Theoretical Results}\label{sec:theo}

\subsection{Assumptions and useful quantities}
{We make the following Assumptions~\ref{assump:support} and \ref{assump:constraint} throughout Section~\ref{sec:theo}. }
\begin{assumption}\label{assump:support}
{$b_1(x)=\cdots = b_p(x)=b(x)$ for all $x\in\mathbb{R}$.} In addition,	$b(x)<\infty$ and $b''(x)>0$ for all $x\in\mathbb{R}$.
\end{assumption}
{We note that this assumption is made for ease of presentation. It can be relaxed to allowing functions $b_j$
to be variable-specific, and similar theoretical results hold following a similar proof.} For each $\alpha>0$, define functions
 $   \kappa_2(\alpha)=\sup_{|x|\leq \alpha} b''(x), \kappa_3(\alpha) =\sup_{|x|\leq \alpha}|b^{(3)}(x)|, \text{ and }\delta_2(\alpha) = \inf_{|x|\leq \alpha} b''(x).$
Let $\MM^*$ have the SVD $\MM^*=\UU^*_r\DD^*_r(\VV^*_r)^T$ where $r$ is the rank of $\MM^*$, $\UU_r^*\in \bR^{n\times r}$  and $\VV_r^*\in \bR^{p\times r}$ are the left and right singular matrices corresponding to the top-$r$ singular values, respectively, and $\DD_r^*\in\bR^{r\times r}$ is a diagonal matrix whose diagonal elements are the singular values $\sigma_1(\MM^*)\geq\cdots\geq \sigma_r(\MM^*)>0$. 
In order to apply the proposed methods, we need to input $C_2$.
\begin{assumption}\label{assump:constraint} We choose $C_2$ such that
   $ C_2\geq \|\VV_r^*\|_{2\to\infty}$.
\end{assumption}

Define the following quantities that depend on $\MM^*$. Let 
 $\rho = \max_{i\in[n],j\in[p]}|m_{ij}^*|,  C_1 = \{\|\UU_r^*\|_{2\to\infty}\vee (r/n)^{1/2}\}\cdot \sigma_1(\MM^*)$, 
$   \kappa_2^* = \kappa_2(2\rho+1), \delta_2^* = \delta_2(2\rho+1), \text{ and } \kappa_3^* = \kappa_3(6 C_1C_2).$
For the missing pattern $\OOO=(\omega_{ij})_{i\in[n],j\in[p]}$, let $\pi_{ij}=\pr(\omega_{ij}=1)$ be the sampling probabilities and  $\pi_{\min}=\min_{i\in[n],j\in[p]}\pi_{ij}$ and $\pi_{\max}=\max_{i\in[n],j\in[p]}\pi_{ij}$ be the minimal and maximal sampling probability, respectively. 

\subsection{Error analysis without data splitting}\label{subsec:thmnosplit}

\begin{theorem}\label{thm:m-bound-no-splitting}
	    Assume that $\lim_{n,p\to\infty}\pr(\|\hat{\MM}-\MM^*\|_F\leq e_{\MM,F})=1$, {$\tilde{\MM}$ is obtained by Method~\ref{meth:nosplit}}, and the following asymptotic regime holds:
    \begin{enumerate}
               \item[R1] $\phi_1=\cdots=\phi_p=\phi\sim 1$; 
               \item[R2] $\pi_{\min}\sim\pi_{\max}\sim \pi$;
    \item[R3] $\|\UU_r^*\|_{2\to\infty}\lesssim (r/n)^{1/2}$, $\|\VV_r^*\|_{2\to\infty} \lesssim(r/p)^{1/2}$, $C_2 \sim (r/p)^{1/2}$;
    \item[R4] $ \sigma_r(\MM^*)\sim \sigma_1(\MM^*) \sim (np)^{1/2} r^{\eta}$ for some constants $\eta\geq -1$;
    \item[R5]          $p\pi
        \gg (\kappa_2^*)^{4}(\delta_2^*)^{-6} (\log(n p))^{3} 
        \max\Big[  r^{(1+2\eta)\vee 5}, (\kappa_3^*)^2 r^{(3+4\eta)\vee 7}  \Big]$;
    \item[R6] $n\pi
         \gg  (\kappa_2^*)^2(\delta_2^*)^{-4}   (\log(np))^2 \max \big\{r^{3
         },  (\kappa_3^*)^2 r^{5} \big\}
$;
     \item[R7] $ (np)^{-1/2}e_{\MM,F}
            \ll   (\kappa_2^*)^{-2}(\delta_2^*)^{3} (\log(n p))^{-2} \min\big[   r^{ -5/2},
            (\kappa_3^*)^{-1}  r^{-7/2 } \big] \pi^{1/2}.$

    \end{enumerate}
    Then, 
    with probability converging to $1$, 
    {estimating equations in steps 3 and 4 of Method~\ref{meth:nosplit} have a unique solution and}
    \begin{equation}\label{eq:asymp-bound-nosplit}
    \begin{split}
    \|\tilde{\MM}-\MM^*\|_{\max}
    \lesssim  (\delta_2^*)^{-2}(\kappa_2^*)^2 (\log(np))^2r^{5/2}\Big[\{(n\wedge p)\pi\}^{-1/2} +  (np\pi)^{-1/2}e_{\MM,F}\Big].
       \end{split}
    \end{equation}
In particular, if we further assume that $r\sim 1$,  then, the asymptotic regime requirements {\bf R5 -- R7} can be simplified as $p\pi\gg(\log(np))^3$, $n\pi \gg (\log(np))^2$ and $(np)^{-1/2}e_{\MM,F}\ll (\log(np))^{-2} \pi^{1/2}$, and we have that with probability converging to $1$,
$
	   \|\tilde{\MM}-\MM^*\|_{\max}
    \lesssim  (\log(np))^2\big[\{(n\wedge p)\pi\}^{-1/2}+(np\pi)^{-1/2}e_{\MM,F}\big].
$

\end{theorem}

\begin{rmk}\label{remark:asymp-regime-nosplit}
We comment on the asymptotic requirement {\bf R1--R7}. {\bf R1} requires the dispersion parameters to be the same for different $j\in[p]$. This assumption is made for ease of presentation, and it can be easily relaxed to allow varying values of dispersion parameters. It further requires that the dispersion parameter is bounded as $n$ and $p$ grow large. {\bf R2} requires $\pi_{\max}$ and $\pi_{\min}$ to be of the same asymptotic order. That is, the missing pattern is not too far from the commonly adopted uniform missingness assumption where all the $\pi_{ij}$ are the same \citep[see, e.g.][]{candes2010power,davenport20141}. 
{\bf R3} is a standard incoherent condition that is commonly assumed for matrix completion to avoid spiky low-rank matrices 
\citep{candes2009exact,jain2013low}. 
{\bf R4} requires that the non-zero singular values  of $\MM^*$ are in the same asymptotic order. In addition, we restrict the analysis to the case where $\eta\geq -1$, because otherwise $\|\MM^*\|_{\max}\ll 1$ and the asymptotic regime is less interesting. We note that {\bf R4} can be relaxed to a more general asymptotic regime allowing $\sigma_r(\MM^*)$ and $\sigma_1(\MM^*)$ to have different asymptotic order, and we provide the error analysis  under a more general setting in the supplementary material. {\bf R5} and {\bf R6} require the expected number of non-missing observations for each row and column to be large enough. {\bf R7} requires the initial F-consistent estimator to have a sufficiently small estimation error in scaled Frobenius norm. In  Corollaries~\ref{coro:binomial-nosplit} -- \ref{coro:Poisson-nosplit} below, we give sufficient conditions for {\bf R5} -- {\bf R7} under the three specific GLFMs described in Section~\ref{sec:setting}.
\end{rmk}

\begin{rmk}\label{rmk:cjmle}
Let $\hat{\MM}_{\text{CJMLE}}$ and $\hat{\MM}_{NBE}$ denote the constrained joint maximum likelihood estimator and nuclear-norm-based estimator described in Section~\ref{subsec:frob}, respectively. Also let $\tilde{\MM}_{\text{CJMLE}}$ and $\tilde{\MM}_{NBE}$ be the corresponding refined estimators by applying Method~\ref{meth:nosplit}.
Theorem~\ref{thm:m-bound-no-splitting} indicates that with high probability $\|\tilde{\MM}_{\text{CJMLE}}-\MM^*\|_{\max}\lesssim (\log(np))^2 \pi^{-1} (n\wedge p)^{-1/2}$ and $\|\tilde{\MM}_{\text{NBE}}-\MM^*\|_{\max}\lesssim (\log(np))^2 \pi^{-3/4} (n\wedge p)^{-1/4}$ when $r$ is bounded, under suitable regularity conditions. Note that $\tilde{\MM}_{\text{CJMLE}}$ is a fixed point of Method~\ref{meth:nosplit} with high probability. Thus, $\hat{\MM}_{\text{CJMLE}}=\tilde{\MM}_{\text{CJMLE}}$, which implies that we have the same error rate for the estimator $\hat{\MM}_{\text{CJMLE}}$ without refinement. {Because $\hat{\MM}_{\text{CJMLE}}$ is asymptotically minimax when $\pi\sim 1$ in Frobenius norm, we also have $\hat{\MM}_{\text{CJMLE}}$ is asymptotically minimax in the matrix max norm.}
\end{rmk}

In the following corollaries, we provide sufficient conditions for {\bf R5 - R7} under specific GLFMs discussed earlier.
\begin{corollary}[Binomial Model]\label{coro:binomial-nosplit}
    Assume that $\lim_{n,p\to\infty}\pr(\|\hat{\MM}-\MM^*\|_{F}\leq e_{\MM,F}) =1$ for some non-random $e_{\MM,F}$.
    In addition, assume that data follow a binomial factor model and that asymptotic requirements {\bf R2 - R4} in Theorem~\ref{thm:m-bound-no-splitting}  hold as $n,p\to\infty$. Then, \eqref{eq:asymp-bound-nosplit} holds if there is a constant $\epsilon_0>0$ such that the following asymptotic regime holds:
    \begin{enumerate}
      \item [R5B] $p\pi \gg (n\vee p)^{\epsilon_0}   r^{(3+4\eta)\vee 7}$;
    \item [R6B] $n\pi
         \gg    (n\vee p)^{\epsilon_0} r^{5}
$;
     \item[R7B] $ (np)^{-1/2}e_{\MM,F}
            \ll  (n\wedge p)^{-\epsilon_0}\pi^{1/2}  r^{ -7/2}$;
    \item [R8B] $k_1=\cdots=k_p = k \sim 1$;
    \item [R9B] %
    $\rho\lesssim \log(n\wedge p)^{1-\epsilon_0}$.
\end{enumerate}
\end{corollary}
In the above corollary,  {\bf R1} automatically holds because the dispersion parameter $\phi_j=1$ in the binomial model. 
\begin{corollary}[Normal Model]\label{coro:linear-nosplit}
    Assume that $\lim_{n,p\to\infty}\pr(\|\hat{\MM}-\MM^*\|_{F}\leq e_{\MM,F}) =1$ for some non-random $e_{\MM,F}$. 
    In addition, assume that data follow a normal factor model and that asymptotic requirements {\bf R1 - R4} in Theorem~\ref{thm:m-bound-no-splitting} hold as $n,p\to\infty$. Then, \eqref{eq:asymp-bound-nosplit} holds under the following asymptotic regime:
    \begin{enumerate}
      \item [R5N] $p\pi \gg (\log(n p))^{3}   r^{(1+2\eta)\vee 5}$;
    \item [R6N] $n\pi
         \gg    (\log(np))^2 r^{3
         }
$;
     \item[R7N] $ (np)^{-1/2}e_{\MM,F}
            \ll  (\log(n p))^{-2}\pi^{1/2}  r^{ -5/2}.$
\end{enumerate}
\end{corollary}

\begin{corollary}[Poisson  Model]\label{coro:Poisson-nosplit}
 Assume that $\lim_{n,p\to\infty}\pr(\|\hat{\MM}-\MM^*\|_{F}\leq e_{\MM,F}) =1$ for some non-random $e_{\MM,F}$.
 In addition, assume that data follow a Poisson factor model and that asymptotic requirements {\bf R2 - R4} in Theorem~\ref{thm:m-bound-no-splitting} and {\bf R5B --R7B} in Corollary~\ref{coro:binomial} hold as $n,p\to\infty$. Then, \eqref{eq:asymp-bound-nosplit} holds if there is a constant $\epsilon_0>0$ such that the following asymptotic regime holds:

 \begin{enumerate}
     \item [R10P]    
    $r^{1+\eta}\lesssim (\log(n\wedge p))^{1-\epsilon_0}$.
\end{enumerate}

\end{corollary}

\begin{rmk}\label{remark:corollaries}
We comment on the asymptotic requirements in the above corollaries. {\bf R5B, R6B, R5N} and {\bf R6N} require that rank $r$ is relatively small comparing with $(n\wedge p)\pi$, and it can grow at most of the order $\{(n\wedge p)\pi\}^{\nu_1}$ for some constant $\nu_1\in (0,1)$. Conditions {\bf R5B} and {\bf R6B} are slightly stronger  than {\bf R5N} and {\bf R6N}, because $\kappa_3^*=0$ for the normal model while $\kappa_3^*\sim 1$ for the binomial model. Conditions {\bf R7B} and {\bf R7N} require the scaled Frobenius norm of the initial estimator to be small. Many F-consistent estimators, including CJMLE and NBE, have the error rate $(np)^{-1/2}e_{\MM,F}\sim  ( (n\wedge p)\pi)^{-\nu_2}$ for some $\nu_2\in (0,1)$. For these estimators, {\bf R7B} and {\bf R7N} require that $r\lesssim  ( (n\wedge p)\pi)^{\nu_3} \pi^{1/2}$ for some $\nu_3\in (0,1)$. Condition {\bf R8B} requires the $k_j$s to be the same for different $j\in [p]$ and are bounded. This condition can be easily relaxed to a more general setting with varying but bounded $k_j$s. Condition {\bf R9B} requires that $\rho$ grows much slower than $n$ and $p$. Similar assumptions are made for 1-bit matrix completion \citep{davenport20141,cai2013max}.  {For Poisson factor models,  {\bf R10P} can be achieved either by an arbitrary $r$ with $\eta=-1$ or by $r\lesssim (\log(n\wedge p))^{(1-\epsilon_0)/(1+\eta)}$ with $\eta>-1$. 
}
\end{rmk}

\subsection{Error analysis with data splitting}\label{subsec:thmsplit}
\begin{theorem}\label{thm:m-bound-same-eta}
Assume that $\lim_{n,p\to\infty}\pr(\|\hat{\MM}_{\gN_k\cdot}-\MM^*_{\gN_k\cdot}\|_{F}\leq e_{\MM,F}) =1$ for some non-random $e_{\MM,F}$ $(k=1,2)$, { and $\tilde{\MM}$ is obtained by Method~\ref{meth:split}}.
Assume asymptotic requirements {\bf R1 - R6} in Theorem~\ref{thm:m-bound-no-splitting} hold as $n,p\to\infty$. Also, assume the following asymptotic requirements:
   \begin{enumerate}
        \item[R7'] $(np)^{-1/2}e_{\MM,F}
            \ll  (\kappa_2^*)^{-2}(\delta_2^*)^{3} (\log(n p))^{-2} 
            \min\big[  r^{-5/2},(\kappa_3^*)^{-1}  r^{-7/2} \big].$
    \end{enumerate}
Then, with probability converging to $1$,     {estimating equations in steps 3 and 4 of Method~\ref{meth:split} have a unique solution and}
 \begin{equation}\label{eq:m-bound-eta-same}
    \begin{split}
    \|\tilde{\MM}-\MM^*\|_{\max}
    \lesssim   (\delta_2^*)^{-2} (\kappa_2^*)^2 \log^{2}(n p)r^{5/2}\Big[   \{ (p\wedge n)\pi\}^{-1/2} + (np)^{-1/2}e_{\MM,F}\Big].
    \end{split}
    \end{equation}
In particular, if we further assume that {$r\sim 1$}, then, the asymptotic regime requirements {\bf R5}, {\bf R6}, and {\bf R7'} can be simplified as $p\pi\gg(\log(np))^3$, $n\pi \gg (\log(np))^2$ and $(np)^{-1/2}e_{\MM,F}\ll (\log(np))^{-2}$, and we have that with probability converging to $1$,
$
	   \|\tilde{\MM}-\MM^*\|_{\max}
    \lesssim  (\log(np))^{2}\big[\{(n\wedge p)\pi\}^{-1/2}+(np)^{-1/2}e_{\MM,F}\big].
$
\end{theorem}
\begin{rmk}
There are two main differences between Theorem~\ref{thm:m-bound-no-splitting} and Theorem~\ref{thm:m-bound-same-eta}. First, the asymptotic requirement {\bf R7} has an extra factor $\pi^{1/2}$ when compared with {\bf R7'}. Second, the error rate \eqref{eq:asymp-bound-nosplit} has an extra ${\pi}^{-1/2}$ factor when compared with \eqref{eq:m-bound-eta-same}. Thus, when $\pi\ll 1$, Method~\ref{meth:nosplit} requires stronger regularity conditions and has a larger error rate.
Additional results under a more general asymptotic regime are provided in the supplementary material.
\end{rmk}
The following corollaries give sufficient conditions for {\bf R7'} to hold under specific GLFMs.
\begin{corollary}[Binomial  Model]\label{coro:binomial}
    Assume that $\lim_{n,p\to\infty}\pr(\|\hat{\MM}_{\gN_k\cdot}^{(k)}-\MM^*_{\gN_k\cdot}\|_{F}\leq e_{\MM,F}) =1$ for some non-random $e_{\MM,F}$ $(k=1,2)$.
     In addition, assume that data follow a binomial factor model and that asymptotic requirements {\bf R2 - R4} in Theorem~\ref{thm:m-bound-no-splitting} and {\bf R5B, R6B, R8B, R9B} in Corollary~\ref{coro:binomial-nosplit} hold as $n,p\to\infty$. Then, \eqref{eq:m-bound-eta-same} holds if there is a constant $\epsilon_0>0$ such that the following asymptotic regime holds:
    \begin{enumerate}
     \item[R7'B] $ (np)^{-1/2}e_{\MM,F}
            \ll  (n\wedge p)^{-\epsilon_0}  r^{ -7/2}.$
\end{enumerate}
\end{corollary}

\begin{corollary}[Normal Model]\label{coro:linear}
    Assume that $\lim_{n,p\to\infty}\pr(\|\hat{\MM}_{\gN_k\cdot}^{(k)}-\MM^*_{\gN_k\cdot}\|_{F}\leq e_{\MM,F}) =1$ for some non-random $e_{\MM,F}$ $(k=1,2)$.
    In addition, assume that data follow a normal factor model and that asymptotic requirements {\bf R1 - R4} in Theorem~\ref{thm:m-bound-no-splitting} and {\bf R5N, R6N} in Corollary~\ref{coro:linear-nosplit} hold as $n,p\to\infty$. Then, \eqref{eq:m-bound-eta-same} holds under the following asymptotic regime:
    \begin{enumerate}
     \item[R7'N] $ (np)^{-1/2}e_{\MM,F}
            \ll  (\log(n p))^{-2}  r^{ -5/2}.$
\end{enumerate}
\end{corollary}

\begin{corollary}[Poisson  Model]\label{coro:Poisson}
Assume that $\lim_{n,p\to\infty}\pr(\|\hat{\MM}_{\gN_k\cdot}^{(k)}-\MM^*_{\gN_k\cdot}\|_{F}\leq e_{\MM,F}) =1$ for some non-random $e_{\MM,F}$ $(k=1,2)$.
In addition, assume that data follow a Poisson factor model and that asymptotic requirements {\bf R2 - R4} in Theorem~\ref{thm:m-bound-no-splitting}, {\bf R5B, R6B} in Corollary~\ref{coro:binomial-nosplit} ,{\bf R7'B} in Corollary~\ref{coro:binomial}, and {\bf R10P} in Corollary~\ref{coro:Poisson-nosplit} hold as $n,p\to\infty$. Then, \eqref{eq:m-bound-eta-same} holds.
\end{corollary}
Remark~\ref{remark:corollaries} still applies to Corollary~\ref{coro:binomial} -- Corollary~\ref{coro:Poisson}, except that now we have a better rate when $\pi$ is close to zero.

\section{Simulation Study}\label{sec:sim}

We evaluate the proposed methods via a simulation study. Eight estimation procedures  are considered as listed in Table~\ref{tab:procedures}. These procedures are applied under 24 simulation settings, where $n$, $p$, $r$, $\pi_{\max} = \pi_{\min} = \pi$, and variable types are varied. 
Settings 1-6 are listed in Table~\ref{tab:settings}. The rest of the settings and additional details on data generation can be found in the supplementary material. 
For each simulation setting, 100 simulations are conducted.

\begin{table}
    \centering
    \begin{tabular}{c|cc|c|cc}
    \hline
      Procedure   & Initial estimator & Refinement method & Procedure   & Initial estimator & Refinement method   \\
      \hline
      1     & NBE &             & 5 & CJMLE & \\
      2     & NBE & Method 1    & 6 & CJMLE & Method 1\\
      3     & NBE & Method 2    & 7 & CJMLE & Method 2\\
      4     & NBE & Method 2' (5 runs)   & 8 & CJMLE & Method 2' (5 runs)\\
      \hline
    \end{tabular}
    \caption{Estimation procedures compared in a simulation study.}
    \label{tab:procedures}
\end{table}

\begin{table}
    \centering
    \begin{tabular}{c|ccccc|c|ccccc}
    \hline
      Setting   & $n$ & $p$ & $r$ & $\pi$ &Variable Types&Setting   & $n$ & $p$ & $r$ & $\pi$&Variable Type    \\
      \hline
        1& 400 & 200 & 3 & 0.6&O & 4& 400 & 200 & 3& 0.2&O \\
        2& 800 & 400 & 3 & 0.6&O & 5& 800 & 400 & 3& 0.2&O\\
        3& 1600 & 800 & 3 & 0.6&O &6& 1600 & 800 & 3& 0.2&O\\
      \hline
    \end{tabular}
    \caption{Simulation settings. `Variable type = O' indicates all the variables are ordinal (with $k_j = 5$), for which the Binomial model is assumed. }
    \label{tab:settings}
\end{table}

The procedures are evaluated under two loss functions, the scaled Frobenius norm $\Vert \hat \MM - \MM^*\Vert_F/\sqrt{np}$ and the max norm $\Vert \hat \MM - \MM^*\Vert_{\max}$. The results for Settings 1-6 are given in Figures \ref{fig:sim1-3} and \ref{fig:sim4-6}, and those for the other settings  show similar patterns and are given in the supplementary material. First, for each procedure and given $r$ and $\pi$, both the scaled Frobenius norm and the max norm decay as $n$ and $p$ grow simultaneously. Second, comparing the two figures, we see that the error rates are larger under Settings 4-6 than those under Settings 1-3 given the same $n, p$, and $r$, as the proportion of missing entries is higher under Settings 4-6.  
Third, Procedure 1 (i.e., NBE with no refinement) has larger error rates than its refined versions (Procedures 2-4), suggesting that the refinement procedures reduce the error of the initial NBE. 
Fourth, we see that Procedures 5 and 6 perform similarly, which is expected as they are asymptotically equivalent, as discussed in Remark~\ref{rmk:cjmle}. 
Fifth,  comparing Procedures 2 and 6, we see that the refined NBE and the refined CJMLE have very similar performance. Similar patterns are observed when comparing Procedures 3 and 7 and when comparing Procedures 4 and 8. At first glance, it may seem a little counter-intuitive. According to Theorems~\ref{thm:m-bound-no-splitting} and \ref{thm:m-bound-same-eta}, the error in the max norm of a refined estimator is upper bounded by the error in the scaled Frobenius norm of its initial estimator, and thus, we would expect the CJMLE-based refinements to have smaller errors in the max norm than the NBE-based refinements. 
The pattern under the current settings may be explained by the SVD steps in Methods 1, 2, and 2' that project the initial estimate to the space of rank-$r$ matrices. Under these settings, the initial NBE after projection tends to approximate the CJMLE.
{We note that this is not always the case under other settings. Under settings 23 and 24 (see their results in the supplementary material), the CJMLE tends to outperform the projected NBE, and thus, the CJMLE-based refinements tend to outperform the NBE-based refinements.} 
Finally, comparing  within Procedures 2-4 and comparing within Procedures 6-8, we see that Method 1 leads to better empirical performance regardless of the value of $\pi$, even though 
Method~\ref{meth:split} has a  faster theoretical convergence speed when $\pi$ approaches 0.  {We conjecture 
that for CJMLE and NBE, the
resulting $\hat\AAA$ in Step 2 of Method~\ref{meth:nosplit} does not have a high dependence with any rows of $\OOO$
when $\omega_{ij}$s are uniformly sampled, and thus, the upper bound in \eqref{eq:asymp-bound-nosplit} may be improved in this case.} We also observe that 
Method 2' outperforms Method~\ref{meth:split} through 
aggregating results from multiple runs of Method~\ref{meth:split}. By running Method~\ref{meth:split} five times, Method 2' has a similar performance as Method 1.

\begin{figure}[h] 
    \centering
    \includegraphics[scale=0.4]{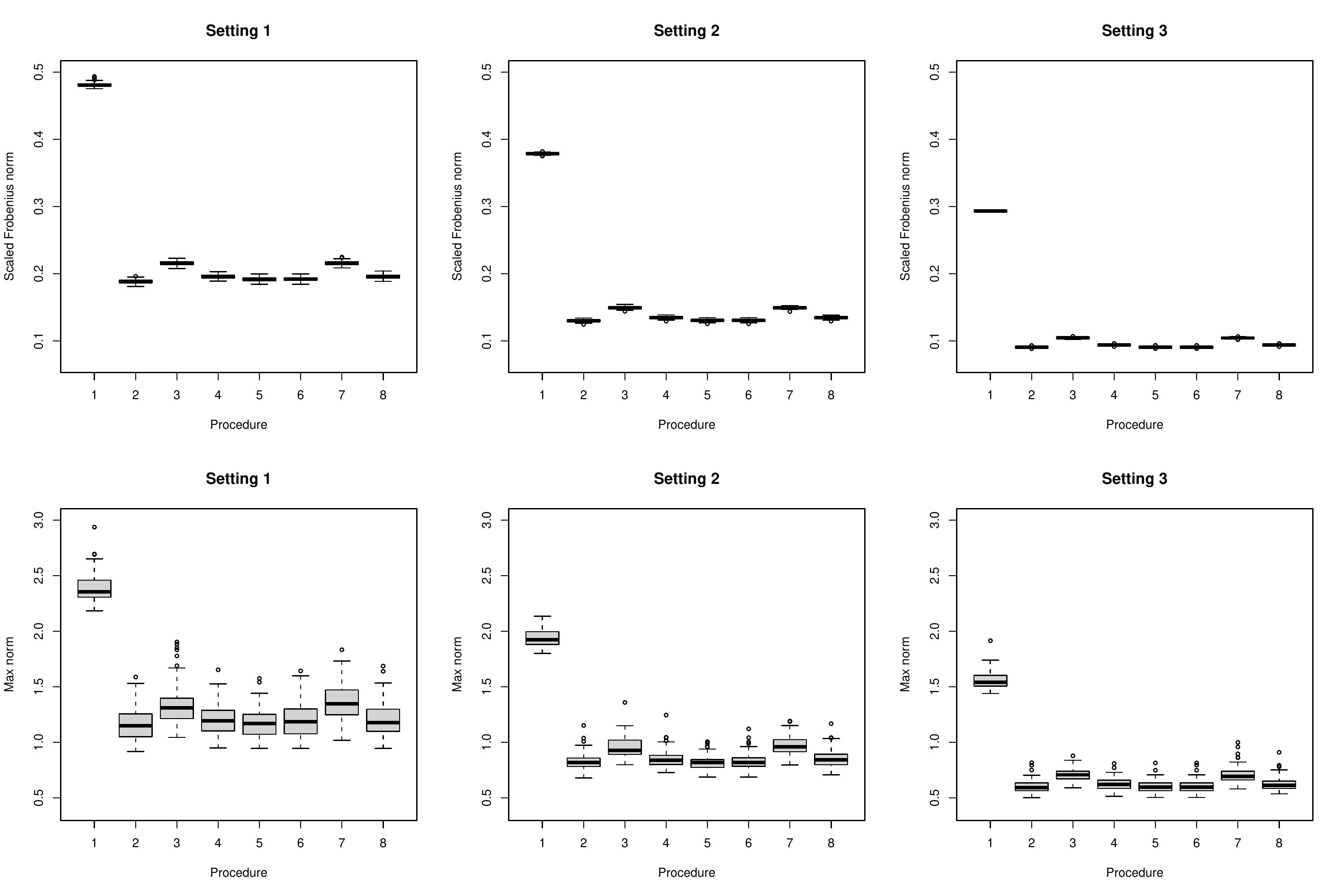}
    \caption{Results from Simulation Settings 1-3. The panels on the first row show the results based on the scaled Frobenius norm, and those on the second row show the results based on the max norm. In each panel,  the box plots show the results of the  eight procedures in Table~\ref{tab:procedures}, each constructed from 100 independent simulations.}
    \label{fig:sim1-3}
\end{figure}

\begin{figure}[h] 
    \centering
    \includegraphics[scale=0.4]{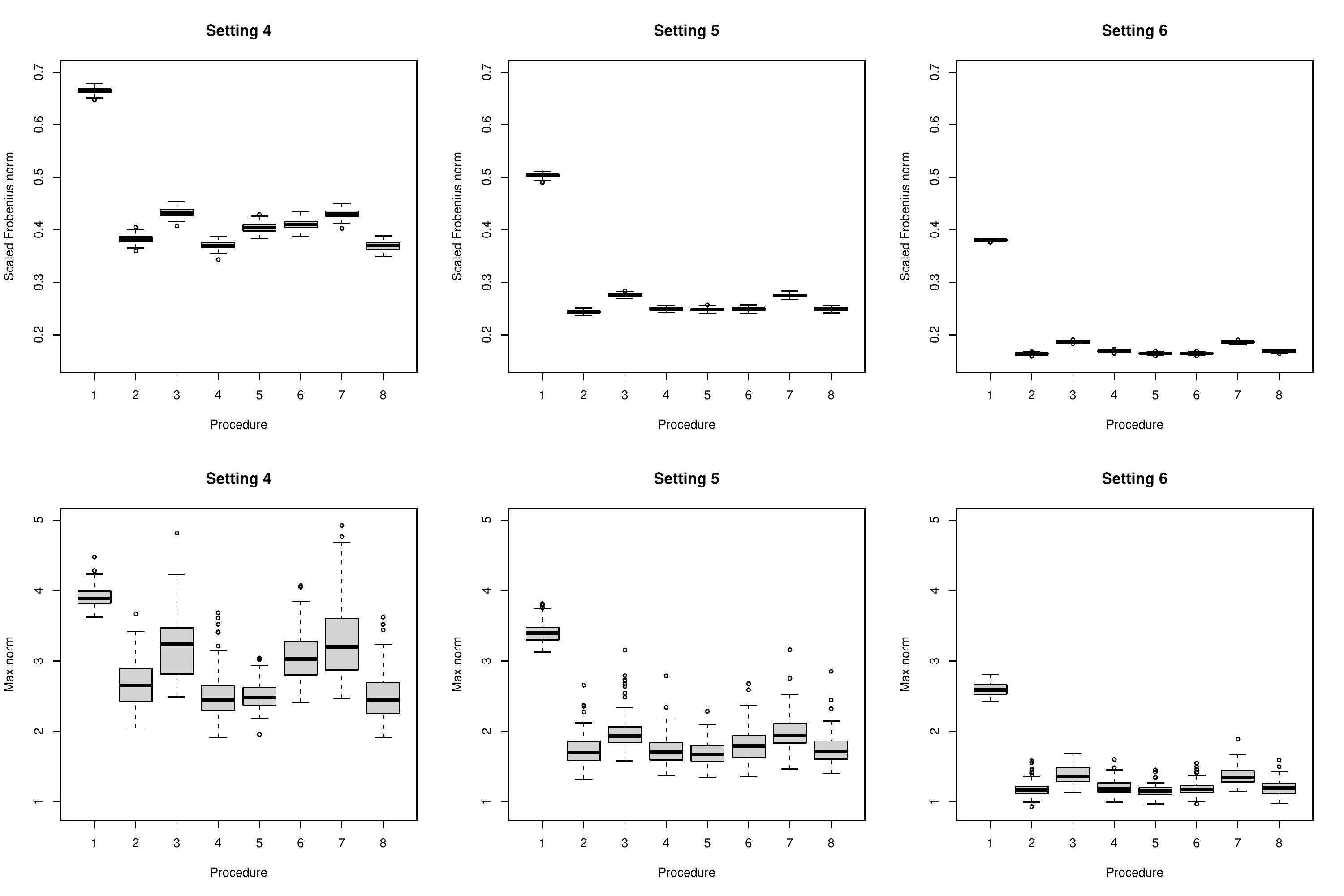}
    \caption{Results from Simulation Settings 4-6. The plots can be interpreted similarly as those in Figure~\ref{fig:sim1-3}.}
    \label{fig:sim4-6}
\end{figure}

\section{Real Data Examples}\label{sec:real}

\subsection{Collaborative Filtering}
We apply the proposed method to a MovieLens dataset for movie recommendation \citep{harper2015movielens}\footnote{The two real datasets in this article are publicly available
in the MovieLens Database at \url{https://grouplens.org/datasets/movielens/100k/} and the OECD PISA Database at \url{https://www.oecd.org/pisa/data/2018database/}. The computation code is available at \url{https://github.com/yunxiaochen/MatrixCompletion_MixedData}.
}. The dataset contains 943 users' ratings on 1,682 movies. Only 6.3\% of the data entries are observed. For each movie, the raw ratings take integer values from 1 to 5. We transform the values from 0 to 4, and then apply the binomial factor model with $k_j=4$ for all $j$. The goal is to predict the unobserved entries for movie recommendations.

\begin{table}[]
    \centering
    \begin{tabular}{c|ccccccccc}
    \hline 
   Rank & 1 & 2 & 3 & 4 & 5 & 6 & 7 & 8 \\
    \hline
    1 &-48928& -49247 & -49397 & -49253 & -49256& -49266 & -49266 & -49163  \\
    2 &-53201& -49505& -49767 & -48875 & -48437& -48493 & -48654& -48341   \\
    3 &-56091 & -49284 & -49754 & -48570 & -49022 & -49217 & -48837 & {\bf -48207}   \\
    4 & -56235 & -49633 & -50037 & -48611 & -51192 & -51986 & -49174& -48271  \\
    \hline
    \end{tabular}
    \caption{Test-set log-likelihoods for the MovieLens data. The eight procedures are listed in Table \ref{tab:procedures}.}
    \label{tab:movie}
\end{table}The eight procedures in Table~\ref{tab:procedures} are considered, with candidate rank $r = 1, 2, 3$, and 4. To evaluate the procedures, we split the data into training and test datasets, where the training and test sets contain 80\% and 20\% of the observed entries, respectively. We estimate the $\MM$ matrix using the training set and then evaluate the prediction accuracy by the test-set log-likelihood at the estimated $\MM$. A larger log-likelihood function value implies a higher prediction accuracy. The results are given in Table~\ref{tab:movie}.  The refinement methods improve the test-set log-likelihood of the NBE when $r = 2, 3, 4$ but not when $r = 1$, likely due to the rank-one model being too restrictive for the current data. Turning to the results from the CJMLE and its refinements, 
we see that Procedures 5 and 6 tend to perform similarly, likely due to the asymptotic equivalence between the CJMLE and its refinement by Method 1. We also see that Procedure 8, which is a refinement of CJMLE by Method 2', tends to improve the 
test-set log-likelihood of CJMLE under all values of $r$. Procedure 7 also performs fine, despite its relatively high variance brought by performing data splitting only once in Method 2. The good performance of Procedures 7 and 8 is likely due to that the distribution of the data missingness indicators $\omega_{ij}$ is far from a uniform distribution. Instead, their distribution likely depends on the true signal matrix (i.e., people may be more likely to have watched movies that they like), which may lead to dependence between the initial estimate $\hat \AAA$ and some rows of $\OOO$ when data splitting is not performed. Such dependence leads to a larger estimation error. 
The largest 
test-set log-likelihood is given by Procedure 8 (i.e., CJMLE refined by Method 2') when $r=3$.

\subsection{Large-scale Assessment in Education} 

We apply the proposed method to data from the 2018 Program for International Student Assessment (PISA; \citealp{organisation2019pisa}), a large-scale international educational survey operated by the Organization for Economic Co-operation and Development (OECD). We consider a subset of the PISA 2018 dataset, containing 9,970 students' responses to 415 assessment items.
The students were from 37 OECD countries. The 415 assessment items measure four knowledge domains, including mathematics, science, reading, and global competence. A matrix sampling design is adopted in PISA 2018, under which each student was only assigned a subset of assessment items. Consequently, only
15.5\% of the entries are observed in the dataset. 
Under this matrix sampling design, 
it is not sensible to directly compare students' performance based on their total scores, as the students answered different assessment items, and  the items measure different knowledge domains and are not equally difficult. 
Among these items, 396 items are dichotomously scored, and 19 items have score levels $0,1$ and $2$.  The goal is to predict students' performance on the items they did not receive in order to compare the  performance based on the entire set of items. 

We apply the binomial factor model. Similar to the above analysis, we split 80\% and 20\% of the data into training and test sets and evaluate the prediction accuracy by the test-set log-likelihood. The eight procedures in Table~\ref{tab:procedures} are considered, with candidate rank $r = 1, 2, 3,$ and 4. The results are given in Table~\ref{tab:pisa}. First, the refinement methods tend to improve the test-set log-likelihood given by the NBE, except for the case when $r=1$. The results given by the CJMLE and its refinement by Method 1 are similar under all values of $r$. They tend to be better than the refinements given by Methods 2 and 2', likely due to that the variance brought by data splitting is high in this analysis. Second, the largest test-set log-likelihood is achieved by the CJMLE when the rank  $r=2$.
The test-set log-likelihoods of the CJMLE and its refinement by Method 1  are similar when $r=2$, and they tend to substantially outperform the rest. In the analysis of PISA data, each of the knowledge domains is believed to correspond  to at least one latent factor. Thus, four- or higher-dimensional factor models are typically adopted to jointly model the item responses \citep[see Chapter 9, page 22,][]{organisation2019pisa2}. Our results suggest that a lower-dimensional factor model may have better prediction performance, though not necessarily have better performance in terms of statistical inference and interpretation. This finding is closely related to the discussion in psychometrics regarding the value of subscores \citep{haberman2008can}.

\begin{table}[]
    \centering
    \begin{tabular}{c|ccccccccc}
    \hline
   Rank & 1 & 2 & 3 & 4 & 5 & 6 & 7 & 8 \\
    \hline
    1 & -67205& -67938& -67958& -67921& -67587& -67516& -68204& -68140 \\
    2 & -71620& -68556& -68733& -67749& {\bf -63250} & -63313& -64914& -64842\\
    3 &-75816 & -70092 & -70067 & -69151& -65476& -65370& -68611& -67693  \\
    4 & -77632 &-72365& -72238& -71640& -72320& -72648& -79466& -75989 \\
    \hline
    \end{tabular}
    \caption{Test-set log-likelihoods for the PISA data. The eight procedures are listed in Table \ref{tab:procedures}.}
    \label{tab:pisa}
\end{table}

\section{Discussions}\label{sec:diss}

This paper concerns matrix completion for mixed data under a GLFM framework. It proposes entrywise consistent methods for estimating GLFMs based on a partially observed data matrix. Probabilistic error bounds are established for the matrix max norm under sensible asymptotic regimes (see Section~\ref{sec:theo}), and they are extended under a more general asymptotic regime in the supplementary material. These error bounds imply the entrywise consistency and further  characterize the asymptotic behaviors of the proposed methods. With these error bounds, optimal results are established under suitable asymptotic regimes. A simulation study shows that for the refined estimators, the error in matrix max norm decays towards zero as $n$ and $p$ grow simultaneously. It also shows 
that the performance of the NBE can be substantially improved by running the proposed refinement procedures. In contrast, the CJMLE can hardly be improved as it is already equivalent to a refined estimator. The simulation results further suggest that Method 2' can improve the accuracy of Method 2 by running this data-splitting procedure
multiple times and aggregating the results. However, although a smaller upper bound is proven for the error rates of Methods 2 and 2' when the missing rate is close to 1, these procedures do not outperform Method 1 - the refinement procedure without data splitting - under our simulation settings.  This phenomenon is likely due to that 
the probability of the worst-case scenario occurring for Method 1 is close to zero
under the current simulation settings, in which case the upper bound in \eqref{eq:asymp-bound-nosplit} may be improved. The proposed procedures are applied to two real data examples, one on movie recommendation and the other on large-scale educational assessment. For the movie recommendation example, the best predictive model is a rank-three model obtained by 
refining the CJMLE with Method 2'. For the educational assessment example, a rank-two model given by the CJMLE turns out to be the most predictive one.

The current work can be extended in several directions. First, some popular factor models, such as the probit model for binary data considered in \cite{davenport20141}, are not   exponential family GLFMs. We believe that our refinement procedures and their theory can be extended to many other models beyond exponential family GLFM. This is because the theoretical properties of these procedures mainly rely on the convexity of the loss function with respect to $\MM$, which still holds under many other non-linear factor models. Second, the optimal rate for estimating GLFMs  is worth future investigation. We currently do not know whether our upper bounds are minimax optimal when the dimension $r$ diverges.  Sharp lower bounds need to be developed to answer this question. {Future research is also needed to investigate whether the error bound for Method 1 can be improved. If not, further simulation studies are needed to find out settings under which Methods 2 and 2' outperform Method 1. }

\clearpage

{\bf\LARGE Supplement Material for ``A Generalized Latent Factor Model Approach to Mixed-data Matrix Completion with Entrywise Consistency''}

\appendix
\section{Proof of Theorem~\ref{thm:m-bound-same-eta} and additional theoretical results for Method~\ref{meth:split} with data splitting}

\sloppy In this section, we obtain the error bound for $\|\tilde{\MM}-\MM^*\|_{\max}\leq \max(\|\tilde{\TTT}_{\gN_1}(\tilde{\AAA}^{(1)})^{T}-\MM^*_{\gN_1\cdot}\|_{\max}, \|\tilde{\TTT}_{\gN_2}(\tilde{\AAA}^{(2)})^{T}-\MM^*_{\gN_2\cdot}\|_{\max}) $. We will provide detailed analysis  for $\|\tilde{\TTT}_{\gN_1}(\tilde{\AAA}^{(1)})^{T}-\MM^*_{\gN_1\cdot}\|_{\max}$. The analysis of $\|\tilde{\TTT}_{\gN_2}(\tilde{\AAA}^{(2)})^{T}-\MM^*_{\gN_2\cdot}\|_{\max}$ is similar and is thus omitted. For the ease of presentation, we drop the superscript $(1)$  in $\hat{\AAA}^{(1)}$ when the context is clear.
Recall that $\MM^*$ has the SVD $\MM^*=\UU_r^*\DD_r^*(\VV_r^*)^T$  where $\UU_r^*\in\bR^{n\times r}$, $\VV_r^*\in \bR^{p\times r}$ denote the left and right singular matrices, and $\DD_r^*=\text{diag}(\sigma_1(\MM^*),\cdots,\sigma_r(\MM^*))$.

The rest of the section is organized as follows. In Section~\ref{sec:hat-A-split}, we obtain an error bound for $\|\hat{\AAA}-\AAA^*\|_F$ where $\AAA^* = \VV_r^*\hat{\PP}$ for a carefully chosen orthogonal matrix $\hat{\PP}$. In Section~\ref{sec:proof-non-prob-split}, we provide non-asymptotic and non-probabilistic bounds for solutions to the non-linear estimation equations used in Step 3 and 4 in the proposed Method~\ref{meth:split}. In Section~\ref{sec:proof-prob-bounds-split}, we obtain non-asymptotic probabilistic bounds for terms involved in Section~\ref{sec:proof-non-prob-split}. In Section~\ref{sec:proof-asymptotic}, we put together results in Sections~\ref{sec:hat-A-split} -- \ref{sec:proof-prob-bounds-split} and obtain asymptotic error bounds for $\|\tilde{\TTT}_{\gN_2}-\TTT^*\|_{2\to\infty}$ (Lemma~\ref{lemma:asymptotic-1st-theta}),  $\|\tilde{\AAA}-\AAA^*\|_{2\to\infty}$ (Lemma~\ref{lemma:2nd-a-asymptotics}), and $\|\tilde{\TTT}_{\gN_1}(\tilde{\AAA}^{(1)})^{T}-\MM^*_{\gN_1\cdot}\|_{\max}$ (Lemma~\ref{lemma:asymptotic-m-max}) where $\TTT^*=\UU_r^*\DD_r^*\hat{\PP}$. Finally, we provide additional theoretical results for Method~\ref{meth:split} in Section~\ref{sec:additional-split} and the proof of Theorem~\ref{thm:m-bound-same-eta} in Section~\ref{sec:proof-asymptotic-simple}.

{Throughout the analysis, for real number operators, we calculate multiplication and division before the max and min operators (`$\vee$' and $`\wedge'$) unless otherwise specified. For example, $u(x y \vee z/w) =u\max(xy, z/w)$ for real numbers $x,y,u,w,z$.
For two events $A$ and $B$, we say `event $A$ has probability at least $1-\epsilon$ on  event $B$', if $\pr(A^c\cap B)\leq \epsilon$. Note that $\pr(A)\geq 1- \epsilon-\pr(B^c)$ in this case.
}
\subsection{Error Analysis for $\hat{\AAA}$}\label{sec:hat-A-split}
In this section, we provide an error bound for $\hat{\AAA}$ given an error bound for $\hat{\MM}_{\gN_1\cdot}$.
\begin{lem}\label{lemma:wedin}
Let $\psi_r=\sigma_r(\MM^*_{\gN_1\cdot})\wedge \sigma_r(\MM^*_{\gN_2\cdot})$ and $\psi_1 =\sigma_r(\MM^*_{\gN_1\cdot})\vee \sigma_r(\MM^*_{\gN_2\cdot})$.
If $\|\hat{\MM}_{\gN_1\cdot}-\MM^*_{\gN_1\cdot}\|_2 \leq 2^{-1}\psi_r$, $\|\VV_r^*\|_{2\to\infty}\leq C_2$ and
   $\text{rank}(\MM^*)=r$, then  there exists an orthogonal matrix $\hat{\PP}\in \bR^{r\times r}$ satisfying
    \begin{equation}
        \|\hat{\AAA}-\VV_r^*\hat{\PP}\|_F\leq 8 \psi_r^{-1}\|\hat{\MM}_{\gN_1\cdot}-\MM^*_{\gN_1\cdot}\|_F^2.
    \end{equation}
\end{lem}
\begin{proof}[Proof of Lemma~\ref{lemma:wedin}]
    According to Weyl’s inequality and the assumption that  $\|\hat{\MM}_{\gN_1\cdot}-\MM_{\gN_1\cdot}^*\|_2 \leq 2^{-1}\psi_r$, $\sigma_r(\hat{\MM}_{\gN_1\cdot})\geq \sigma_r(\MM^*_{\gN_1\cdot})-\|\hat{\MM}_{\gN_1\cdot}-\MM^*_{\gN_1\cdot}\|_2\geq 2^{-1}\sigma_r(\MM^*_{\gN_1\cdot})\geq  2^{-1}\psi_r$.
Thus the gaps of singular value satisfies
\begin{equation}
    \min\Big[\min_{1\leq i\leq r, j> r}\big\{\sigma_i(\hat{\MM}_{\gN_1\cdot})-\sigma_j(\MM_{\gN_1\cdot}^*)\big\},\min_{1\leq i\leq r}\sigma_i(\hat{\MM}_{\gN_1\cdot})\Big]=\min\Big\{\sigma_r(\hat{\MM}_{\gN_1\cdot}), \sigma_r(\MM_{\gN_1 \cdot}^*)\Big\}\geq 2^{-1}\psi_r.
\end{equation}
Let $\VV_{r,{\gN_1\cdot}}^*\in  \bR^{p\times r}$ be the right singular value matrix corresponding to the top-$r$ singular values of $\MM_{{\gN_1\cdot}}^*$ and 
\begin{equation}
    {\PP}^{\dagger}=\argmin_{\PP\in \mathcal{O}_r}\|\hat{\VV}_r-\VV_{r,{\gN_1\cdot}}^*\PP\|_F,
\end{equation}
where $\mathcal{O}_r$ denotes the set of all $r\times r$ orthogonal matrices.
According to the above equations and the Wedin's sine angle theorem \citep{wedin1972perturbation}, 
\begin{equation}\label{eq:wedin}
	\|\hat{\VV}_r-\VV_{r,{\gN_1\cdot}}^*\PP^{\dagger}\|_F=\inf_{\PP\in\mathcal{O}_r}\|\hat{\VV}_r-\VV_{r,{\gN_1\cdot}}^*\PP\|_F\leq 
	\frac{2\|\hat{\MM}_{\gN_1\cdot}-\MM^*_{\gN_1\cdot}\|_F}{\sigma_r(\hat{\MM}_{\gN_1\cdot})}\leq\frac{4\|\hat{\MM}_{\gN_1\cdot}-\MM^*_{\gN_1\cdot}\|_F}{\psi_r}.
\end{equation}
On the other hand, since $\sigma_r(\MM^*_{\gN_1\cdot})\geq \psi_r>0$, the column space of $(\MM_{\gN_1\cdot}^*)^T$ 
is the same as the columns space of  $\VV_{r,{\gN_1\cdot}}^*$ and that of $\VV_r^*$. This implies that there exists an orthogonal matrix $\bar{\PP}\in\bR^{r\times r}$ such that $\VV_{r,{\gN_1\cdot}}^*=\VV_r^*\bar{\PP}$, which further implies that for the orthogonal matrix 
\begin{equation}\label{eq:p-hat}
    \hat{\PP} = \bar{\PP}\PP^{\dagger}, 
\end{equation}
we have
$
    \|\hat{\VV}_r-\VV_{r}^*\hat{\PP}\|_F\leq 4\psi_r^{-1}\|\hat{\MM}_{\gN_1\cdot}-\MM^*_{\gN_1\cdot}\|_F.
$
According to Method~\ref{meth:split},  $\hat{\AAA}$ is the projection of $\hat{\VV}_r$ to the set $\{\AAA\in \bR^{p\times r}:\|\AAA\|_{2\to\infty}\leq C_2\}$ and $\|\VV^*_r\hat{\PP}\|_{2\to\infty}=\|\VV^*_r\|_{2\to\infty}\leq C_2$. Thus,
\begin{equation}
    \|\hat{\AAA}-\VV_r^*\hat{\PP}\|_F \leq  \|\hat{\AAA}-\hat{\VV}_r\|_F + \|\hat{\VV}_r- \VV^*_r\hat{\PP}\|_F\leq 2 \|\hat{\VV}_r- \VV^*_r\hat{\PP}\|_F \leq 8 \psi_r^{-1}\|\hat{\MM}_{\gN_1\cdot}-\MM^*_{\gN_1\cdot}\|_F.
\end{equation}
\end{proof}
The next lemma is obtained by directly applying Lemma~\ref{lemma:wedin}.
\begin{lem}\label{lemma:em-to-ea}
    If $\lim_{n,p\to\infty }\pr(\|\hat{\MM}_{\gN_1\cdot}-\MM_{\gN_1\cdot}^*\|_F\geq e_{\MM,F})= 0$, $e_{\MM,F}$ is a non-random number (depending on  $n$ and $p$),   
    $\|\VV_r^*\|_{2\to\infty}\leq C_2$ and $e_{\MM,F}\leq 2^{-1}\psi_r$, then
    \begin{equation}
        \lim_{n,p\to\infty}\pr(\|\hat{\AAA}-\VV_r^*\hat{\PP}\|_F\geq e_{\AAA,F})= 0,
    \end{equation}
    where $\hat{\PP}$ is defined in \eqref{eq:p-hat} and $e_{\AAA,F} = 8\psi_r^{-1}e_{\MM,F}$.
\end{lem}

\subsection{Non-probabilistic bounds for solutions to estimating equations}\label{sec:proof-non-prob-split}
Recall that for each $i\in[n]$, the partial score function corresponding to $\ttt_i$ is
\begin{equation}
    S_{1,i}(\ttt_i;\AAA)
    := \frac{\partial}{\partial \ttt_i}\ell(\TTT,\AAA)=\phi^{-1}\sum_{j=1}^p \omega_{ij} \{y_{ij}-b'(\aaaa_j^T\ttt_i)\}\aaaa_j
\end{equation}

The next lemma provides a non-probabilistic bound for the solution to the partial score equation $S_{1,i}(\ttt_i,\AAA)=\mathbf{0}_r$.

\begin{lem}
\label{lemma:finite-bound}
Let $\TTT^*\in\bR^{n\times r}$ and $\AAA^*\in\bR^{p\times r}$ be such that $\MM^* = \TTT^*(\AAA^*)^T$ and $\ZZ=(z_{ij})$ with $z_{ij}=y_{ij}-b'(m^*_{ij})$ and $\dOmegai:=\text{diag}(\omega_{i1},\cdots,\omega_{ip})$.
	If $\|\TTT^*\|_{2\to\infty}\leq C_1$,  $\|\AAA^*\|_{2\to\infty},\|\AAA\|_{2\to\infty}\leq C_2$ and there exists $\xi>0$ such that 
	\begin{equation}\label{eq:finite-bound-condition}
\begin{split}
		&2\sigma^{-1}_r({\II}_{1,i}(\AAA))\big\{\| \ZZ_{i\cdot}\dOmegai\AAA\|+\|\BB_{1,i}(\AAA)\|+\beta_{1,i}(\AAA) \kappa_3\big(C_{2}(C_1+\xi)\big)\big\}\\
		\leq &  \xi\leq 2^{-1}\{\gamma_{1,i}(\AAA) \kappa_3\big(C_{2}(C_1+\xi)\}^{-1}\sigma_K({\II}_{1,i}(\AAA)),
\end{split}
\end{equation}
where we define $\ZZ_{i\cdot} = (z_{ij})_{j\in [p]}\in \mathbb{R}^{1\times p}$, 
\begin{equation}\label{eq:U}
	\BB_{1,i}(\AAA) := \sum_{j=1}^p \omega_{ij} b''(m^*_{ij})
		\aaaa_j(\aaaa_j-\aaaa_j^*)^T\ttt^*_i \in {\bR}^r,
\end{equation}
\begin{equation}\label{eq:II}
	{\II}_{1,i}(\AAA):=\sum_{j=1}^p \omega_{ij} b''\big(m_{ij}^*\big)\aaaa_j(\aaaa_j)^T,
\end{equation}
and
\begin{equation}%
	\beta_{1,i}(\AAA):=  \sup_{\|\uu\|=1} \sum_{j}\omega_{ij} ((\aaaa_j-\aaaa_j^*)^T\ttt^*_i)^2|\aaaa_j^T\uu| \text{ and }	\gamma_{1,i}(\AAA) := \sup_{\|\uu\|=1} \sum_j \omega_{ij}|\aaaa_j^T\uu|^3,
\end{equation}
	then, there is $\tilde{\ttt}_i$ such that $
		\|\tilde{\ttt}_i-\ttt^*_i\|\leq \xi$ and $S_{1,i}(\tilde{\ttt}_i;\AAA)=\mathbf{0}$. 
\end{lem}
\begin{proof}[Proof of Lemma~\ref{lemma:finite-bound}]
Let $\ttt$ be a vector such that $\|\ttt-\ttt^*_i\|=\xi$ and let $m_{ij}=\aaaa_j^T\ttt_i$.
Consider the Taylor expansion of $\phi S_{1,i}(\ttt;\AAA)$, 
\begin{equation}\label{eq:taylor}
	\begin{split}
		\phi S_{1,i}(\ttt;\AAA)
		= & \sum_j \omega_{ij} (y_{ij}-b'(m^*_{ij}))\aaaa_j - \sum_j \omega_{ij}(b'(m_{ij})-b'(m_{ij}^*))\aaaa_j\\
		= & \AAA^T \dOmegai \ZZ_{i\cdot}^T - \sum_j 
		\omega_{ij}b''(m^*_{ij}) (m_{ij}-m_{ij}^*)\aaaa_j-2^{-1}\sum_j \omega_{ij} b^{(3)}(\tilde{m}_{ij})(m_{ij}-m_{ij}^*)^2 \aaaa_j,
		\end{split}
		\end{equation}
		for some $\tilde{m}_{ij}$ between $m_{ij}^*$ and $m_{ij}$. Plugging  $m_{ij}-m_{ij}^*=\aaaa_j^T(\ttt-\ttt_i^*)+(\aaaa_j-\aaaa_j^*)^T\ttt_i^*$ into the above display, we obtain
		\begin{equation}
		\begin{split}
	\phi S_{1,i}(\ttt;\AAA)
		= & \AAA^T\dOmegai \ZZ^T_{i\cdot} - \sum_j \omega_{ij}b''(m^*_{ij}) \aaaa_j\aaaa_j^T (\ttt-\ttt_i^*) \\
		&- \sum_j \omega_{ij}b''(m^*_{ij})
		\aaaa_j(\aaaa_j-\aaaa_j^*)^T\ttt_i^*-2^{-1}\sum_j\omega_{ij} b^{(3)}(\tilde{m}_{ij})(m_i-m_{ij}^*)^2 \aaaa_j
		\end{split}.
\end{equation}
Multiplying $(\ttt-\ttt_i^*)^T$ on both sides, we obtain
\begin{equation}\label{eq:multiply}
	\begin{split}
		&\phi (\ttt-\ttt_i^*)^TS_{1,i}(\ttt;\AAA)\\
		= & (\ttt-\ttt_i^*)^T\AAA^T\dOmegai \ZZ^T_{i\cdot} - (\ttt-\ttt_i^*)^T \sum_j\omega_{ij} b''(m^*_{ij}) \aaaa_j\aaaa_j^T (\ttt-\ttt_i^*) \\& - (\ttt-\ttt_i^*)^T\sum_j\omega_{ij} b''(m^*_{ij})
		\aaaa_j(\aaaa_j-\aaaa_j^*)^T\ttt_i^*\\
		&-2^{-1}(\ttt-\ttt_i^*)^T\sum_j\omega_{ij} b^{(3)}(\tilde{m}_{ij})(m_i-m_{ij}^*)^2 \aaaa_j.
	\end{split}
\end{equation}
Recall that $\|\ttt-\ttt^*_i\|=\xi$. Using inequalities about matrix products and singular values, we have the following upper bounds for the first three terms on the right-hand side of the above display.
\begin{equation}\label{eq:non-random-temp1}
	|(\ttt-\ttt_i^*)^T\AAA^T\dOmegai \ZZ^T_{i\cdot}|\leq \xi \|\AAA^T\dOmegai \ZZ^T_{i\cdot}\| = \xi \| \ZZ_{i\cdot}\dOmegai\AAA\|,
\end{equation}
\begin{equation}\label{eq:non-random-temp2}
	-(\ttt-\ttt_i^*)^T \sum_j \omega_{ij} b''(m^*_{ij}) \aaaa_j\aaaa_j^T (\ttt-\ttt_i^*) \leq -\xi^2\sigma_r({\II}_{1,i}(\AAA)),
\end{equation}
where $\sigma_r({\II}_{1,i}(\AAA))$ denotes the $r$-th largest singular value of ${\II}_{1,i}(\AAA)$, and
\begin{equation}\label{eq:non-random-temp3}
	| (\ttt-\ttt_i^*)^T\sum_j \omega_{ij} b''(m^*_{ij})
		\aaaa_j(\aaaa_j-\aaaa_j^*)^T\ttt_i^*|= \|(\ttt-\ttt_i^*)^T \BB_{1,i}\|\leq \xi \|\BB_{1,i}\|.
\end{equation}
Now we analyze the last term $2^{-1}(\ttt-\ttt_i^*)^T\sum_j\omega_{ij} b^{(3)}(\tilde{m}_{ij})(m_i-m_{ij}^*)^2 \aaaa_j$. Note that $|\tilde{m}_{ij}|\leq |m_{ij^*}|\vee |m_{ij}|\leq (C_1+\xi)C_2$ and $m_{ij}-m_{ij}^*=\aaaa_j^T(\ttt-\ttt_i^*)+(\aaaa_j-\aaaa_j^*)^T\ttt_i^*$, we have
\label{note:3rd-moment}
\begin{equation}
	\begin{split}
		&2^{-1}(\ttt-\ttt_i^*)^T\sum_j b^{(3)}(\tilde{m}_{ij})(m_i-m_{ij}^*)^2 \aaaa_j\\
		\leq& 2^{-1}\kappa_3\big((C_1+\xi)C_2\big)\xi\sup_{\|\uu\|=1}\sum_{j}\omega_{ij} ((\aaaa_j-\aaaa_j^*)^T\ttt_i^*+\xi\aaaa_j^T\uu)^2|\aaaa_j^T\uu|\\
		\leq & \kappa_3\big((C_1+\xi)C_2\big)\big\{\xi \sup_{\|\uu\|=1} \sum_j\omega_{ij} ((\aaaa_j-\aaaa_j^*)^T\ttt_i^*)^2|\aaaa_j^T\uu| +\xi^3 \sup_{\|\uu\|=1} \sum_j\omega_{ij} |\aaaa_j^T\uu|^3\big\}\\
		= & \kappa_3\big((C_1+\xi)C_2\big)(\xi \beta_{1,i} +\xi^3 \gamma_{1,i}).
	\end{split}
\end{equation}
Combining the analysis with \eqref{eq:multiply}, \eqref{eq:non-random-temp1}, \eqref{eq:non-random-temp2}, and \eqref{eq:non-random-temp3}, we obtain
\begin{equation}
	\begin{split}
		(\ttt-\ttt_i^*)^T\phi S_{1,i}(\ttt;\AAA)		\leq & - \sigma_r({\II}_{1,i}(\AAA)) \xi^2 + \gamma_{1,i} \kappa_3\big((C_1+\xi)C_2\big)\xi^3 \\
	&	+ \big\{\| \ZZ_{i\cdot}\dOmegai\AAA\|+\|\BB_{1,i}\|+\beta_{1,i} \kappa_3\big((C_1+\xi)C_2\big) \big\}\xi.
	\end{split}
\end{equation}
Now, we view the right-hand side of the above inequality as a cubic function in $\xi$. For any cubic function $f(x)=-ax^2+bx^3+cx$ with $a,b,c>0$, it is easy to verify that if $2c/a\leq x\leq a/(2b)$, then $f(x)\leq 0$.
Applying this result, we can see that $\sup_{\|\ttt-\ttt^*_i\|= \xi}(\ttt-\ttt_i^*)^TS_{1,i}(\ttt;\AAA)\leq 0$, if the following inequalities hold:
\begin{equation}
\begin{split}
		&2\sigma^{-1}_K({\II}_{1,i}(\AAA))\big\{\| \ZZ_{i\cdot}\dOmegai\AAA\|+\|\BB_{1,i}(\AAA)\|+\beta_{1,i}(\AAA) \kappa_3\big((C_1+\xi)C_2\big) \big\}\\
		\leq &  \xi\leq 2^{-1}\{\gamma_{1,i} \kappa_3\big((C_1+\xi)C_2\big)\}^{-1}\sigma_r({\II}_{1,i}(\AAA)).
\end{split}
\end{equation}
According to Result 6.3.4 in \cite{ortega2000iterative},  $\sup_{\|\ttt-\ttt^*_i\|= \xi}(\ttt-\ttt_i^*)^TS_{1,i}(\ttt;\AAA)\leq 0$ implies that there is a solution $
S_{1,i}(\tilde{\ttt};\AAA)=\mathbf{0}$ satisfying $\|\tilde{\ttt}-\ttt^*_i\|\leq \xi$.
\end{proof}
Next, we simplify the result of Lemma~\ref{lemma:finite-bound} to obtain a more user-friendly version  in the next lemma.
\begin{lem}\label{lemma:finite-simplify-kappa}
Let $\TTT^*\in\bR^{n\times r}$ and $\AAA^*\in\bR^{p\times r}$ be such that $\MM^* = \TTT^*(\AAA^*)^T$ and $\ZZ=(z_{ij})$ with $z_{ij}=y_{ij}-b'(m^*_{ij})$. If %
{%
$\|\AAA^*\|_{2\to\infty}\leq C_2$ and $\|\AAA\|_{2\to\infty}\leq C_2$,} and
\begin{equation}\label{eq:condition}
\begin{split}
    	&\| \ZZ_{i\cdot}\dOmegai\AAA\|+\|\BB_{1,i}(\AAA)\|+\beta_{1,i}(\AAA) \kappa_3\big(3C_1C_2\big)\\
     \leq & \min\Big\{ 2^{-2}(\gamma_{1,i}(\AAA))^{-1} (\kappa_3\big(3C_1C_2\big))^{-1}\sigma^2_r({\II}_{1,i}(\AAA)), 2^{-1}\sigma_r({\II}_{1,i}(\AAA))C_1\Big\},
\end{split}
\end{equation}
	then, there is $\tilde{\ttt}_i$ such that  $S_{1,i}(\tilde{\ttt};\AAA)=\mathbf{0}$, and 
\begin{equation}%
\|\tilde{\ttt}_i-\ttt_i^*\|\leq 2\sigma^{-1}_r({\II}_{1,i}(\AAA))\big\{\| \ZZ_{i\cdot}\dOmegai\AAA\|+\|\BB_{1,i}(\AAA)\|+\beta_{1,i}(\AAA) \kappa_3\big(3C_1C_2\big)\big\}.
\end{equation}
Moreover, 
the solution $\tilde{\ttt}_i$ also satisfies $\|\tilde{\ttt}_i-\ttt_i^*\|\leq C_1$.
\end{lem}
\begin{proof}[Proof of Lemma~\ref{lemma:finite-simplify-kappa}]
	Let $\xi=2\sigma^{-1}_r({\II}_{1,i}(\AAA))\big\{\| \ZZ_{i\cdot}\dOmegai\AAA\|+\|\BB_{1,i}(\AAA)\|+\beta_{1,i}(\AAA) \kappa_3\big(3C_1C_2\big)\big\}$. By the assumption that $
	\| \ZZ_{i\cdot}\dOmegai\AAA\|+\|\BB_{1,i}(\AAA)\|+\beta_{1,i}(\AAA) \kappa_3\big(3C_1C_2\big)\leq 2^{-1}\sigma_r({\II}_{1,i})C_1,$
we have $\xi\leq C_1$. Thus,
\begin{equation}
	\kappa_3\big(C_{2}(C_1+\xi)\big)\leq \kappa_3\big(3C_1C_2\big).
\end{equation}
This implies
\begin{equation}
\begin{split}
    	&2\sigma^{-1}_r({\II}_{1,i}(\AAA))\big\{\| \ZZ_{i\cdot}\dOmegai\AAA\|+\|\BB_{1,i}(\AAA)\|+\beta_{1,i}(\AAA) \kappa_3\big(C_{2}(C_1+\xi)\big)\big\}\\
    	\leq & 2\sigma^{-1}_r({\II}_{1,i}(\AAA))\big\{\| \ZZ_{i\cdot}\dOmegai\AAA\|+\|\BB_{1,i}(\AAA)\|+\beta_{1,i}(\AAA) \kappa_3\big(3C_1C_2\big)\big\}.
\end{split}
\end{equation}
Because the right-hand side of the above inequality equals $\xi$, it is simplified as
\begin{equation}\label{eq:check-cond1-finite-bound}
    	2\sigma^{-1}_r({\II}_{1,i}(\AAA))\big\{\| \ZZ_{i\cdot}\dOmegai\AAA\|+\|\BB_{1,i}(\AAA)\|+\beta_{1,i}(\AAA) \kappa_3\big(C_{2}(C_1+\xi)\big)\big\}\leq \xi.
\end{equation}
\sloppy On the other hand, according to the assumption 
that $
	\| \ZZ_{i\cdot}\dOmegai\AAA\|+\|\BB_{1,i}(\AAA)\|+\beta_{1,i}(\AAA) \kappa_3\big(3C_1C_2\big)
	\leq  2^{-2}\gamma_{1,i}^{-1} (\kappa_3\big(3C_1C_2\big))^{-1}\sigma^2_r({\II}_{1,i}(\AAA)),
$
we further have
\begin{equation}\label{eq:check-cond2-finite-bound}
\begin{split}
		\xi  =& 2\sigma^{-1}_r({\II}_{1,i}(\AAA))\big\{\| \ZZ_{i\cdot}\dOmegai\AAA\|+\|\BB_{1,i}(\AAA)\|+\beta_{1,i}(\AAA) \kappa_{3}(3C_1C_2)\big\}\\
		\leq & 2^{-1}\gamma_{1,i}^{-1}(\kappa_3\big(3C_1C_2\big))^{-1}\sigma_r({\II}_{1,i}(\AAA))\\
		\leq &  2^{-1}\{\gamma_{1,i}(\AAA) \kappa_3\big(C_{2}(C_1+\xi)\}^{-1}\sigma_r({\II}_{1,i}(\AAA)).
\end{split}
\end{equation}
Equations \eqref{eq:check-cond1-finite-bound} and \eqref{eq:check-cond2-finite-bound} together imply \eqref{eq:finite-bound-condition}. By Lemma~\ref{lemma:finite-bound}, there is $\tilde{\ttt}_i$ such that $
		\|\tilde{\ttt}_i-\ttt^*_i\|\leq \xi$ and $S_{1,i}(\tilde{\ttt};\AAA)=\mathbf{0}$. We complete the proof by noting that $\xi=2\sigma^{-1}_r({\II}_{1,i}(\AAA))\big\{\| \ZZ_{i\cdot}\dOmegai\AAA\|+\|\BB_{1,i}(\AAA)\|+\beta_{1,i}(\AAA) \kappa_3\big(3C_1C_2\big)\big\}\leq 2\sigma^{-1}_r({\II}_{1,i}) \cdot 2^{-1}\sigma_r({\II}_{1,i}(\AAA))C_1 = C_1$.
\end{proof}

By symmetry, we also have the following non-probabilistic and non-asymptotic analysis for $\tilde{\AAA}$.
For each $j\in[p]$, the estimating equation for $\aaaa_j$ based on $\TTT_{\gN_2}$ and $\OOO_{\gN_2\cdot}$ is defined as
\begin{equation}
    S_{2,j}(\aaaa_j;\TTT_{\gN_2})
    :=\phi^{-1}\sum_{i\in\gN_2} \omega_{ij} \{y_{ij}-b'(\aaaa_j^T\ttt_i)\}\ttt_i.
\end{equation}
Let
\begin{equation}\label{eq:U-2}
	\BB_{2,j}(\TTT_{\gN_2}) = \sum_{i\in\gN_2} \omega_{ij} b''(m^*_{ij})
		\ttt_i(\ttt_i-\ttt_i^*)^T\aaaa^*_j \in {\bR}^r,
\end{equation}
\begin{equation}\label{eq:II-2}
	{\II}_{2,j}(\TTT_{\gN_2})=\sum_{i\in\gN_2} \omega_{ij} b''(m_{ij}^*)\ttt_i(\ttt_i)^T,
\end{equation}
and
\begin{equation}\label{eq:beta}
	\beta_{2,j}(\TTT_{\gN_2})=  \sup_{\|\uu\|=1} \sum_{i\in\gN_2}\omega_{ij} ((\ttt_i-\ttt_i^*)^T\aaaa^*_j)^2|\ttt_j^T\uu| \text{ and }	\gamma_{2,j}(\TTT_{\gN_2}) = \sup_{\|\uu\|=1} \sum_{i\in\gN_2} \omega_{ij}|\ttt_i^T\uu|^3,
\end{equation}

\begin{lem}
\label{lemma:2nd-finite-bound}
Let $\TTT^*_{\gN_2}$ and $\AAA^*$ be such that $\MM^*_{\gN_2\cdot} = \TTT_{\gN_2}^*(\AAA^*)^T$ and $\ZZ=(z_{ij})$ with $z_{ij}=y_{ij}-b'(m^*_{ij})$ and $\dOmegaj:=\text{diag}((\omega_{ij})_{i\in\gN_2})$.
	If $\|\TTT_{\gN_2}\|,\|\TTT_{\gN_2}^*\|_{2\to\infty}\leq C_1$,  $\|\AAA^*\|_{2\to\infty}\leq C_2$ and 
\begin{equation}\label{eq:condition-2nd}
 \begin{split}
     &\| \ZZ_{\gN_2, j}^T \dOmegaj\TTT_{\gN_2}\|+\|\BB_{2,j}(\TTT_{\gN_2})\|+\beta_{2,j}(\TTT_{\gN_2})\kappa_3\big(3C_1C_2\big)\\
     \leq & \min\Big\{2^{-2}\gamma_{2,j}(\TTT_{\gN_2})^{-1} (\kappa_3\big(3C_1C_2\big))^{-1}\sigma^2_r({\II}_{2,j}(\TTT_{\gN_2})), 2^{-1}\sigma_r({\II}_{2,j}(\AAA))C_2\Big\}
 \end{split}
\end{equation}
where $\ZZ_{\gN_2,j}=(z_{ij})_{i\in\gN_2}$, then, there is $\tilde{\aaaa}$ such that  $
     S_{2,j}(\tilde{\aaaa};\TTT_{\gN_2})
=\mathbf{0}_r$, and 
\begin{equation}\label{eq:perturb-bound}
	\|\tilde{\aaaa}_j-\aaaa_j^*\|\leq 2\sigma^{-1}_r({\II}_{2,j}(\TTT_{\gN_2})))\big\{\| \ZZ_{\gN_2,j }^T\dOmegaj\TTT_{\gN_2}\|+\|\BB_{2,j}(\TTT_{\gN_2})\|+\beta_{2,j}(\TTT_{\gN_2}) \kappa_3\big(3C_1C_2\big)\big\}.
\end{equation}
Moreover, $\tilde{\aaaa}_j$ satisfies that $\|\tilde{\aaaa}_j-\aaaa_j^*\|\leq C_2$. 
\end{lem}
\begin{proof}[Proof of Lemma~\ref{lemma:2nd-finite-bound}]
    The lemma follows similar proof as that of Lemma~\ref{lemma:finite-bound} and Lemma~\ref{lemma:finite-simplify-kappa} with $(\AAA,\AAA^*,C_1,C_2)$ replaced by $(\TTT_{\gN_2},\TTT^*_{\gN_2},C_2,C_1)$. We omit the details. %
\end{proof}
\subsection{Non-asymptotic probablistic analysis}\label{sec:proof-prob-bounds-split}

Recall that $\MM^*$ has the SVD $\MM^*=\UU_r^*\DD_r^*\VV_r^*$. 
In this section, we first provide non-asymptotic bounds for each term in Lemma~\ref{lemma:finite-simplify-kappa} with $\AAA$ replaced by $\hat{\AAA}$ and $\AAA^*$ replaced by $\VV_r^*\hat{\PP}$ where $\hat{\PP}$ is defined in \eqref{eq:p-hat}.
Recall that $\hat{\AAA}=\hat{\AAA}^{(1)}$ is constructed based on $\hat{\MM}_{\gN_1\cdot}$ using data $\{Y_{ij}\omega_{ij},\omega_{ij}\}_{i\in\gN_1,j\in[p]}$, and thus, independent with $\{y_{ij},\omega_{ij}\}_{j\in[p]}$ for all $i\in\gN_2$. The results in this section hold in general for any estimator $\hat{\AAA}$ that is independent with $\{\omega_{ij},Y_{ij}\omega_{ij}\}_{i\in\gN_2,j\in [p]}$, including the proposed one.

After the analysis for terms in Lemma~\ref{lemma:finite-simplify-kappa}, we provide non-asymptotic analysis for terms in Lemma~\ref{lemma:2nd-finite-bound} with $\TTT_{\gN_2}$ replaced by $\tilde{\TTT}_{\gN_2}$ and $\TTT^*_{\gN_2}$ replaced by $\UU_r^*\DD_r^*\hat{\PP}$. Unlike $\hat{\AAA}$, $\tilde{\TTT}_{\gN_2}$ is dependent with $\{y_{ij},\omega_{ij}\}_{i\in[p]}$ for $i\in\gN_2$. Thus, we will take a different approach for the error analysis of  $\tilde{\TTT}_{\gN_2}$.

\subsubsection{Non-asymptotic bound for terms in Lemma~\ref{lemma:finite-simplify-kappa}}
\begin{lem}[Upper bound for $\|\ZZ_{i\cdot}\dOmegai\hat{\AAA}\|$ with data splitting]\label{lemma:random-matrix}
Assume $n\geq 2$.
$\|\MM^*\|_{\max}\leq \rho$ and $\|\hat{\AAA}\|_{2\to\infty}\leq C_2$.
Then, with probability at least $1-(nr)^{-1}$,
\begin{equation}\label{eq:lemma-random-1}		\max_{i\in\gN_2}\|\ZZ_{i\cdot}\dOmegai\hat{\AAA}\|\leq 
8\{ \phi^{1/2}(\kappa_2(2\rho+1))^{1/2}C_{2}\log^{1/2}(nr)r^{1/2}p_{\max}^{1/2}\vee  r^{1/2}\phi C_2/(\rho+1) \log(nr)\}
\end{equation}
 where $p_{\max}=\max_{i\in[n]}\sum_{j}\omega_{ij}$ denotes the maximum number of observations in each row.
\end{lem}
\begin{proof}[Proof of Lemma~\ref{lemma:random-matrix}]
	We first verify that under the generalized latent factor model, $\ZZ_{i\cdot}\dOmegai\hat{\AAA}_{\cdot k}$ is sub-exponential given $\OOO_{\gN_2\cdot}=(\omega_{ij})_{i\in\gN_2, j\in[p]}$ and $\hat{\AAA}$. To see this, consider the moment generating function
	\begin{equation}\label{eq:z-mgf}
	\begin{split}
			&\ex[\exp(\lambda\ZZ_{i\cdot}\dOmegai\hat{\AAA}_{\cdot k})|\OOO_{\gN_2\cdot},\hat{\AAA}]\\
		= &\prod_{j\in [p]}\ex[\lambda Z_{ij}\hat{a}_{jk}\omega_{ij}|\OOO_{\gN_2\cdot},\hat{\AAA}]\\
		 =& \exp\Big[\phi^{-1}\sum_j\omega_{ij}\{b(m^*_{ij}+\lambda \hat{a}_{jk}\phi)-b(m^*_{ij})-\lambda \hat{a}_{jk}\phi b'(m^*_{ij})\}\Big]\\
		 = & \exp[2^{-1}\lambda^2\phi \sum_j\omega_{ij}b''(\tilde{m}_{ij})(\hat{a}_{jk})^2]
	\end{split}
	\end{equation}
	for some $\tilde{m}_{ij}$ between $m^*_{ij}$ and $m^*_{ij}+\lambda \hat{a}_{jk}\phi$. Note that here we used the independence between $\hat{\AAA}$ and $\{z_{ij}\omega_{ij}\}_{i\in\gN_2}$ in the first and second equations.

\sloppy Because $|m^*_{ij}|\leq \rho$ and $|\hat{a}_{jk}|\leq C_2$, for $|\lambda|\leq (\rho+1)/(\phi C_2)$, $\tilde{m}_{ij}\leq \rho + \lambda\phi C_2\leq 2\rho+1$. 
Thus,  $\ex[\exp(\lambda\ZZ_{i\cdot}\dOmegai\hat{\AAA}_{\cdot k})|\OOO_{\gN_2\cdot},\hat{\AAA}]\leq \exp\{\lambda^2\phi\sum_j \omega_{ij}(\hat{a}_{jk})^2\kappa_2(2\rho+1)/2\}$  for $|\lambda|\leq (\rho+1)/(\phi C_2)$. 
	This implies that $\ZZ_{i\cdot}\dOmegai\hat{\AAA}_{\cdot k}$ is sub-exponential (conditional on $(\OOO_{\gN_2\cdot},\hat{\AAA})$) with parameters $\nu^2_{ik}=\phi\kappa_2(2\rho+1)\sum_j \omega_{ij} (\hat{a}_{jk})^2\leq C_2^2\phi\kappa_2(2\rho+1) p_{\max}$ and $\alpha=\phi C_2/(\rho+1)$.
	
	Applying tail probability bound for sub-exponential random variables to $\ZZ_{i\cdot}\dOmegai\hat{\AAA}_{\cdot k}$, we have
	\begin{equation}
		\pr(|\ZZ_{i\cdot}\dOmegai\hat{\AAA}_{\cdot k}|\geq t|\OOO_{\gN_2\cdot},\hat{\AAA})\leq 2 (e^{-t^2/(2\nu_{ik}^2)}\vee e^{-t/(2\alpha)})
	\end{equation} %
	for all positive $t$.
	This implies
	\begin{equation}
	\begin{split}
	    		\pr(\|\ZZ_{i\cdot}\dOmegai\hat{\AAA}\|\geq t|\OOO_{\gN_2\cdot},\hat{\AAA})
	    		\leq  \sum_{k\in[r]} \pr(|\ZZ_{i\cdot}\dOmegai\hat{\AAA}_{\cdot k}|\geq t/\sqrt{r}|\OOO_{\gN_2\cdot},\hat{\AAA})
	    		\leq  r \cdot 2 (e^{-t^2/(2r\max_k\nu_{ik}^2)}\vee e^{-t/(2r^{1/2}\alpha)}).
	\end{split}
	\end{equation}
	Combining results for different $i$ with a union bound, we have
	\begin{equation}
\pr\Big(\max_{i\in\gN_2}\|\ZZ_{i\cdot}\dOmegai\hat{\AAA}\|\geq t|\OOO_{\gN_1\cdot},\hat{\AAA}\Big)\leq  2r n \cdot (e^{-t^2/(2r\max_k\nu_{ik}^2)}\vee e^{-t/(2r^{1/2}\alpha)}).
	\end{equation}
	For $t= \{8(\log(nr)r\max_{k\in[r]} \nu_{ik}^2)^{1/2}\}\vee 8 r^{1/2}\alpha\log(nr)$ and $n\geq 2$, the right-hand side of the above inequality is no larger than $(nr)^{-1}$. Because $\nu_{ik}^2\leq \phi\kappa_2(2\rho+1) C_{2}^2p_{\max}$, we obtain 
	\begin{equation}		\max_{i\in\gN_2}\|\ZZ_{i\cdot}\dOmegai\hat{\AAA}\|\leq 
8\{ \phi^{1/2}(\kappa_2(2\rho+1))^{1/2}C_{2}\log^{1/2}(nr)r^{1/2}p_{\max}^{1/2}\vee  r^{1/2}\phi C_2/(\rho+1) \log(nr)\}
	\end{equation}
	with probability at least $1-(nr)^{-1}$.
	
\end{proof}
\begin{lem}[Upper bound for $\|\BB_{1,i}(\hat{\AAA})\|$ with data splitting]\label{lemma:bound-u-i}
 Let $\AAA^*=\VV_r^*\hat{\PP}$ and $\TTT^* = \UU_r^*\DD_r^*\hat{\PP}$. 
If $\hat{\AAA}$  is independent with $\{\omega_{ij}\}_{j\in [p]}$ for $i\in\gN_2$, $\|\hat{\AAA}\|_{2\to\infty},\|\VV_r^*\|_{2\to\infty}\leq C_2$ and $\|\UU_r\DD_r^*\|_{2\to\infty}\leq C_1$, then, for $n\geq 4$ %
with probability at least $1-1/(nr)$,
\begin{equation}
\max_{i\in\gN_2}\|\BB_{1,i}(\hat{\AAA})\|\leq \kappa_2^* \pi_{\max} C_1 \|\hat{\AAA}\|_2 \|\hat{\AAA}-\AAA^*\|_F + 64\log(n)\cdot (\pi_{\max}^{1/2}\kappa_2^* C_1C_2\|\hat{\AAA}-\AAA^*\|_F +  \kappa_2^* C_1C_2^2)
\end{equation}
\end{lem}

\begin{proof}[Proof of Lemma~\ref{lemma:bound-u-i}]
First, by the assumptions and $\hat{\PP}$ is orthogonal, $\|\TTT^*\|_{2\to\infty}=\|\UU_r^*\DD_r^*\|_{2\to\infty}\leq C_1$ and $\|\AAA^*\|_{2\to\infty} = \|\VV_r^*\|_{2\to\infty}\leq C_2$. Let
\begin{equation}
    \SSS_j
     = (\omega_{ij}-\pi_{ij})b''(m^*_{ij})
		\hat{\aaaa}_j(\hat{\aaaa}_j-\aaaa_j^*)^T\ttt^*_i. 
\end{equation}
 Then,
\begin{equation}\label{eq:u-mean-and-random}
			\BB_{1,i}(\hat{\AAA}) = \sum_{j=1}^p \omega_{ij}b''(m^*_{ij})
		\hat{\aaaa}_j(\hat{\aaaa}_j-\aaaa_j^*)^T\ttt^*_i 
  =\sum_{j\in[p]} \SSS_j +\sum_{j\in [p]}\pi_{ij}b''(m^*_{ij})
		\hat{\aaaa}_j(\hat{\aaaa}_j-\aaaa_j^*)^T\ttt^*_i.
	\end{equation}
Note that $\SSS_j$ are independent mean zero random vectors for $j\in[p]$ (conditional on $\hat{\AAA}$) and
\begin{equation}
    \big\|\SSS_j\big\|\leq 4 \kappa_2^* C_1 C_2^2.
\end{equation}
This allow us to apply the matrix Bernstein inequality (Equation (6.1.5) in \cite{Tropp2015AnInequalities}) to $\sum_{j\in[p]}\SSS_j\in\bR^{r}$, and obtain
\begin{equation}
    \pr\Big(\|\sum_{j\in[p]}\SSS_j\| \geq t|\hat{\AAA}\Big)\leq (r+1)\cdot e^{- \frac{3t^2}{8\nu}}\vee e^{- \frac{3t}{8L}}\leq 2r \cdot e^{- \frac{3t^2}{8\nu}}\vee e^{- \frac{3t}{8L}}
\end{equation}
for $t>0$ where $\nu = \max\Big\{\Big\|\sum_{j\in[p]}E\{\SSS_j\SSS_j^T|\hat{\AAA}\}\Big\|_2,\Big\|\sum_{j\in[p]}E\{\SSS_j^T\SSS_j|\hat{\AAA}\}\Big\|_2\Big\}$ 
and $L =
 4 \kappa_2^* C_1 C_2^2\geq \big\|\SSS_j\big\|$ for all $j$.
Thus, for any $0<\epsilon<r$ 
\begin{equation}\label{eq:bound-sum-s}
    \pr\Big(\|\sum_{j\in[p]}\SSS_j\| \geq \{8/3\cdot \log(2r/\epsilon)\}^{1/2}\nu^{1/2}\vee \{(8/3\cdot \log(2r/\epsilon)) L\}|\hat{\AAA}\Big)\leq \epsilon.
\end{equation}
Now we find an upper bound for $\nu$.
Since
\begin{equation}
  \ex\{\SSS_j\SSS_j^T|\hat{\AAA}\}= \pi_{ij}(1-\pi_{ij})\cdot \{b''(m^*_{ij})\}^2 \hat{\aaaa}_j(\hat{\aaaa}_j-\aaaa_j^*)^T\ttt^*_i(\ttt^*_i)^T (\hat{\aaaa}_j-\aaaa_j^*)\hat{\aaaa}_j^T,
\end{equation}
and
\begin{equation}
  \ex\{\SSS_j^T\SSS_j|\hat{\AAA}\}=  \pi_{ij}(1-\pi_{ij})\cdot \{b''(m^*_{ij})\}^2 (\ttt^*_i)^T (\hat{\aaaa}_j-\aaaa_j^*)\hat{\aaaa}_j^T\hat{\aaaa}_j(\hat{\aaaa}_j-\aaaa_j^*)^T\ttt^*_i,
\end{equation}
we have
\begin{equation}
  \max\Big\{ \big\|\ex\{\SSS_j^T\SSS_j|\hat{\AAA}\} \big\|_2,\big\| \ex\{\SSS_j\SSS_j^T|\hat{\AAA}\}\big\|_2\Big\} \leq \pi_{\max}(\kappa_2(\rho))^2 C_1^2C_2^2 \|\hat{\aaaa}_j-\aaaa_j^*\|^2
\end{equation}
which implies
\begin{equation}
   \nu =   \max\Big\{\Big\|\sum_{j\in[p]}\ex\{\SSS_j\SSS_j^T|\hat{\AAA}\}\Big\|_2,\Big\|\sum_{j\in[p]}\ex\{\SSS_j^T\SSS_j|\hat{\AAA}\}\Big\|_2\Big\}\leq \pi_{\max}(\kappa_2(\rho))^2C_1^2C_2^2 \|\hat{\AAA}-\AAA^*\|_F^2.
\end{equation}
Combine the above inequality with \eqref{eq:bound-sum-s}, we have that  with probability at least $1-\epsilon$,
\begin{equation}
    \|\sum_{j\in[p]}\SSS_j\|\leq \{8/3\cdot \log(2r/\epsilon)\}^{1/2}\pi_{\max}^{1/2}\kappa_2(\rho) C_1C_2\|\hat{\AAA}-\AAA^*\|_F +  \{(8/3\cdot \log(2r/\epsilon)) \}\cdot 4 \kappa_2(\rho) C_1C_2^2
\end{equation}
for any $0<\epsilon< r$. Simplifying this inequality, we get that with probability at least $1-\epsilon$,
\begin{equation}\label{eq:sum-s-simple}
        \|\sum_{j\in[p]}\SSS_j\|\leq \{16\cdot \log(r/\epsilon)\}\cdot (\pi_{\max}^{1/2}\kappa_2(\rho) C_1C_2\|\hat{\AAA}-\AAA^*\|_F +  \kappa_2(\rho)C_1C_2^2)
\end{equation}
for $\epsilon\in (0,r/10)$. %

Next, we obtain an upper bound for $\|\sum_{j\in [p]}\pi_{ij}b''(m^*_{ij})
		\hat{\aaaa}_j(\hat{\aaaa}_j-\aaaa_j^*)^T\ttt^*_i\|$ as
  \begin{equation}
  \begin{split}
            &\|\sum_{j\in [p]}\pi_{ij}b''(m^*_{ij})
		\hat{\aaaa}_j(\hat{\aaaa}_j-\aaaa_j^*)^T\ttt^*_i\| \\
  \leq  & C_1 \|\sum_{j\in [p]}\pi_{ij}b''(m^*_{ij})
		\hat{\aaaa}_j(\hat{\aaaa}_j-\aaaa_j^*)^T\|_2\\
  =  & C_1  \| \hat{\AAA}^T \text{diag}(\pi_{i1}b''(m^*_{i1}),\cdots,\pi_{ip}b''(m^*_{ip}))(\hat{\AAA}-\AAA^*)\|_2\\
  \leq & C_1 \|\hat{\AAA}\|_2 \pi_{\max} \kappa_2^* \|\hat{\AAA}-\AAA^*\|_F
  \end{split}
\end{equation}

Combine the above inequality with \eqref{eq:u-mean-and-random} and \eqref{eq:sum-s-simple}, we have
\begin{equation}
    \|\BB_{1,i}(\hat{\AAA})\|\leq \kappa_2^* \pi_{\max} C_1 \|\hat{\AAA}\|_2 \|\hat{\AAA}-\AAA^*\|_F + \{16\cdot \log(r/\epsilon)\}\cdot (\pi_{\max}^{1/2}\kappa_2^* C_1C_2\|\hat{\AAA}-\AAA^*\|_F +  \kappa_2^* C_1C_2^2)
\end{equation}
with probability at least $1-\epsilon$ for $\epsilon\in (0,r/10)$. We complete the proof using a union bound for $i\in\gN_2$ and $\epsilon=1/(rn^2)$. 
\end{proof}
\begin{rmk}\label{remark:leading-term-lemma}
The first term $\kappa_2^* \pi_{\max} C_1 \|\hat{\AAA}\|_2 \|\hat{\AAA}-\AAA^*\|_F$ in the upper bound is the leading term in the error analysis. To obtain this error bound, we need $\{\omega_{ij}\}_{j\in[p]}$ to be independent with $\hat{\AAA}$.  In contrast, if $\{\omega_{ij}\}_{j\in[p]}$ are dependent with $\hat{\AAA}$, then the the leading term in the error analysis may be larger (at the order $1/\sqrt{\pi_{\max}}$ in the worst case). 
\end{rmk}

\begin{lem}[Upper bound for $\beta_{1,i}(\hat{\AAA})$ with data splitting]\label{lemma:beta-bound}
If $\|\UU_r^*\DD_r^*\|_{2\to\infty}\leq C_1$, $\|\hat{\AAA}\|_{2\to\infty},\|\VV_r^*\|_{2\to\infty}\leq C_2$, and $\hat{\AAA}$ is independent with $\{\omega_{ij}\}_{i\in\gN_2, j\in[p]}$, then, with probability at least $1-1/n$,
\begin{equation}
   \max_{i\in\gN_2}\beta_{1,i}(\hat{\AAA})\leq C_1^2 C_2\{\pi_{\max} \|\hat{\AAA}-\AAA^*\|_F^2 +  4\pi_{\max}^{1/2} C_2(\log(n))^{1/2} \|\hat{\AAA}-\AAA^*\|_F 4 C_2^2\log(n)\}.
\end{equation}
\end{lem}
\begin{proof}[Proof of Lemma~\ref{lemma:beta-bound}]
Recall
\begin{equation}\label{eq:beta-first-step}
	\beta_{1,i}(\hat{\AAA})=  \sup_{\|\uu\|=1} \sum_{j}\omega_{ij} ((\hat{\aaaa}_j-\aaaa_j^*)^T\ttt^*_i)^2|\hat{\aaaa}_j^T\uu|\leq C_1^2C_2 \sum_{j\in[p]}\omega_{ij}\|\hat{\aaaa}_j-\aaaa_j^*\|^2.
\end{equation}
Conditional on $\hat{\AAA}$, $(\omega_{ij}-\pi_{ij})\|\hat{\aaaa}_j-\aaaa_j^*\|^2$ are independent, mean-zero, bounded by $4C_2^2$, and has the variance $\pi_{ij}(1-\pi_{ij})\|\hat{\aaaa}_j-\aaaa_j^*\|^4\leq 4\pi_{ij}C_2^2\|\hat{\aaaa}_j-\aaaa_j^*\|^2$. By Bernstein's inequality for bounded random variables (Theorem 2.10 in \cite{Boucheron2015ConcentrationIndependence} with $c=4C_2^2/3$ and $v=4\pi_{ij}C_2^2\|\hat{\AAA}-\AAA^*\|_F^2$),  for $t>0$
\begin{equation}
\begin{split}
&\pr\Big(\sum_{j\in[p]}(\omega_{ij}-\pi_{ij})\|\hat{\aaaa}_j-\aaaa^*\|^2\geq (8\pi_{ij}C_2^2\|\hat{\AAA}-\AAA^*\|_F^2 t)^{1/2}+ 4/3\cdot C_2^2 t|\hat{\AAA}\Big)\leq e^{-t}.\end{split}
\end{equation}
Let $t=2\log(n)$ in the above inequality and note that $\pi_{ij}\leq \pi_{\max}$ and $4/3<2$, we have that with probability at least $1-1/n^2$,
\begin{equation}
    \sum_{j\in[p]}(\omega_{ij}-\pi_{ij})\|\hat{\aaaa}_j-\aaaa^*\|^2
    \leq 4\pi_{\max}^{1/2} C_2(\log(n))^{1/2} \|\hat{\AAA}-\AAA^*\|_F+ 4 C_2^2\log(n).
\end{equation}
This implies that with probability at least $1-1/n^2$,
\begin{equation}
\begin{split}
&\sum_{j\in[p]}\omega_{ij}\|\hat{\aaaa}_j-\aaaa_j^*\|^2\\
        \leq & \sum_{j\in[p]}\pi_{ij}\|\hat{\aaaa}_j-\aaaa_j^*\|^2 +  4\pi_{\max}^{1/2} C_2(\log(n))^{1/2} \|\hat{\AAA}-\AAA^*\|_F+ 4 C_2^2\log(n)\\
        \leq & \pi_{\max} \|\hat{\AAA}-\AAA^*\|_F^2 +  4\pi_{\max}^{1/2} C_2(\log(n))^{1/2} \|\hat{\AAA}-\AAA^*\|_F+ 4 C_2^2\log(n).
\end{split}
\end{equation}
We complete the proof by combining the above inequality with \eqref{eq:beta-first-step} and applying a union bound for $i\in\gN_2$.
\end{proof}
\begin{rmk}
    Similar to Remark~\ref{remark:leading-term-lemma}, the above analysis also requires the independence of $\{\omega_{ij}\}_{j\in[p]}$ and $\hat{\AAA}$ in order to obtain the leading term $C_1^2C_2\pi_{\max}\|\hat{\AAA}-\AAA^*\|_F^2$. 
\end{rmk}
\begin{lem}[Upper bound for $p_{\max}$]\label{lemma:p-max-bound}
Recall $p_{\max}= \max_{i\in [n]}p_i$.
If $p\pi_{\max}\geq 6\log n$, then
\begin{equation}
    \pr(p_{\max}\geq 2p\pi_{\max})\leq 1/n.
\end{equation}
\end{lem}
\begin{proof}[Proof of Lemma \ref{lemma:p-max-bound}]
    First note that $|\omega_{ij}-\pi_{ij}|\leq 1$   and $p_i-\ex(p_i)=\sum_j(\omega_{ij}-\pi_{ij})$. We apply the Bernstein  inequality (Corollary 2.11 in \cite{Boucheron2015ConcentrationIndependence}) and obtain
    \begin{equation}
        \pr(p_i-\ex(p_i)\geq p\pi_{\max})
        \leq \exp\Big\{-\frac{(p\pi_{\max})^2/2}{\sum_j \ex(\omega_{ij}-p_{ij})^2 + (p\pi_{\max})/3}\Big\}.
    \end{equation}
   Because $\sum_j \ex(\omega_{ij}-p_{ij})^2=\sum_j Var(\omega_{ij})\leq \sum_j \pi_{ij}\leq p\pi_{\max}$, the above inequality implies,
   \begin{equation}
         \pr(p_i-\ex(p_i)\geq p\pi_{\max})
        \leq \exp\Big\{-\frac{(p\pi_{\max})^2/2}{(p\pi_{\max}) + (p\pi_{\max})/3}\Big\}
        = \exp\big( -\frac{3}{8} p\pi_{\max}\big),
   \end{equation}
   which further implies
      \begin{equation}
         \pr(p_i\geq 2p\pi_{\max})
        \leq \exp( -3 p \pi_{\max}/8).
   \end{equation}
   Apply a union bound to the above inequality for $i\in[n]$, we obtain
   \begin{equation}
       \pr(\max_{i\in [n]}p_i\geq 2p\pi_{\max}) \leq n \exp( -3 p\pi_{\max}/8)\leq 1/ n,
   \end{equation}
   where the last inequality is due to the assumption that $p\pi_{\max}\geq 6\log n> 16/3\log n$.
\end{proof}
\begin{lem}[Upper bound of $\gamma_{1,i}(\hat{\AAA})$]\label{lemma:gamma-bound}
    If $\|\hat{\AAA}\|_{2\to\infty}\leq C_2$ and $p\pi_{\max}> 6\log n$, then with probability at least $1-1/n$,
\begin{equation}
    \gamma_{1,i}(\hat{\AAA})\leq 2p \pi_{\max} C_2^3.
\end{equation}
\end{lem}
\begin{proof}[Proof of Lemma~\ref{lemma:gamma-bound}]
The lemma follows by Lemma~\ref{lemma:p-max-bound} and the following inequality
    \begin{equation}
	\gamma_{1,i}(\hat{\AAA}) = \sup_{\|\uu\|=1} \sum_{i=1}^p \omega_{ij} |\hat{\aaaa}_j^T\uu|^3 \leq p_{\max} C_{2}^3.
\end{equation}
\end{proof}
The next three lemmas together give a lower bound for $\sigma_r(\II_{1,i}(\hat{\AAA}))$
\begin{lem}\label{lemma:information}
If $\|\dOmegai(\hat{\AAA}-\AAA^*)\|_2\leq 2^{-1}\sigma_r(\dOmegai\AAA^*)$ and $\|\MM^*\|_{\max}\leq \rho$, then
	\begin{equation}
		\sigma_r({\II}_{1,i}(\hat{\AAA}))\geq 2^{-2}\delta_2(\rho)\sigma^2_r(\dOmegai\AAA^*).
	\end{equation}
\end{lem}
\begin{proof}[Proof of Lemma~\ref{lemma:information}]
For any $|\uu|=1$ and $\uu\in {\bR}^r$, 
\begin{equation}
\uu^T{\II}_{1,i}(\hat{\AAA})\uu= \sum_{j=1}^p\omega_{ij} b''\big(m_{ij}^*\big)(\uu^T\hat{\aaaa}_j)^2 \geq \delta_2(\rho) \sum_{j=1}^p\omega_{ij} (\uu^T\hat{\aaaa}_j)^2 \geq \delta_2(\rho)\sigma_r^2(\dOmegai\hat{\AAA}).
\end{equation}
This implies $\sigma_r({\II}_{1,i}(\hat{\AAA}))\geq \delta_2(\rho)\sigma_r^2(\dOmegai\hat{\AAA})$.
	By Weyl's inequality, $\sigma_r(\dOmegai\hat{\AAA})\geq \sigma_r(\dOmegai\AAA^*)-\|\dOmegai(\hat{\AAA}-\AAA^*)\|_2$. Thus, if $\|\dOmegai(\hat{\AAA}-\AAA^*)\|_2\leq 2^{-1}\sigma_r(\dOmegai\AAA^*)$, then $\sigma_r(\dOmegai\hat{\AAA})\geq 2^{-1}\sigma_r(\dOmegai\AAA^*)$, and thus,
	\begin{equation}
		\sigma_r({\II}_{1,i}(\hat{\AAA}))\geq \delta_2(\rho)\sigma^2_r(\dOmegai\hat{\AAA}) \geq  2^{-2}\delta_2(\rho)\sigma^2_r(\dOmegai\AAA^*).
	\end{equation}
\end{proof}
The next two lemmas give a lower bound for $\sigma_r(\dOmegai\AAA^*)$ and an upper bound for $\|\dOmegai(\hat{\AAA}-\AAA^*)\|_2$.
\begin{lem}\label{lemma:submatrix}
Let $\AAA^*=\VV_r^*\hat{\PP}$ and let $\PPP_{1,i}=\text{diag}(\pi_{i1},\cdots,\pi_{ip})=\ex(\dOmegai)$ and $\lambda^*_{i,\min}=\lambda_r((\VV_r^*)^T\PPP_{1,i}\VV_r^*)= \lambda_r((\AAA^*)^T\PPP_{1,i}\AAA^*)$, where $\lambda_r(\cdot)$ denotes the $r$-th largest eigenvalue of a symmetric matrix.
If $\lambda^*_{\min}:=\min_{i\in[n]}\lambda^*_{i,\min}\geq 16\|\VV_r^*\|_{2\to\infty}^2\log(nr)$, then
\begin{equation}
\pr\Big(\min_{i\in[n]}\sigma^2_r(\dOmegai\AAA^*)\leq 2^{-1}\lambda^*_{\min}\Big)\leq 1/(nr)
\end{equation}
Moreover, if $\pi_{\min}\sigma_r^2(\AAA^*)\geq 32\|\AAA^*\|_{2\to\infty}^2 \log(n)$ and $n\geq r$, then
\begin{equation}
   \pr\Big(\min_{i\in[n]}\sigma^2_r(\dOmegai\AAA^*)\leq 2^{-1}\pi_{\min}\sigma^2_r(\AAA^*)\Big)\leq 1/(nr).
\end{equation}
\end{lem}
\begin{rmk}
    In the `moreover part' of the above lemma, $\sigma_r^2(\AAA^*)=\sigma_r^2(\VV_r^*\hat{\PP}) = 1$, so it is possible to further simplify the statement of lemma. We keep the current form without simplification so that similar results can be obtained by symmetry for $\TTT^*=\UU_r^*\DD_r^*\hat{\PP}$, which will be useful for the analysis later.
\end{rmk}
\begin{proof}[Proof of Lemma~\ref{lemma:submatrix}]
First note that $\sigma_r^2(\dOmegai\AAA^*) = \sigma_r^2(\dOmegai\VV_r^*\hat{\PP})= \sigma_r^2(\dOmegai\VV_r^*) = \lambda_r((\VV_r^*)^T\dOmegai\VV_r^*)$. Also note that for all $t\in(0,1)$
\begin{equation}
    \begin{split}
        &\pr\Big(\sigma^2_r(\dOmegai\VV_r^*)\leq (1-t)\lambda^*_{i,\min}\Big)\\
        = & \pr\Big(\lambda_r\big(\sum_j \omega_{ij}\vvv^*_j(\vvv^*_j)^T\big)\leq 
        (1-t)\cdot \lambda_r\big(\sum_j \pi_{ij}\vvv^*_j(\vvv^*_j)^T\big)
        \Big),
    \end{split}
\end{equation}
where $\vvv_j^*\in \bR^r$ denotes the $j$-th row of $\VV_r^*$.
Note that $\lambda_{r}\{\ex(\sum_{j\in[p]} \omega_{ij}\vvv^*_j(\vvv^*_j)^T)\} = \lambda^*_{i,\min}$, $\lambda_{1}(\omega_{ij}\vvv^*_j(\vvv^*_j)^T)\leq \|\VV_r^*\|_{2\to\infty}^2$, and $\omega_{ij}\vvv^*_j(\vvv^*_j)^T$ are independent for different $j$.
    Applying Remark 5.3 in \cite{Tropp2012User-friendlyMatrices}
   to the above probability, we obtain that for all  $t\in (0,1)$,
\begin{equation}
\pr\Big(\lambda_r\big(\sum_j \omega_{ij}\vvv^*_j(\vvv^*_j)^T\big)\leq 
        (1-t)\cdot \lambda^*_{i,\min}
        \Big)\leq r\exp\Big\{- 2^{-1}\|\VV_r^*\|_{2\to\infty}^{-2}(1-t)^2\lambda^*_{i,\min}\Big\}.
\end{equation}
Thus,
\begin{equation}
    \pr\Big(\sigma^2_r(\dOmegai\AAA^*)\leq (1-t)\lambda^*_{i,\min}\Big)\leq r\exp\big\{- 2^{-1}\|\VV_r^*\|_{2\to\infty}^{-2}(1-t)^2\lambda^*_{i,\min}\big\}.
\end{equation}
Let $t=1/2$ in the above inequality, we obtain
\begin{equation}
        \pr\Big(\sigma^2_r(\dOmegai\AAA^*)\leq 2^{-1}\lambda^*_{i,\min}\Big)\leq r\exp\big\{- 8^{-1}\|\VV_r^*\|_{2\to\infty}^{-2}\lambda^*_{i,\min}\big\},
\end{equation}
which further implies
\begin{equation}
            \pr\Big(\sigma^2_r(\dOmegai\AAA^*)\leq 2^{-1}\lambda^*_{\min}\Big)\leq r\exp\big\{- 8^{-1}\|\VV_r^*\|_{2\to\infty}^{-2}\lambda^*_{\min}\big\}.
\end{equation}
Apply a union bound to the above inequality for different $i\in[n]$, we obtain
\begin{equation}
     \pr\big(\min_{i\in[n]}\sigma^2_r(\dOmegai\AAA^*)\leq 2^{-1}\lambda^*_{\min}\big)\leq nr\exp\big\{- 8^{-1}\|\VV_r^*\|_{2\to\infty}^{-2}\lambda^*_{\min}\big\}.
\end{equation}
The right-hand side of the above inequality is no greater than $(nr)^{-1}$ when $\lambda^*_{\min}\geq 16 \|\VV_r^*\|_{2\to\infty}^2\log(nr)=16 \|\AAA^*\|_{2\to\infty}^2\log(nr)$. %

The `moreover' part of the lemma is proved by noting that $\lambda^*_{i,\min}=\lambda_r(\sum_{j\in[p]}\pi_{ij}\aaaa_j^* (\aaaa_j^*)^T)\geq \pi_{\min}\lambda_r(\sum_{j}\aaaa_j^* (\aaaa_j^*)^T)=\pi_{\min}\sigma_r^2(\AAA^*)$.
\end{proof}
\begin{lem}\label{lemma:omega-ahat-a}
    If $\|\hat{\AAA}\|_{2\to\infty},\|\VV_r^*\|_{2\to\infty}\leq C_2$ and $\hat{\AAA}$ is independent with $\{\omega_{ij}\}_{i\in\gN_2,j\in[p]}$, then
 with probability at least $1-1/(nr)$,
\begin{equation}
    \max_{i\in\gN_2} \|\dOmegai(\hat{\AAA}-\AAA^*)\|_2^2\leq \pi_{\max}\|\hat{\AAA}-\AAA^*\|_F^2+  64 \log(n) \cdot \{ (\pi_{\max}^{1/2} C_2\|\hat{\AAA}-\AAA^*\|_F)\vee C_2^2\} 
\end{equation}
for $n\geq 4$.
\end{lem}
\begin{proof}[Proof of Lemma~\ref{lemma:omega-ahat-a}]
      Let $\Delta_{a_j}=\hat{\aaaa}_j-\aaaa_j^*$ and $\Delta_{\AAA}=\hat{\AAA}-\AAA^*=(\Delta_{\aaaa_1}^T,\cdots,\Delta_{\aaaa_p}^T)^T$.   
Conditional on $\hat{\AAA}$, $(\omega_{ij}- \pi_{ij})\Delta_{\aaaa_j}\Delta_{\aaaa_j}^T$ are independent symmetric matrices satisfying $\|(\omega_{ij}-\pi_{ij})\Delta_{\aaaa_j}\Delta_{\aaaa_j}^T\|_2 \leq \|\Delta_{\aaaa_j}\|_{2\to\infty}^2\leq 4C_2^2$, and $\|\ex\{(\omega_{ij}\Delta_{\aaaa_j}\Delta_{\aaaa_j}^T)^T\omega_{ij}\Delta_{\aaaa_j}\Delta_{\aaaa_j}^T\}\|_2\leq \pi_{ij}\|\Delta_{\aaaa_j}\|_{2\to\infty}^2 \|\Delta_{\aaaa_j}\|^2\leq 4 \pi_{ij}C_2^2 \|\Delta_{\aaaa_j}\|^2$. Applying the inequality (6.1.5) in \cite{Tropp2015AnInequalities} to $\sum_{j\in[p]}(\omega_{ij}-\pi_{ij})\Delta_{\aaaa_j}\Delta_{\aaaa_j}^T$, we obtain that for all $t>0$
\begin{equation}
    P\Big(\|\sum_{j\in[p]}(\omega_{ij}-\pi_{ij})\Delta_{\aaaa_j}\Delta_{\aaaa_j}^T\|_2 \geq t|\hat{\AAA}\Big)
    \leq 2 r \cdot \exp\Big\{ - \frac{3t^2}{8\nu}\wedge \frac{3t}{8L}\Big\}
\end{equation}
where $\nu = 4 \pi_{\max}C_2^2 \|\Delta_{\AAA}\|^2_F\geq \sum_{j\in[p]}\|\ex[\{\omega_{ij}\Delta_{\aaaa_j}\Delta_{\aaaa_j}^T\}^T\omega_{ij}\Delta_{\aaaa_j}\Delta_{\aaaa_j}^T]\|$ and $L = 4C_2^2 \geq \|(\omega_{ij}-\pi_{ij})\Delta_{\aaaa_j}\Delta_{\aaaa_j}^T\|_2$.

For $\epsilon\in(0,1)$, let $t= [\{8/3\cdot\log(2r/\epsilon)\}^{1/2} \nu^{1/2}]\vee[\{8/3\cdot\log(2r/\epsilon)\}L]$ in the above inequality,
we obtain 
\begin{equation}
    P\Big(\|\sum_{j\in[p]}(\omega_{ij}-\pi_{ij})\Delta_{\aaaa_j}\Delta_{\aaaa_j}^T\|_2 \geq t|\hat{\AAA}\Big)
    \leq \epsilon.
\end{equation}
Now we give an upper bound for $t= [\{8/3\cdot\log(2r/\epsilon)\}^{1/2} \nu^{1/2}]\vee[\{8/3\cdot\log(2r/\epsilon)\}L]$ for $\epsilon\in (0,r/10)$
\begin{equation}
\begin{split}
       & [\{8/3\cdot\log(2r/\epsilon)\}^{1/2} \nu^{1/2}]\vee[\{8/3\cdot\log(2r/\epsilon)\}L]\\
    \leq & 8\log(r/\epsilon) \cdot ( \nu^{1/2}\vee L)\\
    \leq & 32 \log(r/\epsilon) \cdot \{ (\pi_{\max}^{1/2} C_2\|\Delta_{\AAA}\|_F)\vee C_2^2\}.
\end{split}
\end{equation}
Thus, with probability at least $1-\epsilon$,
\begin{equation}
    \|\sum_{j\in[p]}(\omega_{ij}-\pi_{ij})\Delta_{\aaaa_j}\Delta_{\aaaa_j}^T\|_2 \leq 32 \log(r/\epsilon) \cdot \{ (\pi_{\max}^{1/2} C_2\|\Delta_{\AAA}\|_F)\vee C_2^2\}
\end{equation}
for $\epsilon\in (0,r/10)$. Applying a union bound to the above result with $\epsilon=1/(rn^2)$, we have
\begin{equation}
    \|\sum_{j\in[p]}(\omega_{ij}-\pi_{ij})\Delta_{\aaaa_j}\Delta_{\aaaa_j}^T\|_2 \leq 64 \log(n) \cdot \{ (\pi_{\max}^{1/2} C_2\|\Delta_{\AAA}\|_F)\vee C_2^2\}
\end{equation}
with probability at least $1-1/(nr)$ for all $i\in\gN_2$ and $n\geq 4$.

Next, we give an upper bound for $\lambda_1(\sum_{j=1}^p \pi_{ij}\Delta_{\aaaa_j}\Delta_{\aaaa_j}^T)$.
\begin{equation}
    \lambda_1(\sum_{j=1}^p \pi_{ij}\Delta_{\aaaa_j}\Delta_{\aaaa_j}^T) \leq \pi_{\max}\lambda_1(\sum_{j=1}^p \Delta_{\aaaa_j}\Delta_{\aaaa_j}^T) = \pi_{\max}\|\Delta_{\AAA}\|_2^2 \leq \pi_{\max}\|\Delta_{\AAA}\|_F^2.
\end{equation}
Combining the above two inequalities and note that $\|\dOmegai(\hat{\AAA}-\AAA^*)\|_2^2 = \lambda_1(\sum_{j\in[p]}\omega_{ij}\Delta_{\aaaa_j}\Delta_{\aaaa_j}^T)$, we obtain that with probability at least $1-1/(nr)$,
\begin{equation}
    \|\dOmegai(\hat{\AAA}-\AAA^*)\|_2^2 \leq \pi_{\max}\|\Delta_{\aaaa_j}\|_F^2 +64 \log(n) \cdot \{ (\pi_{\max}^{1/2} C_2\|\Delta_{\AAA}\|_F)\vee C_2^2\}
\end{equation}
for $n\geq 4$.
\end{proof}

\subsubsection{Non-asymptotic bound for terms in Lemma \ref{lemma:2nd-finite-bound}}

Let $n_{\max}=\max_{j\in[p]}\sum_{i\in[n]}\omega_{ij}$ be the maximal number of observations in each column.

\begin{lem}\label{lemma:nmax}
 If $n\pi_{\max}\geq 6\log(p)$, then $\pr(n_{\max}\geq 2 n\pi_{\max})\leq 1/p$.
\end{lem}
\begin{proof}[Proof of Lemma~\ref{lemma:nmax}]
    The proof is similar to that of Lemma~\ref{lemma:gamma-bound}. We ommit the details.
\end{proof}
\begin{lem}\label{lemma:z-max}
With probability at least $1-1/(np)$, 
    $\|\ZZ\|_{\max} \leq 8 \log(np) \{(\phi\kappa^*_2)^{1/2}\vee 1\} $
\end{lem}
\begin{proof}[Proof of Lemma~\ref{lemma:z-max}]
Note that the moment generating function for $z_{ij}$ is
$\ex(\exp(\lambda z_{ij})) = \exp\{\phi^{-1}(b(m_{ij}^*+\lambda)-b(m_{ij}^*)-\lambda b'(m_{ij^*}))\} = \exp\{2^{-1}\lambda^2\phi b''(\tilde{m}_{ij})\}$ for some $\tilde{m}_{ij}$ between $m^*_{ij}$ and $m^*_{ij}+\lambda$. Thus, $z_{ij}$  is sub-exponential with $\nu^2=\phi\kappa_2^*$ and $\alpha=1$, which implies
$\pr(|Z_{ij}|\geq t)\leq 2 e^{-t^2/(2\phi\kappa_2^*)}\vee e^{-t/2}$. Thus,
\begin{equation}
\begin{split}
    	\pr(\|\ZZ\|_{\max}\geq t)
    	\leq 2 (np)  (e^{-t^2/(2\phi\kappa_2^*)}\vee e^{-t/2})
\end{split}
\end{equation}
Let  $t=8 \log(np) \{(\phi\kappa^*_2)^{1/2}\vee 1\}$ in the above probability bound. We see that the right-hand side is no larger than $(np)^{-1}$. 

\end{proof}

\begin{lem}[Upper bound for $\|\ZZ_{\gN_2, j}^T\dOmegaj\tilde{\TTT}_{\gN_2}\|$]\label{lemma:random-matrix-2nd}
Assume that $n\pi_{\max}\geq 6\log(p)$.
    With probability at least $1-3/p - \pr(\|\tilde{\TTT}_{\gN_2}\|_{2\to\infty} > 2C_1)$,
   \begin{equation}    
  \begin{split}
        &\max_{j\in[p]}\|\ZZ_{\gN_2, j}^T\dOmegaj\tilde{\TTT}_{\gN_2}\|\\
\leq &
16\{ \phi^{1/2}(\kappa_2^*)^{1/2}C_{1}\log^{1/2}(pr)r^{1/2}(n\pi_{\max})^{1/2}\vee  r^{1/2}\phi C_1/(\rho+1) \log(pr)\}\\
&+16\|\tilde{\TTT}_{\gN_2}-\TTT^*_{\gN_2}\|_{2\to\infty}\cdot n\pi_{\max} \log(np)\{(\kappa_2^*\phi)^{1/2}\vee 1\}
  \end{split}
\end{equation}
on the event $\{\|\tilde{\TTT}_{\gN_2}\|_{2\to\infty} \leq 2C_1\}$.
\end{lem}
\begin{proof}[Proof of Lemma~\ref{lemma:random-matrix-2nd}]
With similar derivations as that for the inequality \eqref{eq:lemma-random-1}, we have that with probability at least $1-1/(pr)$, 
\begin{equation}    \max_{j\in[p]}\|\ZZ_{\gN_2, j}^T\dOmegaj\TTT^*_{\gN_2}\|\leq 
16\{ \phi^{1/2}(\kappa_2^*)^{1/2}C_{1}\log^{1/2}(pr)r^{1/2}n_{\max}^{1/2}\vee  r^{1/2}\phi C_1/(\rho+1) \log(pr)\}.
\end{equation}
Note that
\begin{equation}
    \|\ZZ_{{\gN_2}, j}^T\dOmegaj(\tilde{\TTT}_{\gN_2}-\TTT^*_{\gN_2})\|
=\| \sum_{i\in\gN_2} \omega_{ij} z_{ij}(\tilde{\ttt}_i-\ttt_i^*)\| \leq \|\tilde{\TTT}_{\gN_2}-\TTT_{\gN_2}^*\|_{2\to\infty}\|\ZZ\|_{\max} n_{\max}.
\end{equation}
Thus, with probability at least $1-1/(pr)$,
  \begin{equation}    
  \begin{split}
        &\max_{j\in[p]}\|\ZZ_{\gN_2, j}^T\dOmegaj\tilde{\TTT}_{\gN_2}\|\\
        \leq & \max_{j\in[p]} \| \ZZ_{{\gN_2}, j}^T\dOmegaj\TTT^*_{\gN_2}\| +  \max_{j\in[p]}\|\ZZ_{\gN_2, j}^T\dOmegaj(\tilde{\TTT}_{\gN_2}-\TTT^*_{\gN_2})\|\\
 \leq  &  
16\{ \phi^{1/2}(\kappa_2^*)^{1/2}C_{1}\log^{1/2}(pr)r^{1/2}(n_{\max})^{1/2}\vee  r^{1/2}\phi C_1/(\rho+1) \log(pr)\}\\
&+\|\tilde{\TTT}_{\gN_2}-\TTT^*_{\gN_2}\|_{2\to\infty}\|\ZZ\|_{\max} n_{\max}
  \end{split}
\end{equation}
Combine the above display with Lemma~\ref{lemma:nmax} and Lemma~\ref{lemma:z-max}, we have that with probability at least $1-3/p$,
\begin{equation}    
  \begin{split}
        &\max_{j\in[p]}\|\ZZ_{\gN_2, j}^T\dOmegaj\tilde{\TTT}_{\gN_2}\|\\
\leq &
16\{ \phi^{1/2}(\kappa_2^*)^{1/2}C_{1}\log^{1/2}(pr)r^{1/2}(n\pi_{\max})^{1/2}\vee  r^{1/2}\phi C_1/(\rho+1) \log(pr)\}\\
&+16\|\tilde{\TTT}_{\gN_2}-\TTT^*_{\gN_2}\|_{2\to\infty}\cdot n\pi_{\max} \log(np)\{(\kappa_2^*\phi)^{1/2}\vee 1\}
  \end{split}
\end{equation}

\end{proof}
\begin{lem}[Upper bound for $\|\BB_{2,j}(\tilde{\TTT}_{\gN_2})\|$]\label{lemma:B-2nd}
Assume that $n\pi_{\max}\geq 6\log(p)$. %
    With probability at least $1-1/p$,
       \begin{equation}
	\max_{j\in[p]}\|\BB_{2,j}(\tilde{\TTT}_{\gN_2})\| 
  \leq  4 C_1C_2\kappa_2^* n\pi_{\max} \|\tilde{\TTT}_{\gN_2}-\TTT^*_{\gN_2}\|_{2\to\infty},
\end{equation} 
on the event $\{\|\tilde{\TTT}_{\gN_2}\|_{2\to\infty}\leq  2C_1\}$.

\end{lem}
\begin{proof}[Proof of Lemma~\ref{lemma:B-2nd}]
     \begin{equation}
	\|\BB_{2,j}(\tilde{\TTT}_{\gN_2})\| = \|\sum_{i\in\gN_2} \omega_{ij} b''(m^*_{ij})
		\tilde{\ttt}_i(\tilde{\ttt}_i-\ttt_i^*)^T\aaaa^*_j\|
  \leq 2C_1C_2 \|\tilde{\TTT}_{\gN_2}-\TTT_{\gN_2}^*\|_{2\to\infty}\max_{ij} b''(m_{ij}^*)n_{\max}
\end{equation}
According to Lemma~\ref{lemma:nmax} and noting that $\max_{ij} b''(m_{ij}^*)\leq \kappa_2^*$, we further have that with probability at least $1-1/p$,
 \begin{equation}
	\max_{j\in[p]}\|\BB_{2,j}(\tilde{\TTT}_{\gN_2})\| 
  \leq 4 C_1C_2\kappa_2^* n\pi_{\max} \|\tilde{\TTT}_{\gN_2}-\TTT_{\gN_2}^*\|_{2\to\infty}
\end{equation}
\end{proof}
\begin{lem}\label{lemma:2nd-beta}
Assume that $n\pi_{\max}\geq 6\log(p)$. 
    With probability at least $1-1/p$, 
    \begin{equation}
	\max_{j\in [p]}\beta_{2,j}(\tilde{\TTT}_{\gN_2})\leq  4 C_1 C_2^2 \|\TTT-\TTT^*\|_{2\to\infty}^2 n\pi_{\max}
\end{equation}
on the event $\{\|\tilde{\TTT}_{\gN_2}\|_{2\to\infty}\leq  2C_1\}$.
\end{lem}
\begin{proof}[Proof of Lemma~\ref{lemma:2nd-beta}]
    \begin{equation}
	\beta_{2,j}(\tilde{\TTT}_{\gN_2})=  \sup_{\|\uu\|=1} \sum_{i\in\gN_2}\omega_{ij} ((\tilde{\ttt}_i-\ttt_i^*)^T\aaaa^*_j)^2|\tilde{\ttt}_j^T\uu|\leq 2C_1 C_2^2 \|\tilde{\TTT}-\TTT^*\|_{2\to\infty}^2 n_{\max}
\end{equation}

The proof is completed by combining the above inequality with Lemma~\ref{lemma:nmax}
\end{proof}
\begin{lem}\label{lemma:2nd-gamma}
Assume that $n\pi_{\max}\geq 6\log(p)$. 
    With probability at least $1-1/p$, 
    \begin{equation}
    \max_{j\in[p]}\gamma_{2,j}(\tilde{\TTT}_{\gN_2}) \leq 16 C_1^3 n\pi_{\max}
\end{equation}
on the event $\{\|\tilde{\TTT}_{\gN_2}\|_{2\to\infty}\leq  2C_1\}$.
\end{lem}
\begin{proof}[Proof of Lemma~\ref{lemma:2nd-gamma}]
     \begin{equation}
    \gamma_{2,j}(\tilde{\TTT}_{\gN_2}) = \sup_{\|\uu\|=1} \sum_j \omega_{ij}|\tilde{\ttt}_i^T\uu|^3\leq 8 C_1^3 n_{\max}
\end{equation}
Combine this with Lemma~\ref{lemma:nmax}, we complete the proof.
\end{proof}

\begin{lem}\label{lemma:2nd-info}
Assume that $\pr(\|\tilde{\TTT}_{\gN_2}-\TTT_{\gN_2}^*\|_{2\to\infty}\leq e_{\TTT,2\to\infty})\geq 1-\epsilon$  for some non-random $e_{\TTT,2\to\infty}$, $n\pi_{\max}\geq 6\log(p)$, $\pi_{\min}\sigma_r^2(\TTT^*_{\gN_2})\geq 32\|\TTT^*_{\gN_2}\|_{2\to\infty}^2 \log(p)$, $p\geq r$, and $2 e_{\TTT,2\to\infty}^2 n\pi_{\max}\leq 2^{-3}\pi_{\min}\sigma_r^2(\TTT_{\gN_2}^*)$. Then, 
with probability at least $1-2/p-\epsilon$
      \begin{equation}
        {\II}_{2,j}(\tilde{\TTT}_{\gN_2})\geq 2^{-2}\delta_2(\rho)\pi_{\min}\sigma_r^2(\TTT^*)\geq 2^{-2}\delta_2(\rho)\pi_{\min}\psi_r^2
    \end{equation}  
\end{lem}

\begin{proof}[Proof of Lemma~\ref{lemma:2nd-info}]
First note that
\begin{equation}
    \begin{split}
       \|\dOmegaj(\tilde{\TTT}_{\gN_2}-\TTT_{\gN_2}^*)\|_2^2= \|\sum_{i\in\gN_2}\omega_{ij}(\tilde{\ttt}_i-\ttt_i^*)(\tilde{\ttt}_i-\ttt_i^*)^T\|_2 
    \leq  \|\tilde{\TTT}_{\gN_2}-\TTT_{\gN_2}^*\|_{2\to\infty}^2 \cdot  n_{\max}
    \end{split}
\end{equation}
Combine the above inequality with Lemma~\ref{lemma:nmax}, we have that with probability at least $1-1/p$,
\begin{equation}
     \|\dOmegaj(\tilde{\TTT}_{\gN_2}-\TTT_{\gN_2}^*)\|_2^2
    \leq  2 \|\tilde{\TTT}_{\gN_2}-\TTT_{\gN_2}^*\|_{2\to\infty}^2 \cdot  n\pi_{\max}.
\end{equation}
On the other hand, with similar argument as those in the proof of Lemma~\ref{lemma:submatrix}, we have that if $\pi_{\min}\sigma_r^2(\TTT_{\gN_2}^*)\geq 32\|\TTT_{\gN_2}^*\|_{2\to\infty}^2 \log(p)$ and $p\geq r$, then
\begin{equation} \pr\Big(\min_{i\in[n]}\sigma^2_r(\dOmegaj\TTT_{\gN_2}^*)\leq 2^{-1}\pi_{\min}\sigma^2_r(\TTT_{\gN_2}^*)\Big)\leq 1/(pr)
\end{equation}
Thus, if $2 e_{\TTT,2\to\infty}^2 n\pi_{\max}\leq 2^{-3}\pi_{\min}\sigma_r^2(\TTT_{\gN_2}^*)$, then with probability at least $1-\epsilon-2/p$,
$$
\|\dOmegaj(\tilde{\TTT}_{\gN_2}-\TTT_{\gN_2}^*)\|_2\leq 2^{-1}\sigma_r(\dOmegaj\TTT_{\gN_2}^*).
$$
With similar arguments as those for Lemma~\ref{lemma:information}, we have that with probability $1-2/p-\epsilon$,
\begin{equation}
    \min_{j\in [p]}{\II}_{2,j}(\tilde{\TTT}_{\gN_2}) \geq 2^{-2}\delta_2(\rho)\pi_{\min}\sigma_r^2(\TTT_{\gN_2}^*)\geq 2^{-2}\delta_2(\rho)\pi_{\min}\psi_r^2
\end{equation}
where the last inequality in the above display holds because $\TTT_{\gN_2}^* = (\UU_r^*)_{\gN_2\cdot}\DD_r^*$ and as a result $\sigma_r(\TTT^*) = \sigma_r(\MM^*_{\gN_2\cdot})\geq\psi_r$.

\end{proof}
\subsubsection{Bounds for $\psi_1$ and $\psi_r$}
\begin{lem}\label{lemma:r-psi}
Let $\RR = \UU\DD\VV^T$ be the singular value decomposition of a non-random matrix $\RR$ with $\UU\in \bR^{n\times r}$, $\VV\in \bR^{p\times r}$ and $\DD=\text{diag}(\sigma_1(\RR,\cdots,\sigma_r(\RR))$, and let $g_i\sim \text{Bernoulli}(1/2)$ be i.i.d. random variables. 

Then,
    \begin{equation}\label{eq:sigma-r-lower}
    \pr(\sigma_r^2(\RR_{\gG})\leq 2^{-2}\sigma_r^2(\RR) )\leq r \exp\big[{-2^{-3}\sigma_r^2(\RR)/\{\|\UU\|_{2\to\infty}^2\sigma_1^2(\RR)\}}\big],
\end{equation}
where $\gG=\{i:g_i=1\}$ and $\RR_{\gG}=(r_{ij})_{i \in \mathcal{G}}$. In particular, if $\sigma_r^2(\RR)/\{\|\UU\|_{2\to\infty}^2\sigma_1^2(\RR)\}\gg \log(r)$,
then with probability converging to $1$, $\sigma_r(\RR)\lesssim\sigma_r(\RR_{\gG})\leq \sigma_1(\RR_{\gG})\leq\sigma_1(\RR)$.

\end{lem}
\begin{proof}
First, as $\RR_{\gG}$ is a submatrix of $\RR$, we have $\sigma_1(\RR_{\gG})\leq \sigma_1(\RR)$. In the rest of the proof, we show that \eqref{eq:sigma-r-lower} holds.
Let $\TT =\UU\DD\in\bR^{n\times r}$. Then, $\RR_{\gG} = \TT_{\gG}\VV^T$ and $\sigma_r^2(\RR_{\gG})=\lambda_r(\RR_{\gG}\RR_{\gG}^T) = \lambda_r(\TT_{\gG}\TT_{\gG}^T) =  \lambda_r(\TT_{\gG}^T\TT_{\gG})=\lambda_r(\sum_{i\in[n]}g_i\bt_i\bt_i^T)$ where $\bt_i=\TT_{i\cdot}^T$ indicates the $i$-th row of the matrix $\TT$. 

Note that for each $i$, $g_i\bt_i\bt_i^T$ is positive semi-definite, and $\lambda_1(g_i\bt_i\bt_i^T) \leq \|\bt_i\|^2\leq \|\TT\|_{2\to\infty}^2$. Also, $\lambda_r(\ex(\sum_{i\in[n]}g_i\bt_i\bt_i^T)) = 2^{-1}\lambda_r(\TT^T\TT)=2^{-1}\sigma_r^2(\RR)$. 
Applying the weak Chernoff bounds for matrices (inequalities on page 61 of \cite{Tropp2015AnInequalities} under equations (5.1.7) with $t=1/2$), we obtain
\begin{equation}
    \pr(\lambda_r(\sum_{i\in[n]}g_i\bt_i\bt_i^T)\leq 2^{-2}\sigma_r^2(\RR) )\leq r e^{-2^{-3}\sigma_r^2(\RR)/\|\TT\|_{2\to\infty}^2}.
\end{equation}
We complete the proof by noting that $\|\TT\|_{2\to\infty}\leq \|\UU\|_{2\to\infty}\sigma_1(\RR)$. 
\end{proof}

\subsection{Asymptotic analysis}\label{sec:proof-asymptotic}
In this section, we provide asymptotic analysis of the estimators based on the non-asymptotic bounds established in previous sections.
\begin{lem}[Asymptotic bounds for $\psi_1$ and $\psi_r$]\label{lemma:psi}
Recall that $\psi_1=\sigma_1(\MM_{\gN_1\cdot }^*)\vee \sigma_1(\MM_{\gN_2\cdot }^*)$  and $\psi_r= \sigma_r(\MM_{\gN_1\cdot}^*)\wedge \sigma_r(\MM_{\gN_2\cdot}^*)$. If $\sigma^2_r(\MM^*)/\sigma^2_1(\MM^*)\gg \|\UU_r^*\|_{2\to\infty}^2\log(r) $, then with probability converging to $1$, $\sigma_r(\MM^*)\lesssim \psi_r\leq \psi_1\leq \sigma_1(\MM^*)$.
\end{lem}
\begin{proof}[Proof of Lemma~\ref{lemma:psi}]
	This lemma is a direct application of Lemma~\ref{lemma:r-psi} with $\RR$, $\UU$, and $\mathcal{G}$ replaced by $\MM^*$, $\UU_r^*$ and $\gN_1$ (or $\gN_2$). We omit the details.
\end{proof}

\begin{lem}[Asymptotic analysis for $\tilde{\TTT}_{\gN_2}$]\label{lemma:asymptotic-1st-theta}
    Let $\AAA^*=\VV_r^*\hat{\PP}$, $\TTT^*=\UU_r^*\DD_r^*\hat{\PP}$, where $\hat{\PP}$ is defined in \eqref{eq:p-hat}. Assume that $\lim_{n,p\to\infty}\pr(\|\hat{\AAA}-\AAA^*\|_{F}\leq e_{\AAA,F})=1$. 
Assume the following asymptotic regime holds:
\begin{enumerate}
        \item $\phi\lesssim 1$;
    \item $\|\UU_r^*\|_{2\to\infty}\lesssim (r/n)^{1/2}$, $\|\VV_r^*\|_{2\to\infty} \lesssim (r/p)^{1/2}$, $C_2 \sim (r/p)^{1/2}$;
    \item $ (np)^{1/2} r^{\eta_2}\lesssim \sigma_r(\MM^*)\leq \sigma_1(\MM^*)\lesssim  (np)^{1/2}r^{\eta_1}$, for constants $\eta_1$ and $\eta_2$;
\item 
$
    p\pi_{\min}\gg (\delta_2^*)^{-4}(\kappa_2^*)^2(\log(n))^2\max\big\{ r^{1\vee (1+2\eta_1)
    \vee(1-2\eta_2) 
    }(\pi_{\max}/\pi_{\min}), (\kappa_3^*)^2(\pi_{\max}/\pi_{\min})^3 r^{5\vee( 3+2\eta_1)\vee (3+4\eta_1)} \big\}
$; 
\item $       e_{\AAA,F}\ll (\kappa_2^*)^{-1}(\delta_2^*)^2\min\big\{
        r^{-(\eta_1-\eta_2)} 
        (\pi_{\min}/\pi_{\max}), (\kappa_3^*)^{-1} r^{-2-\eta_1} (\pi_{\min}/\pi_{\max})^2\big\}$;
\item and $n\gg r^{1+2(\eta_1-\eta_2)}\log(r)$.
\end{enumerate}
Then, with probability converging to $1$,
there is $\tilde{\TTT}_{\gN_2}=(\tilde{\ttt}_i^T)_{i\in\gN_2}\in\mathbb{R}^{|\gN_2|\times r}$ such that $S_{1,i}(\tilde{\ttt}_i;\hat{\AAA})=\mathbf{0}$ for all $i\in\gN_2$, %
and
\begin{equation}
\|\tilde{\TTT}_{\gN_2}-\TTT_{\gN_2}^*\|_{2\to\infty}\lesssim 
 \kappa_2^*(\delta_2^*)^{-1}(\pi_{\max}/\pi_{\min})p^{1/2}
\big\{ r (\log(n))^{1/2}(p\pi_{\max})^{-1/2}+ r^{1/2+\eta_1}e_{\AAA,F} \big\}.
\end{equation}
Moreover, $\tilde{\TTT}_{\gN_2}$ defined above satisfies $\|\tilde{\TTT}_{\gN_2}-\TTT_{\gN_2}^*\|_{2\to\infty}\leq C_1$, and $\tilde{\ttt}_i$ is the unique solution to the optimization problem  $\max_{\ttt_i\in\mathbb{R}^r}\sum_{j\in[p]}\omega_{ij}\{y_{ij}\ttt_i^T\hat{\aaaa}_j -b(\ttt_i^T\hat{\aaaa}_j)\}$ for all $i\in\gN_2$.
\end{lem}

\begin{proof}[Proof of Lemma~\ref{lemma:asymptotic-1st-theta}]

First, we provide analysis on the asymptotic regime. Note that $\kappa_2^*\geq \kappa_2(0)\gtrsim 1$ and $\delta_2^*\leq \delta_2(0)\lesssim 1$. Then, the 4-th requirement on the asymptotic regime, i.e.,
\begin{equation}
       p\pi_{\min}\gg (\delta_2^*)^{-4}(\kappa_2^*)^2(\log(n))^2\max\big\{ r^{1\vee (1+2\eta_1)\vee(1-2\eta_2) 
       }(\pi_{\max}/\pi_{\min}), (\kappa_3^*)^2(\pi_{\max}/\pi_{\min})^3 r^{5\vee( 3+2\eta_1)\vee (3+4\eta_1)} \big\}
\end{equation}
implies the following asymptotic regimes,
\begin{equation}\label{eq:1st-ppi-asymp-implications}
   p\pi_{\min}\gg
   \begin{cases}
       \max[\log(n),      r(\log n)^2,   r^{1+2\eta_1}\log (n)],\\ (\kappa_3^*)^2 (\kappa_2^*)^{-2}r^{3+2\eta_1}\log(n),\\
       (\kappa_3^*)^2 (\kappa_2^*)^{-2}r^{3+4\eta_1}\log(n),\\
      (\kappa_2^*)^2(\kappa_3^*)^2(\delta_2^*)^{-4}(\pi_{\max}/\pi_{\min})^3 r^{5} (\log(n)),\\
      (\pi_{\max}/\pi_{\min})(\kappa_2^*)^2(\delta_2^*)^{-2}r^{1-2\eta_2}\log(n).
   \end{cases}
\end{equation}
Similarly, the 5-th requirement on the asymptotic regime, i.e.,
\begin{equation}
        e_{\AAA,F}\ll (\kappa_2^*)^{-1}(\delta_2^*)^2\min\big\{
        r^{-(\eta_1-\eta_2)} 
        (\pi_{\min}/\pi_{\max}), (\kappa_3^*)^{-1} r^{-2-\eta_1} (\pi_{\min}/\pi_{\max})^2\big\}
\end{equation}
implies
\begin{equation}\label{eq:1st-ea-asymp-implications}
    e_{\AAA,F}\ll
    \begin{cases}
        r^{-1-\eta_1}(\kappa_3^*)^{-1}\kappa_2^*,\\
(\pi_{\min}/\pi_{\max})^{1/2},\\
(\kappa_2^*)^{-1}\delta_2^* r^{-(\eta_1-\eta_2)} (\pi_{\min}/\pi_{\max}),\\
(\kappa_3^*)^{-1}(\kappa_2^*)^{-1}(\delta_2^*)^2 r^{-2-\eta_1} (\pi_{\min}/\pi_{\max})^2,
    \end{cases}
\end{equation}
because $\eta_1-\eta_2\geq 0$ and $-1-2\eta_1>-2-\eta_1$.
According to the 6-th asymptotic requirement, \hl{$n\gg r^{1+2(\eta_1-\eta_2)}\log(r)$}, which implies $\sigma_r^2(\MM^*)/\sigma_1^2(\MM^*)\gg \|\UU_r^*\|_{2\to\infty}^2\log(r)$ and the assumption for Lemma~\ref{lemma:psi} holds. Thus, with probability converging to $1$, 
\begin{equation}\label{eq:psi-asymp}
	(np)^{1/2}r^{\eta_2}\lesssim \psi_r\leq \psi_1\leq (np)^{1/2}r^{\eta_1}.
\end{equation}
Also, we have
\begin{equation}\label{eq:C1C2-asymptotics}
    r^{1/2+\eta_2}p^{1/2}\lesssim C_1\lesssim  r^{1/2+\eta_1}p^{1/2}, C_2\lesssim r^{1/2}p^{-1/2}, \text{ and } C_1C_2 \lesssim r^{1+\eta_1}.
\end{equation}
Throughout the proof, we restrict the analysis on the event $\{\|\hat{\AAA}-\AAA^*\|_F\leq e_{\AAA,F}\}\cap \{p_{\max}\leq 2p\pi_{\max}\}\cap\{(np)^{1/2}r^{\eta_2}\lesssim \psi_r\leq \psi_1\leq (np)^{1/2}r^{\eta_1}\}$, 
which has probability converging to $1$
by the lemma's assumption, \eqref{eq:1st-ppi-asymp-implications}, \eqref{eq:psi-asymp}, and Lemma~\ref{lemma:gamma-bound}. 
On this event, we have that with probability at least $1-1/n$,
\begin{equation}
\max_{i\in\gN_2}\|\ZZ_{i\cdot}\dOmegai\hat{\AAA}\|\leq 
32 \{ \phi^{1/2}(\kappa_2^*)^{1/2}C_{2}\log^{1/2}(n)r^{1/2}(p\pi_{\max})^{1/2}\vee  r^{1/2}\phi C_2/(\rho+1) \log(n)\},
\end{equation}
according to Lemma~\ref{lemma:random-matrix}. 
Under the asymptotic regime that $\phi\lesssim 1$, $C_2\lesssim(r/p)^{1/2}$, the above inequality implies
\begin{equation}
\max_{i\in\gN_2}\|\ZZ_{i\cdot}\dOmegai\hat{\AAA}\|\lesssim 
 (\kappa_2^*)^{1/2}r \log^{1/2}(n)\pi_{\max}^{1/2}\vee  r p^{-1/2} \log(n).
\end{equation}
Note that $\kappa_2^*\gtrsim 1$.
According to \eqref{eq:1st-ppi-asymp-implications}, \hl{ $p\pi_{\min}\gg r(\log n)^2$}, which implies $r p^{-1/2} \log(n)\ll (\kappa_2^*)^{1/2}r \log^{1/2}(n)\pi_{\max}^{1/2}$. Thus, the above display implies 
\begin{equation}\label{eq:zz-a}
\max_{i\in\gN_2}\|\ZZ_{i\cdot}\dOmegai\hat{\AAA}\|\lesssim (\kappa_2^*)^{1/2}r \log^{1/2}(n)\pi_{\max}^{1/2} \lesssim \kappa_2^*r \log^{1/2}(n)\pi_{\max}^{1/2}
\end{equation}
with probability converging to $1$. 
Next, according to Lemma~\ref{lemma:bound-u-i}, with probability converging to $1$, we have
\begin{equation}
\begin{split} &\max_{i\in\gN_2}\|\BB_{1,i}(\hat{\AAA})\|\\
    \leq & \kappa_2^* \pi_{\max} C_1 \|\hat{\AAA}\|_2 \|\hat{\AAA}-\AAA^*\|_F + 64\log(n)\cdot (\pi_{\max}^{1/2}\kappa_2^* C_1C_2\|\hat{\AAA}-\AAA^*\|_F +  \kappa_2^* C_1C_2^2\log(n)).
    \end{split}.
\end{equation}
According to \eqref{eq:C1C2-asymptotics},  
$C_1C_2^2\lesssim r^{3/2+\eta_1}p^{-1/2}$. Also, note that $\|\hat{\AAA}\|_2\leq 1$. Thus, the above display implies {that with probability converging to $1$,}
\begin{equation}\label{eq:BA-bound}
\begin{split} \max_{i\in\gN_2}\|\BB_{1,i}(\hat{\AAA})\|
    \lesssim  \kappa_2^* \big\{ \pi_{\max} r^{1/2+\eta_1}p^{1/2}e_{\AAA,F}+ r^{1+\eta_1} ( \pi_{\max})^{1/2}\log(n) e_{\AAA,F} + r^{3/2+\eta_1}p^{-1/2}\log(n)\big\}
    \end{split}.
\end{equation}
According to \eqref{eq:1st-ppi-asymp-implications},  \hl{   $p\pi_{\min}\gg r(\log n)^2$}, which implies
   $ \pi_{\max}^{1/2} r^{1+\eta_1}\log(n)\ll \pi_{\max} r^{1/2+\eta_1}p^{1/2}$. Thus, \eqref{eq:BA-bound} implies
     {that with probability converging to $1$,}
\begin{equation}\label{eq:BB-a} \max_{i\in\gN_2}\|\BB_{1,i}(\hat{\AAA})\|
    \lesssim \kappa_2^* \big( \pi_{\max} r^{1/2+\eta_1}p^{1/2}e_{\AAA,F}+  r^{3/2+\eta_1}p^{-1/2}\log(n)\big).
\end{equation}
According to \eqref{eq:1st-ppi-asymp-implications}, \hl{ $p\pi_{\min}\gg r^{1+2\eta_1}\log(n)$}, which implies $r^{3/2+\eta_1}p^{-1/2}\log(n)\lesssim r \log^{1/2}(n)\pi_{\max}^{1/2}$.  This, together with equations~\eqref{eq:zz-a} and \eqref{eq:BB-a}, we have
\begin{equation}\label{eq:zz-bb-combined-a}
    \begin{split}
        \max_{i\in\gN_2}\{ \|\ZZ_{i\cdot}\dOmegai\hat{\AAA}\|+\|\BB_{1,i}(\hat{\AAA})\|\}\lesssim  \kappa_2^*\{ r \log^{1/2}(n)\pi_{\max}^{1/2} + \pi_{\max} r^{1/2+\eta_1}p^{1/2}e_{\AAA,F}\}
    \end{split}
\end{equation}
{with probability converging to $1$.}

We proceed to the analysis of $\max_{i\in\gN_2}\beta_{1,i}(\hat{\AAA})\kappa_3^*$.
According to Lemma~\ref{lemma:beta-bound}, with probability $1-1/n$
\begin{equation} 
\begin{split}
\max_{i\in\gN_2}\beta_{1,i}(\hat{\AAA})
\leq  C_1^2 C_2\{\pi_{\max} \|\hat{\AAA}-\AAA^*\|_F^2 +  4\pi_{\max}^{1/2} C_2(\log(n))^{1/2} \|\hat{\AAA}-\AAA^*\|_F+ 4 C_2^2\log(n)\}.
\end{split}
\end{equation}
Note that $C_1^2C_2\lesssim r^{3/2+2\eta_1}p^{1/2}$. Thus, the above display implies
\begin{equation}\label{eq:beta-bound-intermediate}
\begin{split}
\max_{i\in\gN_2}\beta_{1,i}(\hat{\AAA})\kappa_3^*
\leq  \kappa_3^* r^{3/2+2\eta_1}p^{1/2}\{\pi_{\max} e_{\AAA,F}^2 +  \pi_{\max}^{1/2} r^{1/2}p^{-1/2}(\log(n))^{1/2} e_{\AAA,F}+ rp^{-1}\log(n)\}.
\end{split}
\end{equation}
\sloppy First, according to \eqref{eq:1st-ea-asymp-implications}, \hl{$e_{\AAA,F}\lesssim r^{-1-\eta_1}(\kappa_3^*)^{-1}\kappa_2^*$, } which implies $\kappa_3^* r^{3/2+2\eta_1}p^{1/2} \pi_{\max}e_{\AAA,F}^2\lesssim \kappa_2^*\pi_{\max}r^{1/2+\eta_1}p^{1/2}e_{\AAA,F}$. Second, according to \eqref{eq:1st-ppi-asymp-implications}, \hl{$p\pi_{\min}\gg (\kappa_3^*)^2 (\kappa_2^*)^{-2}r^{3+2\eta_1}\log(n)$, } which implies $\kappa_3^*r^{3/2+2\eta_1}p^{1/2}\cdot\pi_{\max}^{1/2} r^{1/2}p^{-1/2}(\log(n))^{1/2} e_{\AAA,F}\lesssim \kappa_2^*\pi_{\max}r^{1/2+\eta_1}p^{1/2}e_{\AAA,F}$. 
Third, according to \eqref{eq:1st-ppi-asymp-implications}, \hl{$p\pi_{\min}\gg (\kappa_3^*)^2 (\kappa_2^*)^{-2}r^{3+4\eta_1}\log(n)$}, which implies $\kappa_3^* r^{3/2+2\eta_1}p^{1/2}\cdot rp^{-1}\log(n) \ll \kappa_2^* r \log^{1/2}(n)\pi_{\max}^{1/2}$.
Thus, \eqref{eq:beta-bound-intermediate} implies  {that with probability converging to one,}
\begin{equation}\label{eq:beta-a}
\begin{split}
\max_{i\in\gN_2}\beta_{1,i}(\hat{\AAA})\kappa_3^*\lesssim \kappa_2^*\{ r \log^{1/2}(n)\pi_{\max}^{1/2} + \pi_{\max} r^{1/2+\eta_1}p^{1/2}e_{\AAA,F}\}.
\end{split}
\end{equation}
Equations \eqref{eq:zz-bb-combined-a} and  \eqref{eq:beta-a} together imply {that with probability converging to $1$}
\begin{equation}\label{eq:numerator-a}
\begin{split}
    \max_{i\in\gN_2}\{\|\ZZ_{i\cdot}\dOmegai\hat{\AAA}\|+\|\BB_{1,i}(\hat{\AAA})\| +\beta_{1,i}(\hat{\AAA})\kappa_3^* \}
    \lesssim  \kappa_2^*\{ r \log^{1/2}(n)\pi_{\max}^{1/2} + \pi_{\max} r^{1/2+\eta_1}p^{1/2}e_{\AAA,F}\}.
\end{split}
\end{equation}
Next, we find a lower bound for $\sigma_r(\II_{1,i}(\hat{\AAA}))$. Note that $\sigma_r(\AAA^*)=1$ and $\|\AAA^*\|_{2\to\infty}^2\lesssim r/p$ by assumption. 
Under the asymptotic regime that $p\pi_{\min}\gg r(\log(n))^2$, $\pi_{\min}\sigma_r^2(\AAA^*)\geq 32\|\AAA^*\|_{2\to\infty}^2\log(n)$ for $n$ large enough. According to Lemma~\ref{lemma:submatrix}, with probability at least $1-1/(nr)$,
\begin{equation}\label{eq:omega-a-lower}
    \min_{i\in\gN_2}\sigma_r^2(\dOmegai\AAA^*)\geq 2^{-1}\pi_{\min}
\end{equation}
for $n$ and $p$ large enough.
According to Lemma~\ref{lemma:omega-ahat-a}, {with probability converging to $1$,} %
\begin{equation}
    \max_{i\in\gN_2} \|\dOmegai(\hat{\AAA}-\AAA^*)\|_2^2\lesssim\pi_{\max}e_{\AAA,F}^2+   \pi_{\max}^{1/2} (r/p)^{1/2}\log(n)e_{\AAA,F}+ (r/p)\log(n).
\end{equation}
First, according to \eqref{eq:1st-ea-asymp-implications}, $e_{\AAA,F}\ll (\pi_{\min}/\pi_{\max})^{1/2}$, which implies $\pi_{\max}e_{\AAA,F}^2\ll \pi_{\min}$. Second, according to \eqref{eq:1st-ppi-asymp-implications} and \eqref{eq:1st-ea-asymp-implications}, $e_{\AAA,F}\ll (\pi_{\min}/\pi_{\max})^{1/2}$ and $\pi_{\min}p\gg r(\log(n))^{2}$, which implies $e_{\AAA,F}\ll(\pi_{\min}/\pi_{\max})^{1/2}(\pi_{\min}p)^{1/2}r^{-1/2}(\log(n))^{-1}$. This  further implies $\pi_{\max}^{1/2} (r/p)^{1/2}\log(n)e_{\AAA,F}\ll \pi_{\min}$. Third, according to \eqref{eq:1st-ppi-asymp-implications}, $p\pi_{\min}\gg r(\log(n))^2$, which implies $(r/p)\log(n)\ll \pi_{\min}$. Combining the analysis, we have that with probability converging to one,
\begin{equation}
     \max_{i\in\gN_2} \|\dOmegai(\hat{\AAA}-\AAA^*)\|_2^2\ll \pi_{\min}.
\end{equation}
Combining the above display with \eqref{eq:omega-a-lower} and using Lemma~\ref{lemma:information}, we have that with probability converging to $1$,
\begin{equation}\label{eq:denominator-a}
    \min_{i\in\gN_2}\sigma_r(\II_{1,i}(\hat{\AAA}))\geq 2^{-3}\delta_2^*\pi_{\min}.
\end{equation}
So far, we have obtained upper bounds for $\max_{i\in\gN_2}\{\|\ZZ_{i\cdot}\dOmegai\hat{\AAA}\|+\|\BB_{1,i}(\hat{\AAA})\| +\beta_{1,i}(\hat{\AAA})\kappa_3^* \}$ and a lower bound for $\sigma_r(\II_{1,i}(\hat{\AAA}))$. In the rest of the proof, we restrict our analysis on the event that \eqref{eq:numerator-a} and \eqref{eq:denominator-a} hold.
To proceed, we verify conditions of of Lemma~\ref{lemma:finite-simplify-kappa}.
According to Lemma~\ref{lemma:gamma-bound}, on the event $p_{\max}\leq 2p\pi_{\max}$, 
$\max_{i\in\gN_2}\gamma_{1,i}(\hat{\AAA})\lesssim p\pi_{\max}(r/p)^{3/2} $. This and \eqref{eq:denominator-a} implies {with probability tending to 1}
\begin{equation}\label{eq:verify-condition-2}
\begin{split}
     &\min_{i\in\gN_2}\Big\{(\gamma_{1,i}(\hat{\AAA}))^{-1} (\kappa_3\big(3C_1C_2\big))^{-1}\sigma^2_r({\II}_{1,i}(\hat{\AAA}))\Big\}\\
     \gtrsim & (p\pi_{\max})^{-1}(r/p)^{-3/2}(\kappa_3^*)^{-1}\pi_{\min}^2(\delta_2^*)^2\\
     =& (\kappa_{3}^*)^{-1}(\delta_2^*)^2 p^{1/2}r^{-3/2}\pi_{\min}^2/\pi_{\max}.
\end{split}
\end{equation}
\sloppy According to \eqref{eq:1st-ppi-asymp-implications}, 
\hl{ $p\pi_{\min}\gg (\kappa_2^*)^2(\kappa_3^*)^2(\delta_2^*)^{-4}(\pi_{\max}/\pi_{\min})^3 r^{5} (\log(n))$}, which implies $\kappa_2^* r \log^{1/2}(n)\pi_{\max}^{1/2} \ll (\kappa_{3}^*)^{-1}(\delta_2^*)^2 p^{1/2}r^{-3/2}\pi_{\min}^2/\pi_{\max}$. According to \eqref{eq:1st-ea-asymp-implications} \hl{$e_{\AAA,F}\ll  (\kappa_3^*)^{-1}(\kappa_2^*)^{-1}(\delta_2^*)^2 r^{-2-\eta_1} (\pi_{\min}/\pi_{\max})^2$}, which implies $\kappa_2^*\pi_{\max} r^{1/2+\eta_1}p^{1/2}e_{\AAA,F}\ll (\kappa_{3}^*)^{-1}(\delta_2^*)^2 p^{1/2}r^{-3/2}\pi_{\min}^2/\pi_{\max}$. Combining the analysis, we have $\kappa_2^* r \log^{1/2}(n)\pi_{\max}^{1/2}+ \kappa_2^*\pi_{\max} r^{1/2+\eta_1}p^{1/2}e_{\AAA,F}\ll (\kappa_{3}^*)^{-1}(\delta_2^*)^2 p^{1/2}r^{-3/2}\pi_{\min}^2/\pi_{\max}$. This, together with \eqref{eq:verify-condition-2} implies
\begin{equation}\label{eq:verify-condition-3}
     \max_{i\in\gN_2}\{\|\ZZ_{i\cdot}\dOmegai\hat{\AAA}\|+\|\BB_{1,i}(\hat{\AAA})\| +\beta_{1,i}(\hat{\AAA})\kappa_3^* \}\ll \min_{i\in\gN_2}\{(\gamma_{1,i}(\hat{\AAA}))^{-1} (\kappa_3\big(3C_1C_2\big))^{-1}\sigma^2_r({\II}_{1,i}(\hat{\AAA}))\}.
\end{equation}
Next, according to \eqref{eq:denominator-a} and $C_1 = \{\|\UU_r^*\|_{2\to\infty}\vee (r/n)^{1/2}\}\cdot \sigma_1(\MM^*)$
\begin{equation}    
\min_{i\in\gN_2}\{\sigma_r(\II_{1,i}(\hat{\AAA}))C_1\}\gtrsim  \delta_2^*\pi_{\min} (r/n)^{1/2} ( np)^{1/2}r^{\eta_2}\gtrsim \delta_2^*\pi_{\min} r^{1/2+\eta_2}p^{1/2}.
\end{equation}
According to \eqref{eq:1st-ppi-asymp-implications}, \hl{$p\pi_{\min}\gg (\pi_{\max}/\pi_{\min})(\kappa_2^*)^2(\delta_2^*)^{-2}r^{1-2\eta_2}\log(n)$}, which implies $\kappa_2^* r \log^{1/2}(n)\pi_{\max}^{1/2}\ll \delta_2^*\pi_{\min} r^{1/2+\eta_2}p^{1/2}$. According to \eqref{eq:1st-ea-asymp-implications}, \hl{$e_{\AAA,F}\ll (\kappa_2^*)^{-1}\delta_2^*(\pi_{\min }/\pi_{\max})r^{-(\eta_1-\eta_2)}$}, which implies
$\kappa_2^*\pi_{\max} r^{1/2+\eta_1}p^{1/2}e_{\AAA,F}\ll \delta_2^*\pi_{\min} r^{1/2+\eta_2}p^{1/2}$. Combining the analysis and \eqref{eq:verify-condition-2}, we get
\begin{equation}\label{eq:verify-condition-4}
     \max_{i\in\gN_2}\{\|\ZZ_{i\cdot}\dOmegai\hat{\AAA}\|+\|\BB_{1,i}(\hat{\AAA})\| +\beta_{1,i}(\hat{\AAA})\kappa_3^* \}\ll \min_{i\in\gN_2}\{\sigma_r(\II_{1,i}(\hat{\AAA}))C_1\}.
\end{equation}

According to %
\eqref{eq:verify-condition-3} and \eqref{eq:verify-condition-4}, conditions of Lemma~\ref{lemma:finite-simplify-kappa} are satisfied. According to Lemma~\ref{lemma:finite-simplify-kappa} and \eqref{eq:numerator-a} and \eqref{eq:denominator-a}, with probability converging to $1$,
there exists $\tilde{\TTT}_{\gN_2}=(\tilde{\ttt}_i^T)_{i\in\gN_2}\in\mathbb{R}^{|\gN_2|\times r}$ such that $S_{1,i}(\tilde{\ttt}_i;\hat{\AAA})=\mathbf{0}$ for all $i\in\gN_2$, and
\begin{equation}
\begin{split}
     &\|\tilde{\TTT}_{\gN_2}-\TTT_{\gN_2}^*\|_{2\to\infty}\\
\leq &\max_{i\in\gN_2}\Big[ (\sigma_r(\II_{1,i}(\hat{\AAA})))^{-1}  \{\|\ZZ_{i\cdot}\dOmegai\hat{\AAA}\|+\|\BB_{1,i}(\hat{\AAA})\| +\beta_{1,i}(\hat{\AAA})\kappa_3^* \}\Big]\\
    \lesssim &  (\delta_2^*\pi_{\min})^{-1}\kappa_2^*\{ r \log^{1/2}(n)\pi_{\max}^{1/2} + \pi_{\max} r^{1/2+\eta_1}p^{1/2}e_{\AAA,F}\}\\
    = & \kappa_2^*(\delta_2^*)^{-1}(\pi_{\max}/\pi_{\min})p^{1/2}
\big\{ r (\log(n))^{1/2}(p\pi_{\max})^{-1/2}+ r^{1/2+\eta_1}e_{\AAA,F} \big\},
\end{split}
\end{equation}
and $\|\tilde{\TTT}_{\gN_2}-\TTT_{\gN_2}^*\|_{2\to\infty}\leq C_1$.
Moreover, $\tilde{\ttt}_i$ described above is the unique solution to to the optimization problem  $\max_{\ttt_i\in\mathbb{R}^r}\sum_{j\in[p]}\omega_{ij}\{y_{ij}\ttt_i^T\hat{\aaaa}_j -b(\ttt_i^T\hat{\aaaa}_j)\}$ for all $i\in\gN_2$ because this optimization is strictly convex by \eqref{eq:denominator-a}.  
\end{proof}

\begin{lem}[Asymptotic analysis for $\tilde{\AAA}$]\label{lemma:2nd-a-asymptotics}
Assume that $\lim_{n,p\to\infty}\pr(\|\tilde{\TTT}_{\gN_2}-\TTT^*_{\gN_2}\|_{2\to\infty}\leq e_{\TTT,2\to\infty})=1$.
Assume the the following asymptotic regime holds,
\begin{enumerate}
        \item $\phi\lesssim 1$;
    \item $\|\UU_r^*\|_{2\to\infty}\lesssim (r/n)^{1/2}$, $\|\VV_r^*\|_{2\to\infty} \lesssim (r/p)^{1/2}$, $C_2 \sim (r/p)^{1/2}$;
    \item $(np)^{1/2} r^{\eta_2} \lesssim \sigma_r(\MM^*)\leq \sigma_1(\MM^*) \lesssim (np)^{1/2} r^{\eta_1}$;
    \item 
    \begin{equation}
    \begin{split}
         &n\pi_{\min}\\
         \gg & (\kappa_2^*)^2(\delta_2^*)^{-4}   (\log(np))^2 \max \big\{(\pi_{\max}/\pi_{\min})r^{(1+2\eta_1-2\eta_2)\vee(1+2\eta_1-4\eta_2)
         },  (\kappa_3^*)^2(\pi_{\max}/\pi_{\min})^3 r^{5+8\eta_1-8\eta_2} \big\};
    \end{split}
    \end{equation}
    \item $e_{\TTT,2\to\infty}\leq C_1$ and
    \begin{equation}
    \begin{split}
          &e_{\TTT,2\to\infty}\\
          \ll&  (\delta_2^*)^2 (\kappa_2^*)^{-1}  p^{1/2}(\log(n p))^{-1}\\
          &\cdot\min\{(\pi_{\min}/  \pi_{\max}) r^{(-1/2-\eta_1+2\eta_2)\wedge(1/2+2\eta_2)},(\kappa_3^*)^{-1} (\pi_{\min}/\pi_{\max})^2 r^{(-5/2-4\eta_1+4\eta_2)\wedge (-3/2-3\eta_1+4\eta_2)}
          \}.
    \end{split}
          \end{equation}
\end{enumerate}
Then, with probability converging to $1$,
there is $\tilde{\AAA}=(\tilde{\aaaa}_j^T)_{j\in[p]}\in\mathbb{R}^{p\times r}$ such that $S_{2,j}(\tilde{\aaaa}_j;\tilde{\TTT}_{\gN_2})=\mathbf{0}$ for all $j\in[p]$, %
$\|\tilde{\AAA}-\AAA^*\|\leq C_2$, and
  \begin{equation}
        \|\tilde{\AAA}-\AAA^*\|_{2\to\infty}\lesssim   \kappa_2^*(\delta_2^*)^{-1}(\pi_{\max}/\pi_{\min})r^{-2\eta_2}\log(np)p^{-1/2}\Big\{r^{1+\eta_1}(n\pi_{\max})^{-1/2}  +  r^{(1+\eta_1)\vee 0}  p^{-1/2}e_{\TTT,2\to\infty}\Big\}.
    \end{equation} 
Moreover, $\tilde{\aaaa}_{j}$ defined above is the unique solution to the optimization problem  $\max_{\aaaa_j\in\mathbb{R}^r}\sum_{i\in\gN_2}\omega_{ij}\{y_{ij}\ttt_i^T\hat{\aaaa}_j -b(\ttt_i^T\hat{\aaaa}_j)\}$ for all $j\in[p]$.
\end{lem}
\begin{proof}[Proof of Lemma~\ref{lemma:2nd-a-asymptotics}]
First, the 4-th condition on the asymptotic regime, i.e.,
\begin{equation}
    \begin{split}
         &n\pi_{\min}\\
         \gg & (\kappa_2^*)^2(\delta_2^*)^{-4}   (\log(np))^2 \max \big\{(\pi_{\max}/\pi_{\min})r^{(1+2\eta_1-2\eta_2)\vee(1+2\eta_1-4\eta_2)
         },  (\kappa_3^*)^2(\pi_{\max}/\pi_{\min})^3 r^{5+8\eta_1-8\eta_2} \big\}
    \end{split}
    \end{equation}
implies the following asymptotic regime holds
\begin{equation}\label{eq:2nd-ppi}
    n\pi_{\min}
    \gg 
    \begin{cases}
       \log(p),\\
         r^{1+2\eta_1-2\eta_2}  \log(p),\\
(\kappa_2^*)^2(\kappa_3^*)^2(\delta_2^*)^{-4}(\pi_{\max}/\pi_{\min})^3 r^{5+8\eta_1-8\eta_2}(\log(n p))^2,\\
 (\kappa_2^*)^2(\delta_2^*)^{-2}(\pi_{\max}/\pi_{\min})r^{1+2\eta_1-4\eta_2}\log^2(np),
    \end{cases}
\end{equation}
and
$
	n\gg r^{1+2(\eta_1-\eta_2)}\log(r),
$
which ensures that the conditions of Lemma~\ref{lemma:psi} holds, and thus, $(np)^{1/2}r^{\eta_2}\lesssim \psi_r\leq\psi_2\lesssim (np)^{1/2}r^{\eta_2}$ with probability converging to $1$. 

The 5-th condition on the asymptotic regime, i.e.,
\begin{equation}
    \begin{split}
          &e_{\TTT,2\to\infty}\\
          \ll &  (\delta_2^*)^2 (\kappa_2^*)^{-1}  p^{1/2}(\log(n p))^{-1}\\
          &\cdot\min\{(\pi_{\min}/  \pi_{\max}) r^{(-1/2-\eta_1+2\eta_2)\wedge(1/2+2\eta_2)},(\kappa_3^*)^{-1} (\pi_{\min}/\pi_{\max})^2 r^{(-5/2-4\eta_1+4\eta_2)\wedge (-3/2-3\eta_1+4\eta_2)}
          \}
    \end{split}
          \end{equation}
implies
\begin{equation}\label{eq:2nd-etheta}
    e_{\TTT,2\to\infty}\ll
    \begin{cases}
        p^{1/2} r^{1/2+\eta_2}\lesssim C_1,\\
     \kappa_2^*(\kappa_3^*)^{-1}r^{-1/2}p^{1/2}\log(n p),\\
(\pi_{\min}/\pi_{\max})^{1/2} p^{1/2}r^{\eta_2},\\
  (\kappa_2^*)^{-1}(\kappa_3^*)^{-1}(\delta_2^*)^2 (\pi_{\min}/\pi_{\max})^2 p^{1/2} r^{(-5/2-4\eta_1+4\eta_2)\wedge (-3/2-3\eta_1+4\eta_2)}(\log(n p))^{-1},\\
  (\kappa_2^*)^{-1}\delta_2^*(\pi_{\min}/\pi_{\max})r^{(-1/2-\eta_1+2\eta_2)\wedge(1/2+2\eta_2)}(\log(np))^{-1}p^{1/2},
    \end{cases}
\end{equation}
where we used $\eta_2>-1/2-\eta_1+2\eta_2$ because $\eta_1-\eta_2\geq 0$.

Throughout the proof, we restrict the analysis on the event $\|\tilde{\TTT}_{\gN_2}-\TTT_{\gN_2}^*\|_{2\to\infty}\leq e_{\TTT,2\to\infty}\leq C_1 $, which has probability converging to $1$ as $n,p\to\infty$, according to the assumption of the lemma and \eqref{eq:2nd-etheta}. This also implies that $\|\tilde{\TTT}_{\gN_2}\|\leq 2C_1$ with probability converging to $1$.
    According to Lemma~\ref{lemma:random-matrix-2nd} and under the asymptotic regime \hl{$n\pi_{\max}\gg \log(p)$},     with probability converging to $1$,
   \begin{equation}    
  \begin{split}
      &\max_{j\in[p]}\|\ZZ_{\gN_2, j}^T\dOmegaj\tilde{\TTT}_{\gN_2}\|\\
\leq &
16\{ \phi^{1/2}(\kappa_2^*)^{1/2}C_{1}\log^{1/2}(pr)r^{1/2}(n\pi_{\max})^{1/2}\vee  r^{1/2}\phi C_1/(\rho+1) \log(pr)\}\\
&+16\|\tilde{\TTT}_{\gN_2}-\TTT^*_{\gN_2}\|_{2\to\infty}\cdot n\pi_{\max} \log(np)\{(\kappa_2^*\phi)^{1/2}\vee 1\}\\
\lesssim & (\kappa_2^*)^{1/2} p^{1/2}r^{1/2+\eta_1} \log^{1/2}(p)r^{1/2}(n\pi_{\max})^{1/2} +   r^{1/2} p^{1/2} r^{1/2+\eta_1}\log(p)\}\\
&+e_{\TTT,2\to\infty}n\pi_{\max} \log(np)(\kappa_2^*)^{1/2}\\
\lesssim & (\kappa_2^*)^{1/2}r^{1+\eta_1}p^{1/2}n^{1/2}\pi_{\max}^{1/2} \log^{1/2}(p) +e_{\TTT,2\to\infty}n\pi_{\max} \log(n\vee p)(\kappa_2^*)^{1/2},
  \end{split}
\end{equation}
where we used $r^{1/2} p^{1/2} r^{1/2+\eta_1}\log(p)\lesssim p^{1/2}r^{1+\eta_1} \log^{1/2}(p)(n\pi_{\max})^{1/2}$ under the asymptotic regime \hl{$n\pi_{\max}\gg\log(p)$} for the last inequality.

According to Lemma~\ref{lemma:B-2nd},
with probability converging to $1$,
       \begin{equation}
       \begin{split}
           \max_{j\in[p]}\|\BB_{2,j}(\tilde{\TTT}_{\gN_2})\| 
           \leq & 4 C_1C_2\kappa_2^* n\pi_{\max} \|\tilde{\TTT}_{\gN_2}-\TTT^*_{\gN_2}\|_{2\to\infty}
  \lesssim  \kappa_2^* r^{1+\eta_1}  n\pi_{\max} e_{\TTT,2\to\infty} 
       \end{split}
\end{equation}

According to Lemma~\ref{lemma:2nd-beta}, with probability converging to $1$,
\begin{equation}
   \max_{j\in[p]} \beta_{2,j}(\tilde{\TTT}_{\gN_2}^*)\leq 4 C_1 C_2^2 \|\tilde{\TTT}_{\gN_2}-\TTT_{\gN_2}^*\|_{2\to\infty}^2 n\pi_{\max}\lesssim r^{3/2+\eta_1}p^{-1/2} e_{\TTT,2\to\infty}^2 n \pi_{\max}.
\end{equation}
Combining the above analysis, we obtain that with probability converging to $1$,
  \begin{equation}    
  \begin{split}
        &\max_{j\in[p]}\{\|\ZZ_{\gN_2, j}^T\dOmegaj\tilde{\TTT}_{\gN_2}\|+\|\BB_{2,j}(\tilde{\TTT}_{\gN_2})\| + \beta_{2,j}(\tilde{\TTT}_{\gN_2})\kappa_3^*\} \\
\lesssim &
(\kappa_2^*)^{1/2}r^{1+\eta_1}p^{1/2}n^{1/2}\pi_{\max}^{1/2} \log^{1/2}(p) +e_{\TTT,2\to\infty}n\pi_{\max} \log(n\vee p)(\kappa_2^*)^{1/2}\\
&+ \kappa_2^* r^{1+\eta_1}  n\pi_{\max} e_{\TTT,2\to\infty} +  r^{3/2+\eta_1}p^{-1/2} e_{\TTT,2\to\infty}^2 n \pi_{\max}\kappa_3^*\\
\lesssim & (\kappa_2^*)^{1/2}r^{1+\eta_1}p^{1/2}n^{1/2}\pi_{\max}^{1/2} \log^{1/2}(p) \\
&+ \kappa_2^* r^{(1+\eta_1)\vee 0}\log(n p)  n\pi_{\max} e_{\TTT,2\to\infty} +  r^{3/2+\eta_1}p^{-1/2} e_{\TTT,2\to\infty}^2 n \pi_{\max}\kappa_3^*.
  \end{split}
\end{equation}
\sloppy Under the asymptotic regime that \hl{ $e_{\TTT,2\to\infty}\lesssim \kappa_2^*(\kappa_3^*)^{-1}r^{-1/2}p^{1/2}\log(n p)$}, $r^{3/2+\eta_1}p^{-1/2} e_{\TTT,2\to\infty}^2 n \pi_{\max}\kappa_3^*\lesssim  \kappa_2^* r^{1+\eta_1}\log(n p)  n\pi_{\max} e_{\TTT,2\to\infty}$. Thus, the above inequality implies
\begin{equation}\label{eq:2nd-numerator}
  \begin{split}
        &\max_{j\in[p]}\{\|\ZZ_{\gN_2, j}^T\dOmegaj\tilde{\TTT}_{\gN_2}\|+\|\BB_{2,j}(\tilde{\TTT}_{\gN_2})\| + \beta_{2,j}(\tilde{\TTT}_{\gN_2})\kappa_3^*\} \\
\lesssim & \kappa_2^*r^{1+\eta_1}p^{1/2}n^{1/2}\pi_{\max}^{1/2} \log(np) + \kappa_2^* r^{(1+\eta_1)\vee 0}\log(n p)  n\pi_{\max} e_{\TTT,2\to\infty}.\\
  \end{split}
\end{equation}

Next, we derive a lower bound for $\sigma_r(\II_{2,j}(\tilde{\TTT}_{\gN_2}))$. Under the asymptotic regime \hl{$ n\pi_{\min}\gg r^{1+2\eta_1-2\eta_2}  \log(p)$, and $e_{\TTT,2\to\infty} \ll (\pi_{\min}/\pi_{\max})^{1/2} p^{1/2}r^{\eta_2}$}, we have $n\pi_{\max}\gg \log(p)$, $\pi_{\min}(np)r^{2\eta_2}\gg r^{1+2\eta_1} p \log(p)$, and $e_{\TTT,2\to\infty}^2 n\pi_{\max}\ll \pi_{\min} (np)r^{2\eta_2}$.
Note that $\sigma_r^2(\TTT_{\gN_2}^*)\geq \sigma_r^2(\MM^*_{\gN_2,\cdot})\geq \psi_r^2\gtrsim (np)r^{2\eta_2}$ and $\|\TTT_{\gN_2}^*\|_{2\to\infty}\lesssim (r/n)^{1/2}\psi_1\lesssim r^{1/2+\eta_1}p^{1/2}$. Thus, under the same asymptotic regime, conditions of Lemma~\ref{lemma:2nd-info} hold. Therefore, with probability converging to $1$,
      \begin{equation}\label{eq:2nd-denominator}
        \sigma_r({\II}_{2,j}(\tilde{\TTT}_{\gN_2}))\geq 2^{-2}\delta_2^*\pi_{\min}\psi_r^2\gtrsim\delta_2^*\pi_{\min}(np)r^{2\eta_2}.
    \end{equation}

Note that 
\begin{equation}
    \begin{split}
       & \min_{j}\{2^{-2}(\gamma_{2,j}(\tilde{\TTT}_{\gN_2}))^{-1} (\kappa_3^*)^{-1}\sigma^2_r(\tilde{\TTT}_{\gN_2})\}\\
        \gtrsim & (C_1^3 n\pi_{\max})^{-1} (\kappa_3^*)^{-1}(\delta_2^*\pi_{\min}\psi_r^2 )^2\\
        \gtrsim & ((p^{1/2}r^{1/2+\eta_1})^3 n\pi_{\max} )^{-1}(\kappa_3^*)^{-1}(\delta_2^*)^2 \pi_{\min}^2 (np)^2 r^{4\eta_2}\\
        = & (\kappa_3^*)^{-1}(\delta_2^*)^2 (\pi_{\min}^2/\pi_{\max})p^{1/2}n r^{-3/2-3\eta_1+4\eta_2 }.
    \end{split}
\end{equation}

\sloppy Under the asymptotic regime \hl{$ n\pi_{\min}\gg (\kappa_2^*)^2(\kappa_3^*)^2(\delta_2^*)^{-4}(\pi_{\max}/\pi_{\min})^3 r^{5+8\eta_1-8\eta_2}(\log(n p))^2$},  we have $\kappa_2^* r^{1+\eta_1}\log(n p)p^{1/2}n^{1/2}\pi_{\max}^{1/2} \ll (\kappa_3^*)^{-1}(\delta_2^*)^2 (\pi_{\min}^2/\pi_{\max})p^{1/2}n r^{-3/2-3\eta_1+4\eta_2 }.$ Under the asymptotic regime \hl{$e_{\TTT,2\to\infty}\ll (\kappa_2^*)^{-1}(\kappa_3^*)^{-1}(\delta_2^*)^2 (\pi_{\min}/\pi_{\max})^2 p^{1/2} r^{(-5/2-4\eta_1+4\eta_2)\wedge (-3/2-3\eta_1+4\eta_2)}(\log(n p))^{-1}$}, we have $    \kappa_2^* r^{(1+\eta_1)\vee 0}\log(n p)\cdot n\pi_{\max} e_{\TTT,2\to\infty}\ll  (\kappa_3^*)^{-1}(\delta_2^*)^2 (\pi_{\min}^2/\pi_{\max})p^{1/2}n r^{-3/2-3\eta_1+4\eta_2 }$. Combining the analysis, we have $    \kappa_2^*r^{1+\eta_1}p^{1/2}n^{1/2}\pi_{\max}^{1/2} \log^{1/2}(np) + \kappa_2^* r^{(1+\eta_1)\vee 0}\log(n p)  n\pi_{\max} e_{\TTT,2\to\infty} \ll (\kappa_3^*)^{-1}(\delta_2^*)^2 (\pi_{\min}^2/\pi_{\max})p^{1/2}n r^{-3/2-3\eta_1+4\eta_2 }.
$
This further implies
\begin{equation}\label{eq:verify-2nd-condition1}
    \|\ZZ_{\gN_2, j}^T\dOmegaj\tilde{\TTT}_{\gN_2}\|+\|\BB_{2,j}(\tilde{\TTT}_{\gN_2})\| + \beta_{2,j}(\tilde{\TTT}_{\gN_2})\kappa_3^*\ll2^{-2}(\gamma_{2,j}(\tilde{\TTT}_{\gN_2}))^{-1} (\kappa_3^*)^{-1}\sigma^2_r(\II_{2,j}(\tilde{\TTT}_{\gN_2}))
\end{equation}
for all $j$.
According to \eqref{eq:2nd-denominator}, $\sigma_r(\II_{2,j}(\tilde{\TTT}_{\gN_2}))C_2\gtrsim \delta_2^*\pi_{\min}(np)r^{2\eta_2}(r/p)^{1/2}\gtrsim \delta_2^*\pi_{\min} n p^{1/2}r^{1/2+2\eta_2}$. 
According to \eqref{eq:2nd-ppi}, \hl{$n\pi_{\min}\gg (\kappa_2^*)^2(\delta_2^*)^{-2}(\pi_{\max}/\pi_{\min})r^{1+2\eta_1-4\eta_2}\log^2(np)$}, which implies $\kappa_2^* r^{1+\eta_1}\log(n p)p^{1/2}n^{1/2}\pi_{\max}^{1/2} \ll \delta_2^*\pi_{\min} n p^{1/2}r^{1/2+2\eta_2}.$ According to \eqref{eq:2nd-etheta}, \hl{$e_{\TTT,2\to\infty}\ll (\kappa_2^*)^{-1}\delta_2^*(\pi_{\min}/\pi_{\max})r^{(-1/2-\eta_1+2\eta_2)\wedge(1/2+2\eta_2)}(\log(np))^{-1}p^{1/2}$}, which implies $ \kappa_2^* r^{(1+\eta_1)\vee 0}\log(n p)\cdot n\pi_{\max} e_{\TTT,2\to\infty}\ll  \delta_2^*\pi_{\min} n p^{1/2}r^{1/2+2\eta_2}.$ Combine the analysis, we obtain 
\begin{equation}\label{eq:verify-2nd-condition2}
    \|\ZZ_{\gN_2, j}^T\dOmegaj\tilde{\TTT}_{\gN_2}\|+\|\BB_{2,j}(\tilde{\TTT}_{\gN_2})\| + \beta_{2,j}(\tilde{\TTT}_{\gN_2})\kappa_3^*\ll\sigma_r(\II_{2,j}(\tilde{\TTT}_{\gN_2}))C_2
\end{equation}
for all $j$.

The inequalities \eqref{eq:verify-2nd-condition1} and \eqref{eq:verify-2nd-condition2} verify conditions of Lemma~\ref{lemma:2nd-finite-bound} (with $C_1$ replaced by $2C_1$).
 According to  Lemma~\ref{lemma:2nd-finite-bound} and combining \eqref{eq:2nd-numerator} and \eqref{eq:2nd-denominator},
with probability converging to $1$,
\begin{equation}
\begin{split}
&\|\tilde{\AAA}-\AAA^*\|_{2\to\infty}\\
     \leq & \max_{j\in[p]}  \sigma_r^{-1}({\II}_{2,j}(\tilde{\TTT}_{\gN_2}))\{\|\ZZ_{\gN_2, j}^T\dOmegaj\tilde{\TTT}_{\gN_2}\|+\|\BB_{2,j}(\tilde{\TTT}_{\gN_2})\| + \beta_{2,j}(\tilde{\TTT}_{\gN_2})\kappa_3^*\}\\
     \lesssim & \kappa_2^*(\delta_2^*)^{-1}\pi_{\min}^{-1}(np)^{-1}r^{-2\eta_2}\Big\{r^{1+\eta_1}p^{1/2}n^{1/2}\pi_{\max}^{1/2} \log(np) +  r^{(1+\eta_1)\vee 0}\log(n p)  n\pi_{\max} e_{\TTT,2\to\infty}\Big\}\\
     \lesssim & \kappa_2^*(\delta_2^*)^{-1}(\pi_{\max}/\pi_{\min})r^{-2\eta_2}\log(np)p^{-1/2}\Big\{r^{1+\eta_1}(n\pi_{\max})^{-1/2}  +  r^{(1+\eta_1)\vee 0}  p^{-1/2}e_{\TTT,2\to\infty}\Big\}.
\end{split}
\end{equation}
According to \eqref{eq:verify-2nd-condition2}, $\|\tilde{\AAA}-\AAA^*\|_{2\to\infty}\leq \max_{j\in[p]}  \sigma_r^{-1}({\II}_{2,j}(\tilde{\TTT}_{\gN_2}))\{\|\ZZ_{\gN_2, j}^T\dOmegaj\tilde{\TTT}_{\gN_2}\|+\|\BB_{2,j}(\tilde{\TTT}_{\gN_2})\| + \beta_{2,j}(\tilde{\TTT}_{\gN_2})\kappa_3^*\}\leq C_2$. In addition, $\tilde{\aaaa}_j$ is the unique solution to to the optimization problem  $\max_{\aaaa_j\in\mathbb{R}^r}\sum_{i\in \gN_2}\omega_{ij}\{y_{ij}\ttt_i^T\hat{\aaaa}_j -b(\ttt_i^T\hat{\aaaa}_j)\}$ for all $j$ because this optimization is strictly convex by \eqref{eq:2nd-denominator}.

\end{proof}

\begin{lem}[Asymptotic analysis for $\tilde{\MM}_{\gN_2\cdot}=\tilde{\TTT}_{\gN_2}\tilde{\AAA}^T$]\label{lemma:asymptotic-m-max}
    Assume that $\lim_{n,p\to\infty}\pr(\|\hat{\MM}_{\gN_1\cdot}-\MM^*_{\gN_1\cdot}\|_F\leq e_{\MM,F})=1$, and the following asymptotic regime holds:
    \begin{enumerate}
               \item $\phi\lesssim 1$;
    \item $\|\UU_r^*\|_{2\to\infty} \lesssim (r/n)^{1/2}$, $\|\VV_r^*\|_{2\to\infty} \lesssim (r/p)^{1/2}$, $C_2 \sim (r/p)^{1/2}$;
    \item $(np)^{1/2} r^{\eta_2} \lesssim \sigma_r(\MM^*)\leq \sigma_1(\MM^*) \lesssim (np)^{1/2} r^{\eta_1}$ for some constants $\eta_1$ and $\eta_2$;
    \item 
 \begin{equation}
    \begin{split}
         &p\pi_{\min}\\
        \gg & (\kappa_2^*)^{4}(\delta_2^*)^{-6} (\log(np))^{3} \\
        &\cdot\max\Big[ (\pi_{\max}/\pi_{\min})^3 r^{(1+2\eta_1) \vee (3+ 2\eta_1-4\eta_2)\vee(1-4\eta_2)}, (\kappa_3^*)^2(\pi_{\max}/\pi_{\min})^5 r^{( 3+2\eta_1)\vee(3+4\eta_1) \vee \{7+8(\eta_1-\eta_2)\}\vee(5+6\eta_1-8\eta_2)}  \Big];
    \end{split}
    \end{equation}
        \item 
$
         n\pi_{\min}
         \gg  (\kappa_2^*)^2(\delta_2^*)^{-4}   (\log(np))^2 \max \big\{(\pi_{\max}/\pi_{\min})r^{(1+2\eta_1-2\eta_2)\vee(1+2\eta_1-4\eta_2)
         },  (\kappa_3^*)^2(\pi_{\max}/\pi_{\min})^3 r^{5+8\eta_1-8\eta_2} \big\}
$;
        \item
  \begin{equation}
        \begin{split}
            &(np)^{-1/2}e_{\MM,F}\\
            \ll & (\kappa_2^*)^{-2}(\delta_2^*)^{3} (\log(n p))^{-1} (\pi_{\min}/\pi_{\max})^3
            \min\big[  r^{(-\eta_1+\eta_2)\wedge(-1-2\eta_1+3\eta_2)\wedge(-\eta_1+3\eta_2) },(\kappa_3^*)^{-1}  r^{(-2-\eta_1)\vee \{-3-5(\eta_1-\eta_2)\}\wedge(-2-4\eta_1+5\eta_2)} \big].
        \end{split}
    \end{equation}
    \end{enumerate}
    Then, with probability converging to $1$,
    \begin{equation}
    \begin{split}
    &\|\tilde{\MM}_{\gN_2\cdot}-\MM^*_{\gN_2\cdot}\|_{\max}\\
    \lesssim &  (\delta_2^*)^{-2} (\kappa_2^*)^2 (\pi_{\max}/\pi_{\min})^2 \log^{3/2}(n p)\Big[r^{(5/2+2\eta_1-2\eta_2)\vee(3/2+\eta_1-2\eta_2)}   \{ (p\wedge n)\pi_{\max}\}^{-1/2} + r^{(2+3\eta_1-3\eta_2)\vee(1+2\eta_1-3\eta_2)} (np)^{-1/2}e_{\MM,F}\Big].
    \end{split}
    \end{equation}
\end{lem}
\begin{proof}[Proof of Lemma~\ref{lemma:asymptotic-m-max}]
First, we analyze the asymptotic regime assumption. The 4-th condition of the asymptotic regime, i.e., 
 \begin{equation}
    \begin{split}
         &p\pi_{\min}\\
        \gg & (\kappa_2^*)^{4}(\delta_2^*)^{-6} (\log(np))^{3} \\
        &\cdot\max\Big[ (\pi_{\max}/\pi_{\min})^3 r^{(1+2\eta_1) \vee (3+ 2\eta_1-4\eta_2)\vee(1-4\eta_2)}, (\kappa_3^*)^2(\pi_{\max}/\pi_{\min})^5 r^{( 3+2\eta_1)\vee(3+4\eta_1) \}\vee \{7+8(\eta_1-\eta_2)\vee(5+6\eta_1-8\eta_2)}  \Big]
    \end{split}
    \end{equation}
implies
\begin{equation}\label{eq:ppi-m}
    p\pi_{\min}\gg
    \begin{cases}
  (\delta_2^*)^{-4}(\kappa_2^*)^2(\log(n))^2\max\big\{ r^{1\vee (1+2\eta_1)
    \vee(1-2\eta_2) 
    }(\pi_{\max}/\pi_{\min}), (\kappa_3^*)^2(\pi_{\max}/\pi_{\min})^3 r^{5\vee( 3+2\eta_1)\vee (3+4\eta_1)} \big\},\\
 (\kappa_2^*)^{4}(\delta_2^*)^{-6}(\pi_{\max}/\pi_{\min})^3 (\log(n p))^{3} r^{(3+ 2\eta_1-4\eta_2)\vee(1-4\eta_2)},\\
(\kappa_3^*)^2 (\kappa_2^*)^{4}(\delta_2^*)^{-6}(\pi_{\max}/\pi_{\min})^5 r^{ \{7+8(\eta_1-\eta_2)\}\vee(5+6\eta_1-8\eta_2)}(\log(np))^{3}, 
    \end{cases}
\end{equation}
where we used the fact $1\leq (1+2\eta_1)\vee (1-2\eta_2)$, $3+ 2\eta_1-4\eta_2>2-2\eta_2$, and $7+8(\eta_1-\eta_2)>5$.

The 6-th condition of the asymptotic regime, i.e.,
   \begin{equation}
        \begin{split}
            &(np)^{-1/2}e_{\MM,F}\\
            \ll & (\kappa_2^*)^{-2}(\delta_2^*)^{3} (\log(n p))^{-1} (\pi_{\min}/\pi_{\max})^3
            \min\big[  r^{(-\eta_1+\eta_2)\wedge(-1-2\eta_1+3\eta_2)\wedge(-\eta_1+3\eta_2) },(\kappa_3^*)^{-1}  r^{(-2-\eta_1)\vee \{-3-5(\eta_1-\eta_2)\}\wedge(-2-4\eta_1+5\eta_2)} \big]
        \end{split}
    \end{equation}
implies
\begin{equation}\label{eq:etheta-m}
    \begin{split}
        (np)^{-1/2}e_{\MM,F}
        \ll 
        \begin{cases}
         r^{\eta_2},\\
         r^{\eta_2}
(\kappa_2^*)^{-1}(\delta_2^*)^2\min\big\{
        r^{-(\eta_1-\eta_2)} 
        (\pi_{\min}/\pi_{\max}), (\kappa_3^*)^{-1} r^{-2-\eta_1} (\pi_{\min}/\pi_{\max})^2\big\},\\
(\kappa_2^*)^{-2}(\delta_2^*)^{3}(\pi_{\min}/\pi_{\max})^{2}(\log(n p))^{-1}  r^{(-1-2\eta_1+3\eta_2)\wedge(-\eta_1+3\eta_2)},\\
(\kappa_2^*)^{-2}(\delta_2^*)^{3}(\pi_{\min}/\pi_{\max})^3(\log(n p))^{-1} (\kappa_3^*)^{-1}  r^{\{-3-5(\eta_1-\eta_2)\}\wedge(-2-4\eta_1+5\eta_2)},
        \end{cases}
    \end{split}
\end{equation}
where we used the fact that $\eta_2\geq -1 -2\eta_1+3\eta_2$ and $\eta_2-(\eta_1-\eta_2)\geq -1-2\eta_1+3\eta_2$.

\sloppy According to \eqref{eq:etheta-m},  \hl{$ e_{\MM,F}\ll (np)^{1/2}r^{\eta_2}\lesssim \psi_r$}, which implies that the conditions for Lemma~\ref{lemma:em-to-ea} holds. Thus, with probability converging to $1$, $\|\hat{\AAA}-\AAA^*\|_F\leq e_{\AAA,F}$, where $e_{\AAA,F}=8\psi_r^{-1}e_{\MM,F}$. Note that $e_{\AAA,F} \lesssim r^{-\eta_2} (np)^{-1/2}e_{\MM,F}$.
According to \eqref{eq:etheta-m}, \hl{$e_{\MM,F}\ll (np)^{1/2}r^{\eta_2}
(\kappa_2^*)^{-1}(\delta_2^*)^2\min\big\{
        r^{-(\eta_1-\eta_2)} 
        (\pi_{\min}/\pi_{\max}), (\kappa_3^*)^{-1} r^{-2-\eta_1} (\pi_{\min}/\pi_{\max})^2\big\}
 $,} which implies $e_{\AAA,F}\ll (\kappa_2^*)^{-1}(\delta_2^*)^2\min\big\{
        r^{-(\eta_1-\eta_2)} 
        (\pi_{\min}/\pi_{\max}), (\kappa_3^*)^{-1} r^{-2-\eta_1} (\pi_{\min}/\pi_{\max})^2\big\}$. 
According to \eqref{eq:ppi-m} %
 \hl{$
    p\pi_{\min}\gg (\delta_2^*)^{-4}(\kappa_2^*)^2(\log(n))^2\max\big\{ r^{1\vee (1+2\eta_1)
    \vee(1-2\eta_2) 
    }(\pi_{\max}/\pi_{\min}), (\kappa_3^*)^2(\pi_{\max}/\pi_{\min})^3 r^{5\vee( 3+2\eta_1)\vee (3+4\eta_1)} \big\}
$ 
}. Thus, the asymptotic regime of Lemma~\ref{lemma:asymptotic-1st-theta} is satisfied. 

According to Lemma~\ref{lemma:asymptotic-1st-theta}, 
$\|\hat{\TTT}_{\gN_2}-\TTT_{\gN_2}^*\|_{2\to\infty}\leq e_{\TTT_{\gN_2},2\to\infty}$, with probability converging to $1$,
for $e_{\TTT_{\gN_2},2\to\infty}$ satisfying
\begin{equation}\label{eq:rate-theta-2toinfty}
\begin{split}
    &e_{\TTT_{\gN_2},2\to\infty}\\ 
    \sim & \kappa_2^*(\delta_2^*)^{-1}(\pi_{\max}/\pi_{\min})p^{1/2}
\big\{ r (\log(n))^{1/2}(p\pi_{\max})^{-1/2}+ r^{1/2+\eta_1}e_{\AAA,F} \}\\
\lesssim & \kappa_2^*(\delta_2^*)^{-1}(\pi_{\max}/\pi_{\min})p^{1/2}
\big\{ r (\log(n))^{1/2}(p\pi_{\max})^{-1/2}+ r^{1/2+\eta_1}\cdot r^{-\eta_2} (np)^{-1/2}e_{\MM,F} \}.
\end{split}    
\end{equation}
Next, we verify that the asymptotic regime of Lemma~\ref{lemma:2nd-a-asymptotics} is satisfied.
We first verify conditions about $e_{\TTT,2\to\infty}$.
According to \eqref{eq:ppi-m},
\hl{
$
 p\pi_{\min}\gg (\kappa_2^*)^{4}(\delta_2^*)^{-6}(\pi_{\max}/\pi_{\min})^3 (\log(n p))^{3} r^{(3+ 2\eta_1-4\eta_2)\vee(1-4\eta_2)}
 $},
which implies
 \begin{equation}\label{eq:temp1}
\begin{split}
    \kappa_2^*(\delta_2^*)^{-1}(\pi_{\max}/\pi_{\min})p^{1/2}
 \cdot  r (\log(n))^{1/2} (p\pi_{\max})^{-1/2}
 \ll (\delta_2^*)^2 (\kappa_2^*)^{-1}  p^{1/2}(\log(np))^{-1}(\pi_{\min}/  \pi_{\max}) r^{(-1/2-\eta_1+2\eta_2)\wedge(1/2+2\eta_2)}.
\end{split}
 \end{equation}
 According to \eqref{eq:ppi-m}, 
\hl{$p\pi_{\min}\gg  (\kappa_3^*)^2 (\kappa_2^*)^{4}(\delta_2^*)^{-6}(\pi_{\max}/\pi_{\min})^5 r^{ \{7+8(\eta_1-\eta_2)\}\vee(5+6\eta_1-8\eta_2)}(\log(np))^{3} $},
which implies
\begin{equation}\label{eq:temp2}
\begin{split}
    &\kappa_2^*(\delta_2^*)^{-1}(\pi_{\max}/\pi_{\min})p^{1/2}
 r (\log(n))^{1/2}(p\pi_{\max})^{-1/2}\\
 \ll & (\delta_2^*)^2 (\kappa_2^*)^{-1}  p^{1/2}(\log(np))^{-1}(\kappa_3^*)^{-1} (\pi_{\min}/\pi_{\max})^2 r^{(-5/2-4\eta_1+4\eta_2)\wedge (-3/2-3\eta_1+4\eta_2)}.
\end{split}
\end{equation}
According to \eqref{eq:etheta-m}, \hl{
$
(np)^{-1/2}e_{\MM,F}
         \ll  (\kappa_2^*)^{-2}(\delta_2^*)^{3}(\pi_{\min}/\pi_{\max})^{2}(\log(n p))^{-1}  r^{(-1-2\eta_1+3\eta_2)\wedge(-\eta_1+3\eta_2)}$
}, which implies
\begin{equation}\label{eq:temp3}
    \begin{split}
        &\kappa_2^*(\delta_2^*)^{-1}(\pi_{\max}/\pi_{\min})p^{1/2}\cdot r^{1/2+\eta_1}\cdot r^{-\eta_2} (np)^{-1/2}e_{\MM,F}\\
        \ll & (\delta_2^*)^2 (\kappa_2^*)^{-1}  p^{1/2}(\log(np))^{-1}(\pi_{\min}/  \pi_{\max}) r^{(-1/2-\eta_1+2\eta_2)\wedge(1/2+2\eta_2)}.
    \end{split}
\end{equation}
According to \eqref{eq:etheta-m}, \hl{$(np)^{-1/2}e_{\MM,F}
         \ll (\kappa_2^*)^{-2}(\delta_2^*)^{3}(\pi_{\min}/\pi_{\max})^3(\log(n p))^{-1} (\kappa_3^*)^{-1}  r^{\{-3-5(\eta_1-\eta_2)\}\wedge(-2-4\eta_1+5\eta_2)}
          $}, which implies
\begin{equation}\label{eq:temp4}
    \begin{split}
        &\kappa_2^*(\delta_2^*)^{-1}(\pi_{\max}/\pi_{\min})p^{1/2}\cdot r^{1/2+\eta_1}\cdot r^{-\eta_2} (np)^{-1/2}e_{\MM,F}\\
        \ll & %
        (\delta_2^*)^2 (\kappa_2^*)^{-1}  p^{1/2}(\log(np))^{-1}(\kappa_3^*)^{-1} (\pi_{\min}/\pi_{\max})^2 r^{(-5/2-4\eta_1+4\eta_2)\wedge (-3/2-3\eta_1+4\eta_2)}.
    \end{split}
\end{equation}
Combining the equations \eqref{eq:temp1}--\eqref{eq:temp4}, we have
 \begin{equation}
\begin{split}
    &\kappa_2^*(\delta_2^*)^{-1}(\pi_{\max}/\pi_{\min})p^{1/2}
\big\{ r (\log(n))^{1/2}(p\pi_{\max})^{-1/2}+ r^{1/2+\eta_1-\eta_2} (np)^{-1/2}e_{\MM_{\gN_1,\cdot},F} \}\\
 \ll &
          (\delta_2^*)^2 (\kappa_2^*)^{-1}  p^{1/2}(\log(n p))^{-1}\\
          &\cdot\min\{(\pi_{\min}/  \pi_{\max}) r^{(-1/2-\eta_1+2\eta_2)\wedge(1/2+2\eta_2)},(\kappa_3^*)^{-1} (\pi_{\min}/\pi_{\max})^2 r^{(-5/2-4\eta_1+4\eta_2)\wedge (-3/2-3\eta_1+4\eta_2)}
          \}
\end{split}
 \end{equation}
 which implies $e_{\TTT_{\gN_2},2\to\infty}$ satisfies  the 5-th condition of the asymptotic regime of Lemma~\ref{lemma:2nd-a-asymptotics}.

 \sloppy On the other hand, according to the lemma's assumption, 
{  
    \begin{equation}
    \begin{split}
         &n\pi_{\min}\\
         \gg & (\kappa_2^*)^2(\delta_2^*)^{-4}   (\log(np))^2 \max \big\{(\pi_{\max}/\pi_{\min})r^{(1+2\eta_1-2\eta_2)\vee(1+2\eta_1-4\eta_2)
         },  (\kappa_3^*)^2(\pi_{\max}/\pi_{\min})^3 r^{5+8\eta_1-8\eta_2} \big\}.
    \end{split}
    \end{equation}
    }
         Thus, the other requirements for the asymptotic regime in Lemma~\ref{lemma:2nd-a-asymptotics} are also satisfied. 
         
According to Lemma~\ref{lemma:2nd-a-asymptotics}, we have $\|\tilde{\AAA}-\AAA^*\|_{2\to\infty}\leq e_{\AAA,2\to\infty}$ with probability converging to $1$,
 where 
 \begin{equation}\label{eq:rate-a-2toinfty}
     \begin{split}
       e_{\AAA,2\to\infty} \sim   \kappa_2^*(\delta_2^*)^{-1}(\pi_{\max}/\pi_{\min})r^{-2\eta_2}\log(np)\Big\{r^{1+\eta_1}p^{-1/2}(n\pi_{\max})^{-1/2}  +  r^{(1+\eta_1)\vee 0}  p^{-1/2}e_{\TTT,2\to\infty}\Big\}.
\end{split}
\end{equation}
Combining the above display with \eqref{eq:rate-theta-2toinfty}, we further have
\begin{equation}\label{eq:rate-a-2toinfty-1}
    \begin{split}
    & e_{\AAA,2\to\infty}\\
       \lesssim & \kappa_2^*(\delta_2^*)^{-1}(\pi_{\max}/\pi_{\min})r^{-2\eta_2}\log(np)p^{-1/2}\Big[r^{1+\eta_1}(n\pi_{\max})^{-1/2} \\
       & +  r^{(1+\eta_1)\vee 0}  p^{-1/2}\cdot \kappa_2^*(\delta_2^*)^{-1}(\pi_{\max}/\pi_{\min})p^{1/2}
\big\{ r (\log(n))^{1/2}(p\pi_{\max})^{-1/2}+ r^{1/2+\eta_1}\cdot r^{-\eta_2} (np)^{-1/2}e_{\MM,F} \}\Big]\\ 
\lesssim & (\delta_2^*)^{-2} (\kappa_2^*)^2 (\log(np))^{3/2}(\pi_{\max}/\pi_{\min})^2p^{-1/2}\Big[ r^{(2+\eta_1-2\eta_2)\vee (1-2\eta_2)} \{(p\wedge n)\pi_{\max}\}^{-1/2} + r^{(3/2+2\eta_1-3\eta_2)\vee(1/2+\eta_1-3\eta_2)}(np)^{-1/2}e_{\MM,F}\Big].    
     \end{split}
 \end{equation}

Now, we combine the above analysis to find an upper bound for $\|\tilde{\MM}_{\gN_2\cdot}-\MM_{\gN_2}^*\|_{\max}$.
Recall that $\tilde{\MM}_{\gN_2\cdot}=\tilde{\TTT}_{\gN_2}\tilde{\AAA}^T$. Thus, for $\hat{\PP}\in\mathcal{O}_{r\times r}$ defined in \eqref{eq:p-hat},  and $\TTT^*_{\gN_2}= (\UU_r^*)_{\gN_2\cdot}\DD_r^*\hat{\PP}$, $\AAA^* = \VV_r^*\hat{\PP}$, we have
\begin{equation}
    \begin{split}
        &\tilde{\MM}_{\gN_2\cdot}-\MM^*_{\gN_2\cdot}\\
         = & \tilde{\TTT}_{\gN_2}\tilde{\AAA}^T - (\UU_r^*)_{\gN_2\cdot}\DD_r^*(\VV_r^*)^T\\
         = &  \tilde{\TTT}_{\gN_2}\tilde{\AAA}^T - (\UU_r^*)_{\gN_2\cdot}\DD_r^*\hat{\PP}(\VV_r^*\hat{\PP})^T\\
         = & \tilde{\TTT}_{\gN_2}\tilde{\AAA}^T - \TTT_{\gN_2}^*(\AAA^*)^T\\
         = & (\tilde{\TTT}_{\gN_2}-\TTT_{\gN_2}^*)(\AAA^*)^T + \tilde{\TTT}_{\gN_2}(\tilde{\AAA}-\AAA^*)^T.
    \end{split}
\end{equation}
Therefore, according to Lemma~\ref{lemma:asymptotic-1st-theta}, with probability converging to $1$, 
\begin{equation}
\begin{split}
           &\|\tilde{\MM}_{\gN_2\cdot}-\MM^*_{\gN_2\cdot}\|_{\max}\\
           \leq & \|\tilde{\TTT}_{\gN_2}-\TTT^*_{\gN_2}\|_{2\to\infty}\|{\AAA}^*\|_{2\to\infty} +  \|\tilde{\AAA}-\AAA^*\|_{2\to\infty}\|\tilde{\TTT}_{\gN_2}\|_{2\to\infty} \\
           \lesssim &(r/p)^{1/2}e_{\TTT,2\to\infty} + p^{1/2}r^{1/2+\eta_1} e_{\AAA,2\to\infty}.
\end{split}
\end{equation}
Combine the above inequality with 
\eqref{eq:rate-theta-2toinfty}
 and \eqref{eq:rate-a-2toinfty-1}, we obtain
 \begin{equation}
     \begin{split}      &\|\tilde{\MM}_{\gN_2\cdot}-\MM^*_{\gN_2\cdot}\|_{\max}\\
     \lesssim & (r/p)^{1/2}\cdot \kappa_2^*(\delta_2^*)^{-1}(\pi_{\max}/\pi_{\min})p^{1/2}
\big\{ r (\log(n))^{1/2}(p\pi_{\max})^{-1/2}+ r^{1/2+\eta_1-\eta_2} (np)^{-1/2}e_{\MM,F} \}\\
&+p^{1/2}r^{1/2+\eta_1}(\delta_2^*)^{-2} (\kappa_2^*)^2 (\log(np))^{3/2}(\pi_{\max}/\pi_{\min})^2p^{-1/2} \\
&\cdot\Big[ r^{(2+\eta_1-2\eta_2)\vee (1-2\eta_2)} \{(p\wedge n)\pi_{\max}\}^{-1/2} + r^{(3/2+2\eta_1-3\eta_2)\vee(1/2+\eta_1-3\eta_2)}(np)^{-1/2}e_{\MM,F}\Big]  \\
     \lesssim & 
(\delta_2^*)^{-2} (\kappa_2^*)^2 (\pi_{\max}/\pi_{\min})^2 \log^{3/2}(n p)\Big[r^{(5/2+2\eta_1-2\eta_2)\vee(3/2+\eta_1-2\eta_2)}   \{ (p\wedge n)\pi_{\max}\}^{-1/2} + r^{(2+3\eta_1-3\eta_2)\vee(1+2\eta_1-3\eta_2)} (np)^{-1/2}e_{\MM,F}
       \Big].
\end{split}
 \end{equation}

\end{proof}

\subsection{Additional theoretical results for Method~\ref{meth:split} with data splitting}\label{sec:additional-split}
We provide the following theoretical result for $\tilde{\MM}$ obtained from Method~\ref{meth:split} that extends Theorem~\ref{thm:m-bound-same-eta} to allow $\sigma_r(\MM^*)$ and $\sigma_1(\MM^*)$ growing at different asymptotic orders and $\pi_{\min}$ and $\pi_{\max}$ decaying at different orders.
\begin{lem}[Asymptotic analysis for $\tilde{\MM}$ with data splitting]\label{lemma:asymptotic-m-max-both}
    Assume that $\lim_{n,p\to\infty}\pr(\|\hat{\MM}_{\gN_k\cdot}-\MM^*_{\gN_k\cdot}\|_F\leq e_{\MM,F})=1$ ($k=1,2$), and the following asymptotic regime holds:
    \begin{enumerate}
               \item $\phi\lesssim 1$;
    \item $\|\UU_r^*\|_{2\to\infty} \lesssim (r/n)^{1/2}$, $\|\VV_r^*\|_{2\to\infty} \lesssim (r/p)^{1/2}$, $C_2 \sim (r/p)^{1/2}$;
    \item $(np)^{1/2} r^{\eta_2} \lesssim \sigma_r(\MM^*)\leq \sigma_1(\MM^*) \lesssim (np)^{1/2} r^{\eta_1}$ for some constants $\eta_1$ and $\eta_2$;
    \item 
 \begin{equation}
    \begin{split}
         &p\pi_{\min}\\
        \gg & (\kappa_2^*)^{4}(\delta_2^*)^{-6} (\log(np))^{3} \\
        &\cdot\max\Big[ (\pi_{\max}/\pi_{\min})^3 r^{(1+2\eta_1) \vee (3+ 2\eta_1-4\eta_2)\vee(1-4\eta_2)}, (\kappa_3^*)^2(\pi_{\max}/\pi_{\min})^5 r^{( 3+2\eta_1)\vee(3+4\eta_1) \}\vee \{7+8(\eta_1-\eta_2)\vee(5+6\eta_1-8\eta_2)}  \Big];
    \end{split}
    \end{equation}
        \item 
$
         n\pi_{\min}
         \gg  (\kappa_2^*)^2(\delta_2^*)^{-4}   (\log(np))^2 \max \big\{(\pi_{\max}/\pi_{\min})r^{(1+2\eta_1-2\eta_2)\vee(1+2\eta_1-4\eta_2)
         },  (\kappa_3^*)^2(\pi_{\max}/\pi_{\min})^3 r^{5+8\eta_1-8\eta_2} \big\};
$
        \item
  \begin{equation}
        \begin{split}
            &(np)^{-1/2}e_{\MM,F}\\
            \ll & (\kappa_2^*)^{-2}(\delta_2^*)^{3} (\log(n p))^{-1} (\pi_{\min}/\pi_{\max})^3
            \min\big[  r^{(-\eta_1+\eta_2)\wedge(-1-2\eta_1+3\eta_2)\wedge(-\eta_1+3\eta_2) },(\kappa_3^*)^{-1}  r^{(-2-\eta_1)\vee \{-3-5(\eta_1-\eta_2)\}\wedge(-2-4\eta_1+5\eta_2)} \big].
        \end{split}
    \end{equation}
    \end{enumerate}
    Then, with probability converging to $1$, estimating equations in steps 3 and 4 of Method~\ref{meth:split} have a unique solution and
    \begin{equation}
    \begin{split}
    &\|\tilde{\MM}-\MM^*\|_{\max}\\
    \lesssim &  (\delta_2^*)^{-2} (\kappa_2^*)^2 (\pi_{\max}/\pi_{\min})^2 \log^{3/2}(n p)\Big[r^{(5/2+2\eta_1-2\eta_2)\vee(3/2+\eta_1-2\eta_2)}   \{ (p\wedge n)\pi_{\max}\}^{-1/2} + r^{(2+3\eta_1-3\eta_2)\vee(1+2\eta_1-3\eta_2)} (np)^{-1/2}e_{\MM,F}\Big].
    \end{split}
    \end{equation}
\end{lem}
\begin{proof}[Proof of Lemma~\ref{lemma:asymptotic-m-max-both}]
Recall that $\tilde{\MM}=(\tilde{m}_{ij})_{i\in[n],j\in[p]}$, where  $ (\tilde m_{ij})_{i \in \gN_1, j\in [p]}= \tilde{\TTT}_{\gN_1}^{(2)}(\tilde{\AAA}^{(2)})^{T}$ and 
 $(\tilde m_{ij})_{i \in \gN_2, j\in [p]}
=  \tilde{\TTT}_{\gN_2}^{(1)}(\tilde{\AAA}^{(1)})^{T}$. 
The error rate for $(\tilde m_{ij})_{i \in \gN_2, j\in [p]}
=  \tilde{\TTT}_{\gN_2}^{(1)}(\tilde{\AAA}^{(1)})^{T}$ is obtained by Lemma~\ref{lemma:asymptotic-m-max}, and the error rate of $(\tilde m_{ij})_{i \in \gN_1, j\in [p]}$ is obtained by swapping $(\hat{\AAA}^{(1)}, \tilde{\TTT}_{\gN_2}^{(1)},\tilde{\AAA}^{(1)}, \gN_1)$ with $(\hat{\AAA}^{(2)}, \tilde{\TTT}_{\gN_1}^{(2)},\tilde{\AAA}^{(2)}, \gN_2)$ in the proof of Lemma~\ref{lemma:asymptotic-m-max}.

The uniqueness of the solution to estimating equations in steps 3 and 4 of Method~\ref{meth:split} is proved by the uniqueness property in Lemma~\ref{lemma:asymptotic-1st-theta} and \ref{lemma:2nd-a-asymptotics}.

\end{proof}

\subsection{Proof of Theorem~\ref{thm:m-bound-same-eta}}\label{sec:proof-asymptotic-simple}
\begin{proof}[Proof of Theorem~\ref{thm:m-bound-same-eta}]
Note that when $\pi_{\min}\sim\pi_{\max}\sim\pi$ and $\eta_1=\eta_2=\eta$, the 4-th asymptotic requirement in Lemma~\ref{lemma:asymptotic-m-max-both} becomes
 \begin{equation}
    \begin{split}
         p\pi
        \gg & (\kappa_2^*)^{4}(\delta_2^*)^{-6} (\log(np))^{3} 
        \cdot\max\Big[ r^{(1+2\eta) \vee (3-2\eta)\vee(1-4\eta)}, (\kappa_3^*)^2r^{( 3+2\eta)\vee(3+4\eta) \vee \{7\vee(5-2\eta)\}}  \Big].
    \end{split}
    \end{equation}
When $\eta\geq -1$, the above requirement is implied by 
 \begin{equation}
    \begin{split}
         p\pi
        \gg & (\kappa_2^*)^{4}(\delta_2^*)^{-6} (\log(np))^{3} 
        \cdot\max\Big[ r^{(1+2\eta) \vee 5}, (\kappa_3^*)^2r^{(3+4\eta) \vee 7}  \Big],
    \end{split}
    \end{equation}
which is the asymptotic requirement {\bf R5}.

Similarly, the 5-th asymptotic requirement in Lemma~\ref{lemma:asymptotic-m-max-both} becomes
$
         n\pi
         \gg  (\kappa_2^*)^2(\delta_2^*)^{-4}   (\log(np))^2 \max \big\{r^{1\vee(1-2\eta)},  (\kappa_3^*)^2r^{5} \big\},
$
which is implied by the asymptotic requirement {\bf R6}: $
         n\pi
         \gg  (\kappa_2^*)^2(\delta_2^*)^{-4}   (\log(np))^2 \max \big\{r^{3},  (\kappa_3^*)^2r^{5} \big\}
$.

The 6-th asymptotic requirement becomes
\begin{equation}
        \begin{split}
            (np)^{-1/2}e_{\MM,F}
            \ll  (\kappa_2^*)^{-2}(\delta_2^*)^{3} (\log(n p))^{-1} 
            \min\big[  r^{0\wedge(-1+\eta)\wedge(2\eta) },(\kappa_3^*)^{-1}  r^{(-2-\eta)\wedge (-3)\wedge(-2+\eta)} \big],
        \end{split}
    \end{equation}
and is implied by $(np)^{-1/2}e_{\MM,F}
            \ll  (\kappa_2^*)^{-2}(\delta_2^*)^{3} (\log(n p))^{-1} 
            \min\big[  r^{-2},(\kappa_3^*)^{-1}  r^{-3} \big],$
and further implied by the asymptotic requirement {\bf R7'}.

Thus, under {\bf R1-R6} and {\bf R7'}, the conditions of Lemma~\ref{lemma:asymptotic-m-max-both} is satisfied and with probability converging to $1$,
    \begin{equation}\label{eq:error-proof-split}
    \begin{split}
    &\|\tilde{\MM}-\MM^*\|_{\max}\\
    \lesssim &  (\delta_2^*)^{-2} (\kappa_2^*)^2 \log^{3/2}(n p)\Big[r^{(5/2+2\eta_1-2\eta_2)\vee(3/2+\eta_1-2\eta_2)}   \{ (p\wedge n)\pi\}^{-1/2} + r^{(2+3\eta_1-3\eta_2)\vee(1+2\eta_1-3\eta_2)} (np)^{-1/2}e_{\MM,F}\Big]\\
    \lesssim &  (\delta_2^*)^{-2} (\kappa_2^*)^2 \log^{3/2}(n p)\Big[r^{5/2\vee(3/2-\eta)}   \{ (p\wedge n)\pi\}^{-1/2} + r^{2\vee(1-\eta)} (np)^{-1/2}e_{\MM,F}\Big]\\
        \lesssim &  (\delta_2^*)^{-2} (\kappa_2^*)^2 \log^{3/2}(n p)\Big[r^{5/2}   \{ (p\wedge n)\pi\}^{-1/2} + r^{2} (np)^{-1/2}e_{\MM,F}\Big]\\
                \lesssim &  (\delta_2^*)^{-2} (\kappa_2^*)^2 \log^{2}(n p)r^{5/2}\Big[   \{ (p\wedge n)\pi\}^{-1/2} +(np)^{-1/2}e_{\MM,F}\Big].
    \end{split}
    \end{equation}

The above analysis gives the error bound of $\tilde{\MM}$.

To proceed to prove the `in particular' part of the theorem. We note that if $r\lesssim 1$, then $\sigma_1(\MM^*)\sim\sigma_r(\MM^*)\sim(np)^{1/2}$ and $C_1 \sim n^{-1/2} \sigma_1(\MM^*) \lesssim p^{1/2}$ and $C_2\sim p^{-1/2}$. As a result, $\|\MM^*\|_{\max}\leq C_1C_2 \lesssim 1$ and thus $2\rho+1 \lesssim 1$. This implies that $\delta_2^*\gtrsim 1$, $\kappa_2^*,\kappa_3^*\lesssim 1$. The proof is completed by combining the above analysis with \eqref{eq:error-proof-split}.

\end{proof}

\section{Proof of Theorem~\ref{thm:m-bound-no-splitting} and additional theoretical results for Method~\ref{meth:nosplit} without data splitting}\label{sec:proof-nosplit}
In this section, we provide analysis for $\tilde{\TTT}$, $\tilde{\AAA}$, and $\tilde{\MM}$ obtained from Method~\ref{meth:nosplit} without data splitting. Let
\begin{equation}\label{eq:hat-p-nosplit}
	\hat{\PP}=\arg\min_{\PP\in\mathcal{O}_r}\|\hat{\VV}_r-\VV_r^*\PP\|_F
\end{equation}
and $\AAA^*=\VV_r^*\hat{\PP}$ and $\TTT^*=\UU_r^*\DD_r^*\hat{\PP}$. With similar derivations as those for Lemma~\ref{lemma:em-to-ea}, we have the following lemma.
\begin{lem}\label{lemma:em-to-ea-nosplit}	
   If $\lim_{n,p\to\infty }\pr(\|\hat{\MM}-\MM^*\|_F\geq e_{\MM,F})= 0$, $e_{\MM,F}$ is a non-random number (depending on  $n$ and $p$),   
    $\|\VV_r^*\|_{2\to\infty}\leq C_2$, $e_{\MM,F}\leq 2^{-1}\sigma_r(\MM^*)$, and $\hat{\PP}$ is defined in \eqref{eq:hat-p-nosplit} then 
    \begin{equation}
        \lim_{n,p\to\infty}\pr(\|\hat{\AAA}-\VV_r^*\hat{\PP}\|_F\geq e_{\AAA,F})= 0,
    \end{equation}
    where $e_{\AAA,F} = 8\sigma_r^{-1}(\MM^*)e_{\MM,F}$.
\end{lem}

The rest of the section is organized as follows.  In Section~\ref{sec:proof-prob-bounds-nosplit}, we obtain non-asymptotic probabilistic bounds for terms involved in the estimating equations in Step 3 and 4 of Method~\ref{meth:nosplit}. In Section~\ref{sec:asymp-no-split}, we obtain asymptotic error bounds for $\|\tilde{\TTT}-\TTT^*\|_{2\to\infty}$ (Lemma~\ref{lemma:asymptotic-1st-theta-nosplit}). In Section~\ref{sec:asymp-no-split-additional}, we provide error bound $\|\tilde{\MM}-\MM^*\|_{\max}$ (Lemma~\ref{lemma:asymptotic-m-max-nosplit}) under a general setting. Finally,  the proof of Theorem~\ref{thm:m-bound-no-splitting} is given in Section~\ref{sec:proof-asymptotic-simple-nosplit}.

\subsection{Non-asymptotic analysis}\label{sec:proof-non-prob-nosplit}\label{sec:proof-prob-bounds-nosplit}
We first analyze each term in Lemma~\ref{lemma:finite-simplify-kappa} with $\AAA=\hat{\AAA}$ obtained from Method~\ref{meth:nosplit} without data splitting.

	\begin{lem}[Upper bound for $\|\ZZ_{i\cdot}\dOmegai\hat{\AAA}\|$ without data splitting]\label{lemma:random-matrix-nosplit}
Assume $n\geq 2$.
$\|\MM^*\|_{\max}\leq \rho$. Assume that $\|\AAA^*\|_{2\to\infty}\leq C_2$ and $\hat{\AAA}$ may be dependent with $\OOO$, $n\geq r$.
Then, with probability at least $1-2(nr)^{-1}$,
\begin{equation}\label{eq:lemma-random-1-nosplit}
\begin{split}
&\max_{i\in[n]}\|\ZZ_{i\cdot}\dOmegai\hat{\AAA}\|\\
\leq & 
8\{ \phi^{1/2}(\kappa_2(2\rho+1))^{1/2}C_{2}\log^{1/2}(nr)r^{1/2}p_{\max}^{1/2}\vee  r^{1/2}\phi C_2/(\rho+1) \log(nr)\}\\
& +    8 \log(np) \{(\phi\kappa^*_2)^{1/2}\vee 1\}\cdot p_{\max}^{1/2}\|\hat{\AAA}-\AAA^*\|_F.
\end{split}		
\end{equation}
\end{lem}

\begin{proof}[Proof of Lemma~\ref{lemma:random-matrix-nosplit}]
	Note that
	\begin{equation}
		\|\ZZ_{i\cdot}\dOmegai\hat{\AAA}\|
		\leq \|\ZZ_{i\cdot}\dOmegai\AAA^*\|+\|\ZZ_{i\cdot}\dOmegai(\hat{\AAA}-\AAA^*)\|
	\end{equation}
	and
	\begin{equation}
		\|\ZZ_{i\cdot}\dOmegai(\hat{\AAA}-\AAA^*)\|=\|\sum_{j=1}^p z_{ij} \omega_{ij}(\hat{\aaaa}_j-\aaaa_j^*)\|\leq \|\ZZ\|_{\max} \sum_{j=1}^p \omega_{ij} \|\hat{\aaaa}_j-\aaaa_j^*\| \leq \|\ZZ\|_{\max} p_{\max}^{1/2}\|\hat{\AAA}-\AAA^*\|_F.
	\end{equation}
	Combining the above two inequalities and taking maximum over $i\in [n]$, we have
	\begin{equation}\label{eq:random-matrix-nosplit-decom}
			\max_{i\in[n]}\|\ZZ_{i\cdot}\dOmegai\hat{\AAA}\|
		\leq \max_{i\in[n]}\{\|\ZZ_{i\cdot}\dOmegai\AAA^*\|\} + \|\ZZ\|_{\max} p_{\max}^{1/2}\|\hat{\AAA}-\AAA^*\|_F.
	\end{equation}
	For the first term on the right-hand side of the above inequality, we follow a similar proof as that in the proof of Lemma~\ref{lemma:random-matrix} (with $\hat{\AAA}$ replaced by $\AAA^*$) and obtain that with probability at least $1-(nr)^{-1}$
	\begin{equation}
		 \max_{i\in[n]}\{\|\ZZ_{i\cdot}\dOmegai\AAA^*\|\}\leq 8\{ \phi^{1/2}(\kappa_2(2\rho+1))^{1/2}C_{2}\log^{1/2}(nr)r^{1/2}p_{\max}^{1/2}\vee  r^{1/2}\phi C_2/(\rho+1) \log(nr)\}.
	\end{equation}
	For the second term on the right-hand side of equation \eqref{eq:random-matrix-nosplit-decom}, we apply Lemma~\ref{lemma:z-max} and obtain that with probability at least $1-(np)^{-1}$,
	\begin{equation}
		 \|\ZZ\|_{\max} p_{\max}^{1/2}\|\hat{\AAA}-\AAA^*\|_F\leq 8 \log(np) \{(\phi\kappa^*_2)^{1/2}\vee 1\}\cdot p_{\max}^{1/2}\|\hat{\AAA}-\AAA^*\|_F.
	\end{equation}
	The proof is completed by combining the above two inequalities.
\end{proof}

\begin{lem}[Upper bound for $\|\BB_{1,i}(\hat{\AAA})\|$ without data splitting]\label{lemma:bound-u-i-nosplit}
 Let $\AAA^*=\VV_r^*\hat{\PP}$ and $\TTT^* = \UU_r^*\DD_r^*\hat{\PP}$. 
Assume  $\|\hat{\AAA}\|_{2\to\infty},\|\VV_r^*\|_{2\to\infty}\leq C_2$ and $\|\UU_r^*\DD_r^*\|_{2\to\infty}\leq C_1$, and $\hat{\AAA}$ may be dependent with  $\OOO_{i\cdot}$ Then, 
\begin{equation}
\|\BB_{1,i}(\hat{\AAA})\| 
  \leq C_1C_2 \kappa_2^* p_{\max}^{1/2}
		\|\hat{\AAA}-\AAA^*\|_F.
	\end{equation}
\end{lem}
\begin{proof}[Proof of Lemma~\ref{lemma:bound-u-i-nosplit}]
First, by the assumptions and $\hat{\PP}$ is orthogonal, $\|\TTT^*\|_{2\to\infty}=\|\UU_r^*\DD_r^*\|_{2\to\infty}\leq C_1$ and $\|\AAA^*\|_{2\to\infty} = \|\VV_r^*\|_{2\to\infty}\leq C_2$. 
Recall that
\begin{equation}
\|\BB_{1,i}(\hat{\AAA})\| = \|\sum_{j=1}^p \omega_{ij}b''(m^*_{ij})
		\hat{\aaaa}_j(\hat{\aaaa}_j-\aaaa_j^*)^T\ttt^*_i\| 
  \leq C_1C_2\sum_{j=1}^p \omega_{ij}b''(m^*_{ij})
		\|\hat{\aaaa}_j-\aaaa_j^*\|
  \leq C_1C_2 \kappa_2^* \sum_{j=1}^p \omega_{ij}
		\|\hat{\aaaa}_j-\aaaa_j^*\|.
	\end{equation}
Applying  Cauchy-Schwarz inequality, we further obtain
\begin{equation}
\|\BB_{1,i}(\hat{\AAA})\| 
  \leq C_1C_2 \kappa_2^* p_{\max}^{1/2}
		\|\hat{\AAA}-\AAA^*\|_F.
	\end{equation}
The proof is completed by taking maximum for $i\in[n]$.
\end{proof}

\begin{lem}[Bound for $\beta_{1,i}(\hat{\AAA})$, without data splitting]\label{lemma:beta-bound-nosplit}
If $\|\UU_r^*\DD_r^*\|_{2\to\infty}\leq C_1$, $\|\hat{\AAA}\|_{2\to\infty},\|\VV_r^*\|_{2\to\infty}\leq C_2$, then,
\begin{equation}
   \max_{i\in[n]}\beta_{1,i}(\hat{\AAA})\leq C_1^2C_2 \|\hat{\AAA}-\AAA^*\|_F^2.
\end{equation}
\end{lem}
\begin{proof}[Proof of Lemma~\ref{lemma:beta-bound-nosplit}]
Recall
\begin{equation}
	\beta_{1,i}(\hat{\AAA})=  \sup_{\|\uu\|=1} \sum_{j}\omega_{ij} ((\hat{\aaaa}_j-\aaaa_j^*)^T\ttt^*_i)^2|\hat{\aaaa}_j^T\uu|\leq C_1^2C_2 \sum_{j\in[p]}\omega_{ij}\|\hat{\aaaa}_j-\aaaa_j^*\|^2\leq C_1^2C_2 \|\hat{\AAA}-\AAA^*\|_F^2.
\end{equation}
\end{proof}

\begin{lem}[Bound for $\gamma_{1,i}(\hat{\AAA})$, without data splitting]\label{lemma:gamma-bound-nosplit}
    If $\|\hat{\AAA}\|_{2\to\infty}\leq C_2$, then with probability at least $1-1/n$,
    \begin{equation}
    \max_{i\in[n]}\gamma_{1,i}(\hat{\AAA})\leq 2p \pi_{\max} C_2^3.
\end{equation}
\end{lem}
\begin{proof}[Proof of Lemma~\ref{lemma:gamma-bound-nosplit}]
The proof of this Lemma is the same as that of Lemma~\ref{lemma:gamma-bound} which does not require the independence between $\hat{\AAA}$ and $\OOO_{i\cdot}$.
\end{proof}

\begin{lem}\label{lemma:omega-ahat-a-nosplit}
\begin{equation}
    \max_{i\in [n]} \|\dOmegai(\hat{\AAA}-\AAA^*)\|_2^2\leq \|\hat{\AAA}-\AAA^*\|_F^2. 
\end{equation}
\end{lem}
\begin{proof}[Proof of Lemma~\ref{lemma:omega-ahat-a-nosplit}]
\begin{equation}
\begin{split}
\max_{i\in [n]} \|\dOmegai(\hat{\AAA}-\AAA^*)\|_2^2
   \leq \max_{i\in [n]} \|\dOmegai\|_2^2\|\hat{\AAA}-\AAA^*\|_F^2 = \|\hat{\AAA}-\AAA^*\|_F^2. 
\end{split}	
\end{equation}

\end{proof}

\subsection{Asymptotic analysis for Method~\ref{meth:nosplit} without data splitting}\label{sec:asymp-no-split}
\begin{lem}[Asymptotic analysis of $\tilde{\AAA}$ without data splitting]\label{lemma:asymptotic-1st-theta-nosplit}
    Let $\AAA^*=\VV_r^*\hat{\PP}$, $\TTT^*=\UU_r^*\DD_r^*\hat{\PP}$, and $\hat{\PP}$ is defined in \eqref{eq:hat-p-nosplit}. Assume that $\lim_{n,p\to\infty}\pr(\|\hat{\AAA}-\AAA^*\|_{F}\geq e_{\AAA,F})=1$.

Assume the following asymptotic regime holds:
\begin{enumerate}
        \item $\phi\sim 1$, $\pi_{\min}\sim\pi_{\max}\sim\pi$;
    \item $\|\UU_r^*\|_{2\to\infty}\lesssim (r/n)^{1/2}$, $\|\VV_r^*\|_{2\to\infty} \lesssim (r/p)^{1/2}$, $C_2 \sim (r/p)^{1/2}$;
\item $(np)^{1/2}r^{\eta_2}\lesssim \sigma_r(\MM^*)\leq \sigma_1(\MM^*)\lesssim (np)^{1/2}r^{\eta_1}$, and $\eta_1$ and $\eta_2$ are constants;
\item        $p\pi\gg (\delta_2^*)^{-4}(\kappa_2^*)^2\log^2(n)\max\big\{ r^{1\vee (1-2\eta_2)}
,(\kappa_3^*)^2 r^{5} \big\}$;
\item $e_{\AAA,F}\ll (\kappa_2^*)^{-1}(\delta_2^*)^2 (\log(np))^{-1}
\min\{r^{0\wedge(-1/2-\eta_1+\eta_2)\wedge(1/2+\eta_2)},(\kappa_3^*)^{-1}r^{ (-5/2-\eta_1)\wedge (-3/2)} \}\pi^{1/2}$.
\end{enumerate}
Then, with probability converging to $1$,
there is $\tilde{\TTT}=(\tilde{\ttt}_i^T)_{i\in[n]}$ such that $S_{1,i}(\tilde{\ttt}_i,\hat{\AAA})=\mathbf{0}$, for all $i\in[n]$,$\|\tilde{\TTT}-\TTT^*\|_{2\to\infty}\leq C_1$, %
and
\begin{equation}
\|\tilde{\TTT}-\TTT^*\|_{2\to\infty}\lesssim 
 \kappa_2^*(\delta_2^*)^{-1}\pi^{-1/2}
\big\{ r (\log(n))^{1/2}+ \log (np)r^{(1+\eta_1)\vee 0}p^{1/2}e_{\AAA,F} \big\}.
\end{equation}
Moreover,  $\tilde{\ttt}_i$ is the unique solution to the optimization problem  $\max_{\ttt_i\in\mathbb{R}^r}\sum_{j\in[p]}\omega_{ij}\{y_{ij}\ttt_i^T\hat{\aaaa}_j -b(\ttt_i^T\hat{\aaaa}_j)\}$ for all $i\in[n]$.

\end{lem}
\begin{proof}[Proof of Lemma~\ref{lemma:asymptotic-1st-theta-nosplit}]
First, we provide analysis on the asymptotic regime. Note that $\kappa_2^*\geq \kappa_2(0)\gtrsim 1$ and $\delta_2^*\leq \delta_2(0)\lesssim 1$. Then, the 4-th condition on the asymptotic regime, i.e.,
\begin{equation}
       p\pi\gg (\delta_2^*)^{-4}(\kappa_2^*)^2\log^2(n)\max\big\{ r^{1\vee (1-2\eta_2)}
,(\kappa_3^*)^2 r^{5} \big\}
\end{equation}
implies the following asymptotic regimes,
\begin{equation}\label{eq:1st-ppi-asymp-implications-nosplit}
   p\pi\gg
   \begin{cases}
          r(\log n)^2,\\   
             (\kappa_2^*)^2(\kappa_3^*)^2 (\delta_2^*)^{-4}r^5\log(n),\\
             (\kappa_2^*)^2(\delta_2^*)^{-2}\log(n)r^{1-2\eta_2}.
             \end{cases}
\end{equation}
Similarly, the 5-th condition on the asymptotic regime, i.e.,
\begin{equation}
        e_{\AAA,F}\ll (\kappa_2^*)^{-1}(\delta_2^*)^2 (\log(np))^{-1}
\min\{r^{0\wedge(-1/2-\eta_1+\eta_2)\wedge(1/2+\eta_2)},(\kappa_3^*)^{-1}r^{ (-5/2-\eta_1)\wedge (-3/2)} \}\pi^{1/2}
\end{equation}
implies
\begin{equation}\label{eq:1st-ea-asymp-implications-nosplit}
    e_{\AAA,F}\ll
    \begin{cases}
    (\kappa_3^*)^{-1} r^{-1/2-\eta_1}\pi^{1/2},\\
\pi^{1/2},\\
(\kappa_2^*)^{-1}(\kappa_3^*)^{-1}(\delta_2^*)^2 (\log(np))^{-1}r^{(-5/2-\eta_1)\wedge (-3/2)}\pi^{1/2},\\
\delta_2^*(\kappa_2^*)^{-1}(\log(np))^{-1}r^{(-1/2-\eta_1+\eta_2)\wedge(1/2+\eta_2)}\pi^{1/2},
    \end{cases}
\end{equation}
where we used the fact that $-1/2-\eta_1>-5/2-\eta_1$.

Throughout the proof, we restrict the analysis on the event $\{\|\hat{\AAA}-\AAA^*\|_F\leq e_{\AAA,F}\}\cap \{p_{\max}\leq 2p\pi_{\max}\}$, 
which has probability converging to $1$
by the lemma's assumption, and Lemma~\ref{lemma:p-max-bound}. %
On this event, we have that with probability at least $1-1/n$,
\begin{equation}
\begin{split}
&\max_{i\in[n]}\|\ZZ_{i\cdot}\dOmegai\hat{\AAA}\|\\
\leq & 
16\{ \phi^{1/2}(\kappa_2(2\rho+1))^{1/2}C_{2}\log^{1/2}(nr)r^{1/2}(p \pi_{\max})^{1/2}\vee  r^{1/2}\phi C_2/(\rho+1) \log(nr)\}\\
& +    8\{(\phi\kappa_2^*)^{1/2}\vee 1\}\log(np) \cdot (p\pi_{\max})^{1/2}e_{A,F}.
\end{split}		
\end{equation}
according to Lemma~\ref{lemma:random-matrix-nosplit}. Under the asymptotic regime that $\phi\lesssim 1$, $\pi_{\min}\sim \pi_{\max}\sim \pi$, $C_2\lesssim (r/p)^{1/2}$, the above inequality implies
\begin{equation}
\max_{i\in\gN_2}\|\ZZ_{i\cdot}\dOmegai\hat{\AAA}\|\lesssim 
 (\kappa_2^*)^{1/2}r \log^{1/2}(n)\pi^{1/2}+  r p^{-1/2} \log(n) + (\kappa_2^*)^{1/2}\log(np) p^{1/2} \pi^{1/2}e_{A,F}.
\end{equation}

According to \eqref{eq:1st-ppi-asymp-implications-nosplit}, \hl{ $p\pi\gg r(\log n)^2$}, which implies $r p^{-1/2} \log(n)\ll (\kappa_2^*)^{1/2}r \log^{1/2}(n)\pi^{1/2}$. Thus, the above display implies 
\begin{equation}\label{eq:zz-a-nosplit}
\max_{i\in\gN_2}\|\ZZ_{i\cdot}\dOmegai\hat{\AAA}\|\lesssim (\kappa_2^*)^{1/2}r \log^{1/2}(n)\pi^{1/2} +(\kappa_2^*)^{1/2}\log(np) p^{1/2} \pi^{1/2}e_{A,F} 
\end{equation}
with probability converging to $1$.

Next, according to Lemma~\ref{lemma:bound-u-i-nosplit},
\begin{equation}
\max_{i\in[n]}\|\BB_{1,i}(\hat{\AAA})\| 
  \leq C_1C_2 \kappa_2^* p_{\max}^{1/2}
		\|\hat{\AAA}-\AAA^*\|_F.
	\end{equation}
Note that $C_1C_2 \lesssim r^{1+\eta_1}$. Thus, the above display implies that with probability converging to one,
\begin{equation}\label{eq:BB-a-nosplit}
\begin{split} \max_{i\in[n]}\|\BB_{1,i}(\hat{\AAA})\|
    \lesssim  \kappa_2^* r^{1+\eta_1}p^{1/2}\pi^{1/2}e_{A,F}
    \end{split}.
\end{equation}
Combining equations~\eqref{eq:zz-a-nosplit} and \eqref{eq:BB-a-nosplit}, we obtain
\begin{equation}\label{eq:zz-bb-combined-a-nosplit}
    \begin{split}
        \max_{i\in[n]}\{ \|\ZZ_{i\cdot}\dOmegai\hat{\AAA}\|+\|\BB_{1,i}(\hat{\AAA})\|\}\lesssim 
        \kappa_2^*\{r \log^{1/2}(n)\pi^{1/2} +\log(np)  r^{(1+\eta_1)\vee 0} p^{1/2} \pi^{1/2}e_{A,F}\}. 
    \end{split}
\end{equation}
Next, we consider $\max_{i\in[n]}\{\beta_{1,i}(\hat{\AAA})\}\kappa_3^*$. According to Lemma~\ref{lemma:beta-bound-nosplit}, we have 
\begin{equation}
   \max_{i\in[n]}\beta_{1,i}(\hat{\AAA})\leq  C_1^2C_2 \|\hat{\AAA}-\AAA^*\|_F^2.
   \end{equation}
  Note that $C_1^2C_2\lesssim r^{3/2+2\eta_1}p^{1/2}$. Thus, the above display implies
\begin{equation}\label{eq:beta-bound-intermediate-nosplit}
\begin{split}
\max_{i\in[n]}\{\beta_{1,i}(\hat{\AAA})\}\kappa_3^*
\lesssim   \kappa_3^*  r^{3/2+2\eta_1} p^{1/2}e_{A,F}^2. 
\end{split}
\end{equation}
According to \eqref{eq:1st-ea-asymp-implications-nosplit}, \hl{$e_{A,F}\ll (\kappa_3^*)^{-1} r^{-1/2-\eta_1}\pi^{1/2}$}. This implies $\kappa_3^* r^{3/2+2\eta_1}p^{1/2}e_{A,F}^2\lesssim\kappa_2^* \log(np)r^{(1+\eta_1)\vee 0} p^{1/2}\pi^{1/2}e_{A,F}$. Thus, combining \eqref{eq:zz-bb-combined-a-nosplit} and \eqref{eq:beta-bound-intermediate-nosplit}, we obtain
\begin{equation}\label{eq:numerator-a-nosplit}
\begin{split}
    \max_{i\in[n]}\{\|\ZZ_{i\cdot}\dOmegai\hat{\AAA}\|+\|\BB_{1,i}(\hat{\AAA})\| +\beta_{1,i}(\hat{\AAA})\kappa_3^* \}
    \lesssim  \kappa_2^* \{\log^{1/2}(n)r\pi^{1/2} + \log(np) r^{(1+\eta_1)\vee 0} p^{1/2} \pi^{1/2}e_{A,F} \}.
\end{split}
\end{equation}
Next, we find a lower bound for $\sigma_r(\II_{1,i}(\hat{\AAA}))$. With similar derivations as those for \eqref{eq:omega-a-lower}, we have
\begin{equation}
    \min_{i\in[n]}\sigma_r^2(\dOmegai\AAA^*)\geq 2^{-1}\pi
\end{equation}
with probability converging to $1$ under the asymptotic regime $p\pi\gg r(\log(n))^2$. According to Lemma~\ref{lemma:omega-ahat-a-nosplit},
\begin{equation}
	    \max_{i\in [n]} \|\dOmegai(\hat{\AAA}-\AAA^*)\|_2^2\leq \|\hat{\AAA}-\AAA^*\|_F^2\leq e_{A,F}^2.
\end{equation}
According to \eqref{eq:1st-ea-asymp-implications-nosplit}, \hl{$e_{A,F}\ll \pi^{1/2}$}. Thus, the above two inequalities and Lemma~\ref{lemma:information} together imply that with probability converging to $1$,
\begin{equation}\label{eq:denominator-a-nosplit}
   \min_{i\in[n]} \sigma_r(\II_{1,i}(\hat{\AAA}))\geq 2^{-3}\delta_2^*\pi.
\end{equation}
Next, we verify conditions of Lemma~\ref{lemma:finite-simplify-kappa}.
According to Lemma~\ref{lemma:gamma-bound-nosplit}, on the event $p_{\max}\leq 2p\pi_{\max}$, 
$\max_{i\in[n]}\gamma_{1,i}(\hat{\AAA})\lesssim (p\pi(r/p)^{3/2}) $. 
Following similar arguments as those for \eqref{eq:verify-condition-2}, we have { with probability tending to 1,}
\begin{equation}\label{eq:verify-condition-2-nosplit}
\begin{split}
     &\min_{i\in[n]}\{(\gamma_{1,i}(\hat{\AAA}))^{-1} (\kappa_3\big(3C_1C_2\big))^{-1}\sigma^2_r({\II}_{1,i}(\hat{\AAA}))\}\\
     \gtrsim & (p\pi)^{-1}(r/p)^{-3/2}(\kappa_3^*)^{-1}\pi^2(\delta_2^*)^2\\
     =& (\kappa_{3}^*)^{-1}(\delta_2^*)^2 p^{1/2}r^{-3/2}\pi.
\end{split}
\end{equation}
\sloppy Under the asymptotic regime \hl{$p\pi\gg (\kappa_2^*)^2(\kappa_3^*)^2 (\delta_2^*)^{-4}r^5\log(n)$}, we have $\kappa_2^*\pi^{1/2} r (\log(n))^{1/2}\ll (\kappa_{3}^*)^{-1}(\delta_2^*)^2 p^{1/2}r^{-3/2}\pi$. Under the asymptotic regime \hl{$e_{A,F}\ll (\kappa_2^*)^{-1}(\kappa_3^*)^{-1}(\delta_2^*)^2 (\log(np))^{-1}r^{(-5/2-\eta_1)\wedge (-3/2)}\pi^{1/2}$}, we have $\kappa_2^* \log(np) r^{(1+\eta_1)\vee 0} p^{1/2} \pi^{1/2}e_{A,F} \ll (\kappa_{3}^*)^{-1}(\delta_2^*)^2 p^{1/2}r^{-3/2}\pi$. Combining the analysis, we have $ \kappa_2^* \{\log^{1/2}(n)r\pi^{1/2} + \log(np) r^{(1+\eta_1)\vee 0} p^{1/2} \pi^{1/2}e_{A,F} \}\ll (\kappa_{3}^*)^{-1}(\delta_2^*)^2 p^{1/2}r^{-3/2}\pi.$ This, together with \eqref{eq:verify-condition-2-nosplit} implies {with probability tending to 1,}
\begin{equation}%
     \max_{i\in[n]}\{\|\ZZ_{i\cdot}\dOmegai\hat{\AAA}\|+\|\BB_{1,i}(\hat{\AAA})\| +\beta_{1,i}(\hat{\AAA})\kappa_3^* \}\ll \min_{i\in[n]}\{(\gamma_{1,i}(\hat{\AAA}))^{-1} (\kappa_3\big(3C_1C_2\big))^{-1}\sigma^2_r({\II}_{1,i}(\hat{\AAA}))\}.
\end{equation}
According to \eqref{eq:denominator-a-nosplit} and note that $C_1\gtrsim r^{1/2+\eta_2}p^{1/2}$,
we have
\begin{equation}\label{eq:verify-condition-3-nosplit}
	\sigma_r(\II_{1,i}(\hat{\AAA}))C_2\gtrsim \delta_2^*\pi r^{1/2+\eta_2}p^{1/2}.
\end{equation}
According to \eqref{eq:1st-ppi-asymp-implications-nosplit}, \hl{$p\pi\gg(\kappa_2^*)^2(\delta_2^*)^{-2}\log(n)r^{1-2\eta_2}$}, which implies $\kappa_2^* \log^{1/2}(n)r\pi^{1/2} \ll \delta_2^*\pi r^{1/2+\eta_2}p^{1/2}$. According to \eqref{eq:1st-ea-asymp-implications-nosplit}, \hl{$e_{\AAA,F}\ll \delta_2^*(\kappa_2^*)^{-1}(\log(np))^{-1}r^{(-1/2-\eta_1+\eta_2)\wedge(1/2+\eta_2)}\pi^{1/2}$}, which implies $\kappa_2^*  \log(np) r^{(1+\eta_1)\vee 0} p^{1/2} \pi^{1/2}e_{A,F} \ll \delta_2^*\pi r^{1/2+\eta_2}p^{1/2}$. Combining the analysis with \eqref{eq:1st-ppi-asymp-implications-nosplit} and \eqref{eq:numerator-a-nosplit}, we obtain {with probability tending to 1,}
\begin{equation}
	 \max_{i\in[n]}\{\|\ZZ_{i\cdot}\dOmegai\hat{\AAA}\|+\|\BB_{1,i}(\hat{\AAA})\| +\beta_{1,i}(\hat{\AAA})\kappa_3^* \}\ll \min_{i\in[n]}\sigma_r(\II_{1,i}(\hat{\AAA}))C_2.
\end{equation}
Thus, conditions of Lemma~\ref{lemma:finite-simplify-kappa} are satisfied. According to Lemma~\ref{lemma:finite-simplify-kappa} with $\AAA$ replaced by $\hat{\AAA}$ and according to \eqref{eq:numerator-a-nosplit} and \eqref{eq:denominator-a-nosplit}, we have $\|\tilde{\TTT}-\TTT^*\|_{2\to\infty}\leq C_1$ and
\begin{equation}
\begin{split}
&\|\tilde{\TTT}-\TTT^*\|_{2\to\infty}\\
    \leq  &\max_{i\in[n]}\Big[ (\sigma_r(\II_{1,i}(\hat{\AAA})))^{-1}  \{\|\ZZ_{i\cdot}\dOmegai\hat{\AAA}\|+\|\BB_{1,i}(\hat{\AAA})\| +\beta_{1,i}(\hat{\AAA})\kappa_3^* \}\Big]\\
    \lesssim &  (\delta_2^*\pi)^{-1}\kappa_2^*\{ r \log^{1/2}(n)\pi^{1/2} +
    \log(np) r^{(1+\eta_1)\vee 0} p^{1/2} \pi^{1/2}e_{A,F} \}\\    
    = & \kappa_2^*(\delta_2^*)^{-1}\pi^{-1/2}
\big\{ r (\log(n))^{1/2}+ \log (np)r^{(1+\eta_1)\vee 0}p^{1/2}e_{\AAA,F} \big\}
\end{split}
\end{equation}
with probability converging to $1$. Moreover, from \eqref{eq:denominator-a-nosplit} the optimization problem  $\max_{\ttt_i\in\mathbb{R}^r}\sum_{j\in[p]}\omega_{ij}\{y_{ij}\ttt_i^T\hat{\aaaa}_j -b(\ttt_i^T\hat{\aaaa}_j)\}$ is strictly convex. Thus,  $\tilde{\ttt}_i$ is the unique solution to this optimization problem.

\end{proof}

\subsection{Additional theoretical result for Method~\ref{meth:nosplit} without data splitting}\label{sec:asymp-no-split-additional}
\begin{lem}\label{lemma:asymptotic-m-max-nosplit}
Let $\tilde{\MM}$ be obtained by Method~\ref{meth:nosplit}. 
	    Assume that $\lim_{n,p\to\infty}\pr(\|\hat{\MM}-\MM^*\|_F\leq e_{\MM,F})=1$, and the following asymptotic regime holds:
    \begin{enumerate}
               \item $\phi\sim 1$, $\pi_{\min}\sim\pi_{\max}\sim \pi$;
    \item $\|\UU_r^*\|_{2\to\infty}\lesssim (r/n)^{1/2}$, $\|\VV_r^*\|_{2\to\infty} \lesssim(r/p)^{1/2}$, $C_2 \sim (r/p)^{1/2}$;
    \item $(np)^{1/2} r^{\eta_2} \lesssim \sigma_r(\MM^*)\leq \sigma_1(\MM^*) \lesssim (np)^{1/2} r^{\eta_1}$ for some constants $\eta_1$ and $\eta_2$;
    \item          $p\pi
        \gg (\kappa_2^*)^{4}(\delta_2^*)^{-6} (\log(n p))^{3} 
        \cdot\max\Big[  r^{(1+2\eta_1)\vee (3+2\eta_1-4\eta_2)\vee (1-4\eta_2) }, (\kappa_3^*)^2 r^{\{7+8(\eta_1-\eta_2)\}\vee (5+6\eta_1-8\eta_2)}  \Big]$;
    \item $n\pi
         \gg  (\kappa_2^*)^2(\delta_2^*)^{-4}   (\log(np))^2 \max \big\{r^{(1+2\eta_1-2\eta_2)\vee(1+2\eta_1-4\eta_2)
         },  (\kappa_3^*)^2 r^{5+8\eta_1-8\eta_2} \big\};
$
     \item 
        \begin{equation}
        \begin{split}
           & (np)^{-1/2}e_{\MM,F}\\
            \ll &  (\kappa_2^*)^{-2}(\delta_2^*)^{3} (\log(n p))^{-2}\pi^{1/2}\cdot \\
            &\min\big[   r^{ (1/2+2\eta_2)\wedge(-3/2-2\eta_1+3\eta_2)\wedge (-1/2-\eta_1+3\eta_2)\wedge(1/2+3\eta_2)},
            (\kappa_3^*)^{-1}  r^{ (-7/2-5\eta_1+5\eta_2)\wedge (-5/2-4\eta_1+5\eta_2)\wedge(-3/2-3\eta_1+5\eta_2) } \big].
        \end{split}
    \end{equation}

    \end{enumerate}
    Then, with probability converging to $1$, estimating equations in steps 3 and 4 of Method~\ref{meth:nosplit} have a unique solution and
    \begin{equation}
    \begin{split}
   & \|\tilde{\MM}-\MM^*\|_{\max}\\
    \lesssim & (\delta_2^*)^{-2}(\kappa_2^*)^2 (\log(np))^2\Big[r^{(5/2+2\eta_1-2\eta_2)\vee (3/2+\eta_1-2\eta_2)}\{(n\wedge p)\pi\}^{-1/2} + r^{(5/2+3\eta_1-3\eta_2)\vee (3/2+2\eta_1-3\eta_2)\vee(1/2+\eta_1-3\eta_2)} (np\pi)^{-1/2}e_{\MM,F}\Big].
       \end{split}
    \end{equation}

\end{lem}

\begin{proof}
First, we analyze the asymptotic regime assumption. The 4-th condition of the asymptotic regime, i.e., 
 \begin{equation}
    \begin{split}
         p\pi
        \gg (\kappa_2^*)^{4}(\delta_2^*)^{-6} (\log(n p))^{3} 
        \cdot\max\Big[  r^{(1+2\eta_1)\vee (3+2\eta_1-4\eta_2)\vee (1-4\eta_2) }, (\kappa_3^*)^2 r^{\{7+8(\eta_1-\eta_2)\}\vee (5+6\eta_1-8\eta_2)}  \Big]
    \end{split}
    \end{equation}
implies
\begin{equation}\label{eq:ppi-m-nosplit}
    p\pi\gg
    \begin{cases}
   (\delta_2^*)^{-4}(\kappa_2^*)^2\log^2(n)\max\big\{ r^{1\vee (1-2\eta_2)}
,(\kappa_3^*)^2 r^{5} \big\},\\
(\delta_2^*)^{-6}(\kappa_2^*)^4(\log(np))^3 r^{(3+2\eta_1-4\eta_2)\vee (1-4\eta_2)},\\
(\delta_2^*)^{-6}(\kappa_2^*)^4(\kappa_3^*)^2(\log(np))^3 r^{\{7+8(\eta_1-\eta_2)\}\vee (5+6\eta_1-8\eta_2)},\\
    \end{cases}
\end{equation}
where we used the fact that $7+8(\eta_1-\eta_2)\geq 7>5$, $(1+2\eta_1)\vee(2-2\eta_2)\geq 1$, and $2-2\eta_2<3+2\eta_1-4\eta_2$.

	The 6-th condition of the asymptotic regime, i.e.,
   \begin{equation}
        \begin{split}
           & (np)^{-1/2}e_{\MM,F}\\
            \ll &  (\kappa_2^*)^{-2}(\delta_2^*)^{3} (\log(n p))^{-2}\pi^{1/2}\cdot \\
            &\min\big[   r^{ (1/2+2\eta_2)\wedge(-3/2-2\eta_1+3\eta_2)\wedge (-1/2-\eta_1+3\eta_2)\wedge(1/2+3\eta_2)},
            (\kappa_3^*)^{-1}  r^{ (-7/2-5\eta_1+5\eta_2)\wedge (-5/2-4\eta_1+5\eta_2)\wedge(-3/2-3\eta_1+5\eta_2) } \big]
        \end{split}
    \end{equation}
implies
\begin{equation}\label{eq:etheta-m-nosplit}
    \begin{split}
        (np)^{-1/2}e_{\MM,F}
        \ll 
        \begin{cases}
        r^{\eta_2}(\kappa_2^*)^{-1}(\delta_2^*)^2 (\log(np))^{-1}
r^{0\wedge(-1/2-\eta_1+\eta_2)\wedge(1/2+\eta_2)}\pi^{1/2},\\
        r^{\eta_2}(\kappa_2^*)^{-1}(\delta_2^*)^2 (\log(np))^{-1}(\kappa_3^*)^{-1}r^{ (-5/2-\eta_1)\wedge (-3/2)}\pi^{1/2},\\
        (\delta_2^*)^3(\kappa_2^*)^{-2}(\log(np))^{-2}r^{(-3/2-2\eta_1+3\eta_2)\wedge (-1/2-\eta_1+3\eta_2)\wedge(1/2+3\eta_2)}\pi^{1/2},\\
         (\delta_2^*)^{3}(\kappa_2^*)^{-2}(\kappa_3^*)^{-1}(\log(np))^{-2} r^{(-7/2-5\eta_1+5\eta_2)\wedge (-5/2-4\eta_1+5\eta_2)\wedge(-3/2-3\eta_1+5\eta_2)}\pi^{1/2},
        \end{cases}
    \end{split}
\end{equation}
where we used the fact that $\eta_2\geq -1/2 -\eta_1+2\eta_2$, $-1/2-\eta_1+2\eta_2>-3/2-2\eta_1+3\eta_2$, $-5/2-\eta_1+\eta_2>-7/2-5\eta_1+5\eta_2$, and $-3/2+\eta_2>-5/2-4\eta_1+5\eta_2$.

\sloppy According to \eqref{eq:etheta-m-nosplit},  \hl{$ e_{\MM,F}\ll (np)^{1/2}r^{\eta_2}(\kappa_2^*)^{-1}(\delta_2^*)^2 (\log(np))^{-1}
\min\{r^{0\wedge(-1/2-\eta_1+\eta_2)\wedge(1/2+\eta_2)},(\kappa_3^*)^{-1}r^{ (-5/2-\eta_1)\wedge (-3/2)} \}\pi^{1/2}$}, which implies  $e_{\AAA,F}\ll (\kappa_2^*)^{-1}(\delta_2^*)^2 (\log(np))^{-1}
\min\{r^{0\wedge(-1/2-\eta_1+\eta_2)\wedge(1/2+\eta_2)},(\kappa_3^*)^{-1}r^{ (-5/2-\eta_1)\wedge (-3/2)} \}\pi^{1/2}$ by Lemma~\ref{lemma:em-to-ea-nosplit}. Also, according to the lemma's assumption, \hl{$p\pi\gg (\delta_2^*)^{-4}(\kappa_2^*)^2\log^2(n)\max\big\{ r^{1\vee (1-2\eta_2)}
,(\kappa_3^*)^2 r^{5} \big\}$.}
Thus, the conditions of Lemma~\ref{lemma:asymptotic-1st-theta-nosplit} are satisfied. According to Lemma~\ref{lemma:asymptotic-1st-theta-nosplit}, 
$\|\hat{\TTT}-\TTT^*\|_{2\to\infty}\leq e_{\TTT,2\to\infty}$, with probability converging to $1$,
for $e_{\TTT,2\to\infty}$ satisfying
\begin{equation}\label{eq:rate-theta-2toinfty-nosplit}
\begin{split}
    &e_{\TTT,2\to\infty}\\ 
  \sim & \kappa_2^*(\delta_2^*)^{-1}\pi^{-1/2}
\big\{ r (\log(n))^{1/2}+ \log (np)r^{(1+\eta_1)\vee 0}p^{1/2}e_{\AAA,F} \big\}\\
\lesssim & \kappa_2^*(\delta_2^*)^{-1}\pi^{-1/2}
\big\{ r (\log(n))^{1/2}+ \log (np)r^{(1+\eta_1)\vee 0}p^{1/2}\cdot r^{-\eta_2}(np)^{-1/2}e_{\MM,F}\big\}\\
\sim & \kappa_2^*(\delta_2^*)^{-1}\pi^{-1/2}
\big\{ r (\log(n))^{1/2}+ \log (np)r^{(1+\eta_1-\eta_2)\vee (-\eta_2)}n^{-1/2}e_{\MM,F}\big\}.
\end{split}    
\end{equation}
Note that the proof of  Lemma~\ref{lemma:2nd-a-asymptotics} does not require the independence between $\tilde{\TTT}_{\gN_2}$ and the missing pattern $\OOO$. Thus, following similar arguments, Lemma~\ref{lemma:2nd-a-asymptotics} still applies with $\tilde{\TTT}_{\gN_2}$
 replaced with $\tilde{\TTT}$ and $\gN_2$ replaced with $[n]$. Next, we verify that the asymptotic regime of Lemma~\ref{lemma:2nd-a-asymptotics} is satisfied.

According to \eqref{eq:ppi-m-nosplit},
\hl{$p\pi\gg (\delta_2^*)^{-6}(\kappa_2^*)^4(\log(np))^3 r^{(3+2\eta_1-4\eta_2)\vee (1-4\eta_2)}$}, which implies 
\begin{equation}\label{eq:temp1-nosplit}
	\kappa_2^*(\delta_2^*)^{-1}\pi^{-1/2}
r (\log(n))^{1/2}
          \ll  (\delta_2^*)^2 (\kappa_2^*)^{-1}  p^{1/2}(\log(n p))^{-1} r^{(-1/2-\eta_1+2\eta_2)\wedge(1/2+2\eta_2)}.
\end{equation}
According to \eqref{eq:ppi-m-nosplit}, \hl{$p\pi\gg (\delta_2^*)^{-6}(\kappa_2^*)^4(\kappa_3^*)^2(\log(np))^3 r^{\{7+8(\eta_1-\eta_2)\}\vee (5+6\eta_1-8\eta_2)}$}, which implies
\begin{equation}\label{eq:temp2-nosplit}
\kappa_2^*(\delta_2^*)^{-1}\pi^{-1/2}
r (\log(n))^{1/2}
          \ll (\delta_2^*)^2 (\kappa_2^*)^{-1}  p^{1/2}(\log(n p))^{-1}(\kappa_3^*)^{-1}  r^{(-5/2-4\eta_1+4\eta_2)\wedge (-3/2-3\eta_1+4\eta_2)}.
\end{equation}
According to \eqref{eq:etheta-m-nosplit}, \hl{$(np)^{-1/2}e_{\MM,F}\ll (\delta_2^*)^3(\kappa_2^*)^{-2}(\log(np))^{-2}r^{(-3/2-2\eta_1+3\eta_2)\wedge (-1/2-\eta_1+3\eta_2)\wedge(1/2+3\eta_2)}\pi^{1/2}$}, which implies
\begin{equation}\label{eq:temp3-nosplit}
    \begin{split}
    \kappa_2^*(\delta_2^*)^{-1}\pi^{-1/2}
\log (np)r^{(1+\eta_1-\eta_2)\vee (-\eta_2)}n^{-1/2}e_{\MM,F}        \ll 
(\delta_2^*)^2 (\kappa_2^*)^{-1}  p^{1/2}(\log(n p))^{-1} r^{(-1/2-\eta_1+2\eta_2)\wedge(1/2+2\eta_2)}.
    \end{split}
\end{equation}
According to \eqref{eq:etheta-m-nosplit}, \hl{$(np)^{-1/2}e_{\MM,F}\ll (\delta_2^*)^{3}(\kappa_2^*)^{-2}(\kappa_3^*)^{-1}(\log(np))^{-2} r^{(-7/2-5\eta_1+5\eta_2)\wedge (-5/2-4\eta_1+5\eta_2)\wedge(-3/2-3\eta_1+5\eta_2)}\pi^{1/2}$}, which implies
\begin{equation}\label{eq:temp4-nosplit}
    \begin{split}
              \kappa_2^*(\delta_2^*)^{-1}\pi^{-1/2}
\log (np)r^{(1+\eta_1-\eta_2)\vee (-\eta_2)}n^{-1/2}e_{\MM,F}         \ll  (\delta_2^*)^2 (\kappa_2^*)^{-1}  p^{1/2}(\log(n p))^{-1}(\kappa_3^*)^{-1}  r^{(-5/2-4\eta_1+4\eta_2)\wedge (-3/2-3\eta_1+4\eta_2)}.
    \end{split}
\end{equation}
Combining equations \eqref{eq:rate-theta-2toinfty-nosplit}, \eqref{eq:temp1-nosplit}, \eqref{eq:temp2-nosplit}, \eqref{eq:temp3-nosplit} and \eqref{eq:temp4-nosplit}, we have
 \begin{equation}
\begin{split}
&e_{\TTT,2\to\infty}\\
\ll 
    & \kappa_2^*(\delta_2^*)^{-1}\pi^{-1/2}
(\delta_2^*)^2 (\kappa_2^*)^{-1}  p^{1/2}(\log(np))^{-1}\cdot\min\{ r^{(-1/2-\eta_1+2\eta_2)\wedge(1/2+2\eta_2)},(\kappa_3^*)^{-1}  r^{(-5/2-4\eta_1+4\eta_2)\wedge (-3/2-3\eta_1+4\eta_2)}
          \},
\end{split}
 \end{equation}
 which implies $e_{\TTT,2\to\infty}$ satisfies %
 the 5-th condition of the asymptotic regime of Lemma~\ref{lemma:2nd-a-asymptotics}.

On the other hand, according to the lemma's assumption,  
         {   \begin{equation}
    \begin{split}
         &n\pi_{\min}\\
         \gg & (\kappa_2^*)^2(\delta_2^*)^{-4}   (\log(np))^2 \max \big\{(\pi_{\max}/\pi_{\min})r^{(1+2\eta_1-2\eta_2)\vee(1+2\eta_1-4\eta_2)
         },  (\kappa_3^*)^2(\pi_{\max}/\pi_{\min})^3 r^{5+8\eta_1-8\eta_2} \big\}.
    \end{split}
    \end{equation}
}
Thus, the other requirements for the asymptotic regime in Lemma~\ref{lemma:2nd-a-asymptotics} are also satisfied.         

According to Lemma~\ref{lemma:2nd-a-asymptotics}, we have $\|\tilde{\AAA}-\AAA^*\|_{2\to\infty}\leq e_{\AAA,2\to\infty}$ with probability converging to $1$,
 where
 \begin{equation}\label{eq:rate-a-2toinfty-nosplit}
     \begin{split}
       e_{\AAA,2\to\infty} \sim   \kappa_2^*(\delta_2^*)^{-1}r^{-2\eta_2}\log(np)p^{-1/2}\Big\{r^{1+\eta_1}(n\pi)^{-1/2}  +  r^{(1+\eta_1)\vee 0}  p^{-1/2}e_{\TTT,2\to\infty}\Big\}.
\end{split}
\end{equation}
Combining the above display with \eqref{eq:rate-theta-2toinfty-nosplit}, we further have
\begin{equation}\label{eq:rate-a-2toinfty-nosplit-1}
    \begin{split}
    & e_{\AAA,2\to\infty}\\
    \lesssim & \kappa_2^*(\delta_2^*)^{-1}r^{-2\eta_2}\log(np)p^{-1/2}\Big[r^{1+\eta_1}(n\pi)^{-1/2}  \\
    &+  r^{(1+\eta_1)\vee 0}  p^{-1/2}\cdot \kappa_2^*(\delta_2^*)^{-1}\pi^{-1/2}
\big\{ r (\log(n))^{1/2}+ \log (np)r^{(1+\eta_1-\eta_2)\vee (-\eta_2)}n^{-1/2}e_{\MM,F}\big\}
  \Big] \\
  \lesssim &(\delta_2^*)^{-2} (\kappa_2^*)^2(\log(np))^2p^{-1/2}\Big[r^{(2+\eta_1-2\eta_2)\vee(1-2\eta_2)}\{(n\wedge p)\pi\}^{-1/2} + r^{(2+2\eta_1-3\eta_2)\vee(1+\eta_1-3\eta_2)\vee(-3\eta_2)} (np\pi)^{-1/2}e_{\MM,F} \Big].
     \end{split}
 \end{equation}

Next, we derive an asymptotic upper bound for $\|\tilde{\MM}-\MM^*\|_{\max}$. Recall that $\tilde{\MM}=\tilde{\TTT}\tilde{\AAA}^T$. Thus, for $\hat{\PP}\in\mathcal{O}_{r\times r}$ defined in \eqref{eq:hat-p-nosplit}  and $\TTT^*= (\UU_r^*)\DD_r^*\hat{\PP}$, $\AAA^* = \VV_r^*\hat{\PP}$, we have $       \tilde{\MM}-\MM^*
         =  \tilde{\TTT}\tilde{\AAA}^T - \TTT^*(\AAA^*)^T
         =  (\tilde{\TTT}-\TTT^*)(\AAA^*)^T + \tilde{\TTT}(\tilde{\AAA}-\AAA^*)^T.
$
Thus,
\begin{equation}
	\|\tilde{\MM}-\MM^*\|_{\max}\leq \|\tilde{\TTT}-\TTT^*\|_{2\to\infty}\|{\AAA}^*\|_{2\to\infty} + \|\tilde{\AAA}-\AAA^*\|_{2\to\infty}\|\tilde{\TTT}\|_{2\to\infty}. 
\end{equation}
According to Lemma~\ref{lemma:asymptotic-1st-theta-nosplit} and the assumption $\|\AAA^*\|_{2\to\infty}\leq C_2\lesssim (r/p)^{1/2}$, with probability converging to $1$, the above display is further bounded by
\begin{equation}
\begin{split}
           \|\tilde{\MM}-\MM^*\|_{\max}
         \lesssim e_{\TTT,2\to\infty}r^{1/2}p^{-1/2} + e_{\AAA,2\to\infty}r^{1/2+\eta_1}p^{1/2}.
\end{split}
\end{equation}
Combining the above inequality with 
\eqref{eq:rate-theta-2toinfty-nosplit}
 and \eqref{eq:rate-a-2toinfty-nosplit-1}, we obtain {with probability tending to 1}
 \begin{equation}
     \begin{split}      &\|\tilde{\MM}-\MM^*\|_{\max}\\
     \lesssim & r^{1/2}p^{-1/2}\cdot \kappa_2^*(\delta_2^*)^{-1}\pi^{-1/2}
\big\{ r (\log(n))^{1/2}+ \log (np)r^{(1+\eta_1-\eta_2)\vee (-\eta_2)}n^{-1/2}e_{\MM,F}\big\}
\\
&+ r^{1/2+\eta_1}p^{1/2}\cdot (\delta_2^*)^{-2} (\kappa_2^*)^2(\log(np))^2p^{-1/2}\Big[r^{(2+\eta_1-2\eta_2)\vee(1-2\eta_2)}\{(n\wedge p)\pi\}^{-1/2} \\
&~~~~~~~~~+ r^{(2+2\eta_1-3\eta_2)\vee(1+\eta_1-3\eta_2)\vee(-3\eta_2)} (np\pi)^{-1/2}e_{\MM,F} \Big]\\
\lesssim & (\delta_2^*)^{-2}(\kappa_2^*)^2 (\log(np))^2\Big[r^{(5/2+2\eta_1-2\eta_2)\vee (3/2+\eta_1-2\eta_2)}\{(n\wedge p)\pi\}^{-1/2}\\
& + r^{(3/2+\eta_1-\eta_2)\vee(1/2-\eta_2)\vee (5/2+3\eta_1-3\eta_2)\vee (3/2+2\eta_1-3\eta_2)\vee(1/2+\eta_1-3\eta_2)} (np\pi)^{-1/2}e_{\MM,F}\Big]\\
          \lesssim & (\delta_2^*)^{-2}(\kappa_2^*)^2 (\log(np))^2\Big[r^{(5/2+2\eta_1-2\eta_2)\vee (3/2+\eta_1-2\eta_2)}\{(n\wedge p)\pi\}^{-1/2} + r^{(5/2+3\eta_1-3\eta_2)\vee (3/2+2\eta_1-3\eta_2)\vee(1/2+\eta_1-3\eta_2)} (np\pi)^{-1/2}e_{\MM,F}\Big],
          \end{split}
 \end{equation}
 where we used the fact that $3/2+\eta_1-\eta_2<5/2+3\eta_1-3\eta_2$ and $1/2-\eta_2< 3/2+2\eta_1-3\eta_2$ in the last inequality.
This completes the proof.
\end{proof}
\subsection{Proof of Theorem~\ref{thm:m-bound-no-splitting}}\label{sec:proof-asymptotic-simple-nosplit}

\begin{proof}[Proof of Theorem~\ref{thm:m-bound-no-splitting}]

Note that when $\pi_{\min}\sim\pi_{\max}\sim\pi$ and $\eta_1=\eta_2=\eta$, the 4-th asymptotic requirement in Lemma~\ref{lemma:asymptotic-m-max-nosplit} becomes
 \begin{equation}
    \begin{split}
         p\pi
        \gg & (\kappa_2^*)^{4}(\delta_2^*)^{-6} (\log(np))^{3} 
        \cdot\max\Big[ r^{(1+2\eta) \vee (3-2\eta)\vee(1-4\eta)}, (\kappa_3^*)^2r^{7\vee(5-2\eta)}  \Big].
    \end{split}
    \end{equation}
When $\eta\geq -1$, the above requirement is implied by 
 \begin{equation}
    \begin{split}
         p\pi
        \gg & (\kappa_2^*)^{4}(\delta_2^*)^{-6} (\log(np))^{3} 
        \cdot\max\Big[ r^{(1+2\eta) \vee 5}, (\kappa_3^*)^2r^{7}  \Big],
    \end{split}
    \end{equation}
which is implied by the asymptotic requirement {\bf R5}.

Similarly, the 5-th asymptotic requirement in Lemma~\ref{lemma:asymptotic-m-max-nosplit} becomes
$
         n\pi
         \gg  (\kappa_2^*)^2(\delta_2^*)^{-4}   (\log(np))^2 \max \big\{r^{1\vee(1-2\eta)},  (\kappa_3^*)^2r^{5} \big\}.
$
which is implied by the asymptotic requirement {\bf R6}: $
         n\pi
         \gg  (\kappa_2^*)^2(\delta_2^*)^{-4}   (\log(np))^2 \max \big\{r^{3},  (\kappa_3^*)^2r^{5} \big\}
$.

The 6-th asymptotic requirement {in Lemma~\ref{lemma:asymptotic-m-max-nosplit}} becomes
   \begin{equation}
        \begin{split}
           (np)^{-1/2}e_{\MM,F}
            \ll (\kappa_2^*)^{-2}(\delta_2^*)^{3} (\log(n p))^{-2}\pi^{1/2}
            \min\big[   r^{ (1/2+2\eta)\wedge(-3/2+\eta)\wedge (-1/2+2\eta)\wedge(1/2+3\eta)},
            (\kappa_3^*)^{-1}  r^{ (-7/2)\wedge (-5/2+\eta)\wedge(-3/2+2\eta) } \big]
        \end{split}
    \end{equation}
and is implied by {\bf R7}: $          (np)^{-1/2}e_{\MM,F}
            \ll (\kappa_2^*)^{-2}(\delta_2^*)^{3} (\log(n p))^{-2}\pi^{1/2}
            \min\big[   r^{ -5/2},
            (\kappa_3^*)^{-1}  r^{ -7/2 } \big]$ for $\eta\geq-1$.

Thus, under {\bf R1-R7}, the conditions of Lemma~\ref{lemma:asymptotic-m-max-nosplit} are satisfied, and thus with probability converging to $1$,
    \begin{equation}\label{eq:error-proof-nosplit}
    \begin{split}
    & \|\tilde{\MM}-\MM^*\|_{\max}\\
    \lesssim & (\delta_2^*)^{-2}(\kappa_2^*)^2 (\log(np))^2\\
    &\cdot\Big[r^{(5/2+2\eta_1-2\eta_2)\vee (3/2+\eta_1-2\eta_2)}\{(n\wedge p)\pi\}^{-1/2} + r^{(5/2+3\eta_1-3\eta_2)\vee (3/2+2\eta_1-3\eta_2)\vee(1/2+\eta_1-3\eta_2)} (np\pi)^{-1/2}e_{\MM,F}\Big]\\
    \lesssim & (\delta_2^*)^{-2}(\kappa_2^*)^2 (\log(np))^2 \Big[r^{5/2\vee (3/2-\eta)}\{(n\wedge p)\pi\}^{-1/2} + r^{(5/2)\vee (3/2-\eta)\vee(1/2-2\eta_2)} (np\pi)^{-1/2}e_{\MM,F}\Big] \\
      \lesssim & (\delta_2^*)^{-2}(\kappa_2^*)^2 (\log(np))^2 \Big[r^{5/2}\{(n\wedge p)\pi\}^{-1/2} + r^{5/2} (np\pi)^{-1/2}e_{\MM,F}\Big].
    \end{split}
    \end{equation}

The above analysis gives the error bound of $\tilde{\MM}$.
The proof for the `in particular' part of the theorem is similar to that of the proof of Theorem~\ref{thm:m-bound-same-eta}, and we skip the repetitive details.
\end{proof}
\section{Proof of Corollaries}

\begin{proof}[Proof of Corollary~\ref{coro:binomial-nosplit}]
For binomial model $b''(x)= k e^x (1+e^{x})^{-2}$ and $b^{(3)}(x)=k e^x(1+e^x)^{-2}\{1-2 (1+e^{-x})^{-1}\}$. Thus, $\kappa_2(\alpha) \leq k$, $\kappa_3(\alpha)\leq k$, and $\delta_2(\alpha)\geq k e^{\alpha}(1+e^{\alpha})^{-2}\gtrsim k e^{-\alpha}$. This implies that $\kappa_2^*,\kappa_3^*\lesssim 1$ under the asymptotic regime that $k\sim 1$ ({\bf R8B}). Also, $\delta_2^*\gtrsim k e^{-2(\rho+1)}\gtrsim e^{-2\rho}\gtrsim k e^{- 2\log(n\wedge p)^{1-\epsilon_0}}\gg (n\wedge p)^{-\epsilon_1}$ for any constant $\epsilon_1>0$, where the third inequality is due to {\bf R9B}. Combining the analysis above, we have $(\kappa_2^*)^4(\delta_2^*)^{-6}\log(np)^3\ll (n\vee p)^{6\epsilon_1}\log(np)^3 \ll (n\vee p)^{7 \epsilon_1}$. Similarly, $(\kappa_2^*)^2(\delta_2^*)^{-4}(\log(np))^2\ll (n\vee p)^{5\epsilon_1}$, and $(\kappa_2^*)^{-2}(\delta_2^*)^3(\log(np))^{-2}\gg (n\wedge p)^{-4\epsilon_1}$. 

Combine the above analysis with {\bf R5B} -- {\bf R7B}, 
and note that $(1+2\eta)\vee 5\leq (3+4\eta)\vee 7$ for $\eta\geq -1$, we verify that {\bf R5} -- {\bf R7} hold with $7\epsilon_1<\epsilon_0$.
\end{proof}

\begin{proof}[Proof of Corollary~\ref{coro:linear-nosplit}]
For normal model, $b''(x)=1$ and $b^{(3)}(x)=0$ for all $x$. Thus, $\kappa_2^*=\delta_2^*=1$ and $\kappa_3^*=0$. Corollary~\ref{coro:linear-nosplit} then follows by simplifying Theorem~\ref{thm:m-bound-no-splitting}. 

\end{proof}

\begin{proof}[Proof of Corollary~\ref{coro:Poisson-nosplit}]
First, note that $\|\MM^*\|\leq C_1C_2$ so we could choose $\rho\leq C_1C_2\lesssim r^{1+\eta}$. Under {\bf R10P}, $r^{1+\eta}\lesssim (\log(n\wedge p))^{1-\epsilon_0}$, so $\max(\rho,C_1C_2)\lesssim (\log(n\wedge p))^{1-\epsilon_0}$.

For Poisson model, $b(x)=e^x$ so $b''(x)=b^{(3)}(x)=e^x$. Thus, $\kappa_2(\alpha),\kappa_3(\alpha)\leq e^{\alpha}$ and $\delta_2(\alpha)\geq e^{-\alpha}$. This implies $\kappa_2^*\leq e^{2\rho+1}\lesssim e^{2\rho}\lesssim e^{2(\log(n\wedge p))^{1-\epsilon_0}}\lesssim (n\wedge p)^{\epsilon_1}$ for any constant $\epsilon_1>0$. Similarly, $\delta_2^*\gtrsim e^{-2\rho}\gtrsim (n\wedge p)^{-\epsilon_1}$ and $\kappa_3^*\lesssim e^{6 C_1C_2}\lesssim (n\vee p)^{\epsilon_1}$ for any constant $\epsilon_1>0$. The proof then follows similarly as that for Corollary~\ref{coro:linear-nosplit}.

\end{proof}
\begin{proof}[Proof of Corollary~\ref{coro:binomial} -- \ref{coro:Poisson}]
    The proof of Corollary~\ref{coro:binomial} -- \ref{coro:Poisson} is similar to that of Corollary~\ref{coro:binomial-nosplit} -- \ref{coro:Poisson-nosplit}, except that {\bf R7B} is replaced by {\bf R7'B} to ensure {\bf R7'} holds. We omit the repetitive details.
\end{proof}

\section{Simulation Settings and Additional Results}
 
\subsection{Simulation Setting Details}

A full list of our simulation settings is given in Table~\ref{tab:settings_full} below. For each setting, data are generated as follows. 
For each replication, we first generate $\TTT^* = (\theta_{ik}^*)_{n\times r}$ and $\AAA^* = (a_{ij}^*)_{p\times r}$, where $\theta_{ik}^*$s and $a_{ij}^*$s are independently from a uniform distribution over the interval $[-0.9, 0.9]$. Then $\MM^*$ is given by $\MM^* = \TTT^* (\AAA^*)^T$. The missing indicators $\omega_{ij}$s are generated independently from a Bernoulli distribution with parameter $\pi$, where $\pi = 0.6$ and $0.2$ are considered in the simulation settings. When $\omega_{ij}=1$ and for an ordinal variable $j$, $Y_{ij}$ is generated from a Binomial distribution with $k_j = 5$ trials and success probability $\exp(m_{ij}^*)/(1+\exp(m_{ij}^*))$. When $\omega_{ij}=1$ and for an continuous variable $j$, $Y_{ij}$ is generated from a normal distribution $N(m_{ij}^*,1)$. In the implementation, we set $C_2 = 2\sqrt{r/p}$ in Methods 1, 2, and 2'. We set $\rho' = r$ in the NBE and $C = \sqrt{r}$ in the CJMLE.

\begin{table}
    \centering
    \begin{tabular}{c|ccccc|c|ccccc}
    \hline
      Setting   & $n$ & $p$ & $r$ & $\pi$ &Variable Types&Setting   & $n$ & $p$ & $r$ & $\pi$&Variable Type    \\
      \hline
        1& 400 & 200 & 3 & 0.6&O &13& 400 & 200 & 5 & 0.6&O \\
        2& 800 & 400 & 3 & 0.6&O &14& 800 & 400 & 5 & 0.6&O \\
        3& 1600 & 800 & 3 & 0.6&O &15& 1600 & 800 & 5 & 0.6&O \\
        4& 400 & 200 & 3& 0.2&O &16& 400 & 200 & 5 & 0.2&O\\        
        5& 800 & 400 & 3& 0.2&O &17& 800 & 400 & 5 & 0.2&O \\ 
        6& 1600 & 800 & 3& 0.2&O &18& 1600 & 800 & 5 & 0.2&O \\ 
        7& 400 & 200 & 3 & 0.6&O + C &19& 400 & 200 & 5 & 0.6&O + C \\
        8& 800 & 400 & 3 & 0.6&O + C&20& 800 & 400 & 5 & 0.6&O + C \\
        9& 1600 & 800 & 3 & 0.6&O + C&21& 1600 & 800 & 5 & 0.6&O + C \\
        10& 400 & 200 & 3& 0.2&O + C&22& 400 & 200 & 5 & 0.2&O + C \\
        11& 800 & 400 & 3& 0.2&O + C&23& 800 & 400 & 5 & 0.2&O + C \\ 
        12& 1600 & 800 & 3& 0.2&O + C&24& 1600 & 800 & 5 & 0.2&O + C \\ 
      \hline
    \end{tabular}
    \caption{Simulation settings. `Variable type = O' indicates all the variables are ordinal (with $k_j = 5$), and `Variable type = O + C' indicates half of the variables are ordinal (with $k_j = 5$) and half are continuous. For continuous and ordinal variables, we assume the Normal and Binomial models, respectively. }
    \label{tab:settings_full}
\end{table}

\subsection{Additional Simulation Results}

In Figures \ref{fig:sim7-9} though \ref{fig:sim22-24} below, we give the results under Settings 7 through 24. The patterns are similar to those in Figures \ref{fig:sim1-3} and \ref{fig:sim4-6}, except for few cases when $n$ and $p$ are relatively small. 

\begin{figure}
    \centering
    \includegraphics[scale=0.4]{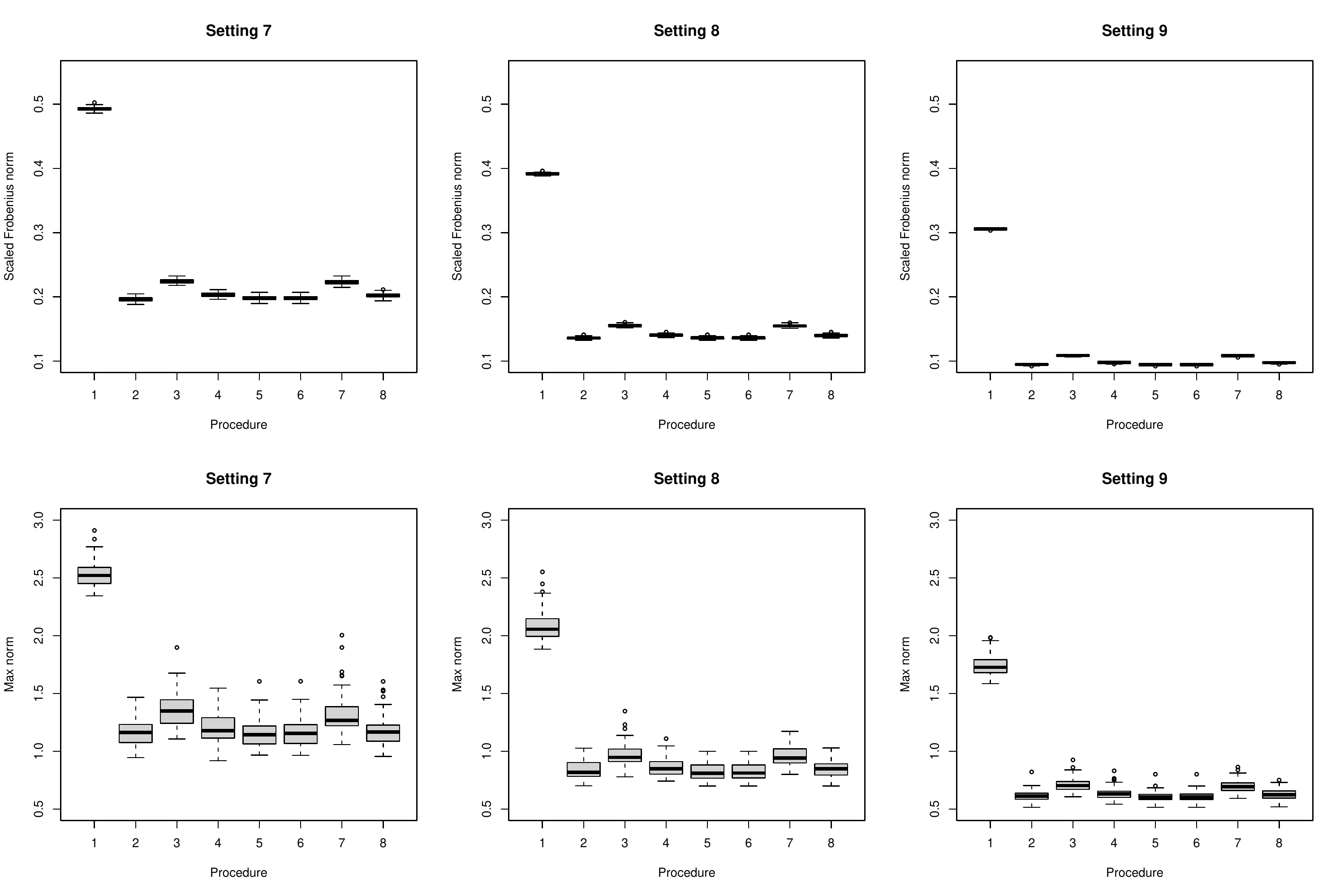}
    \caption{Results from Simulation Settings 7-9. The plots can be interpreted similarly as those in Figure~\ref{fig:sim1-3}.}
    \label{fig:sim7-9}
\end{figure}

\begin{figure}
    \centering
    \includegraphics[scale=0.4]{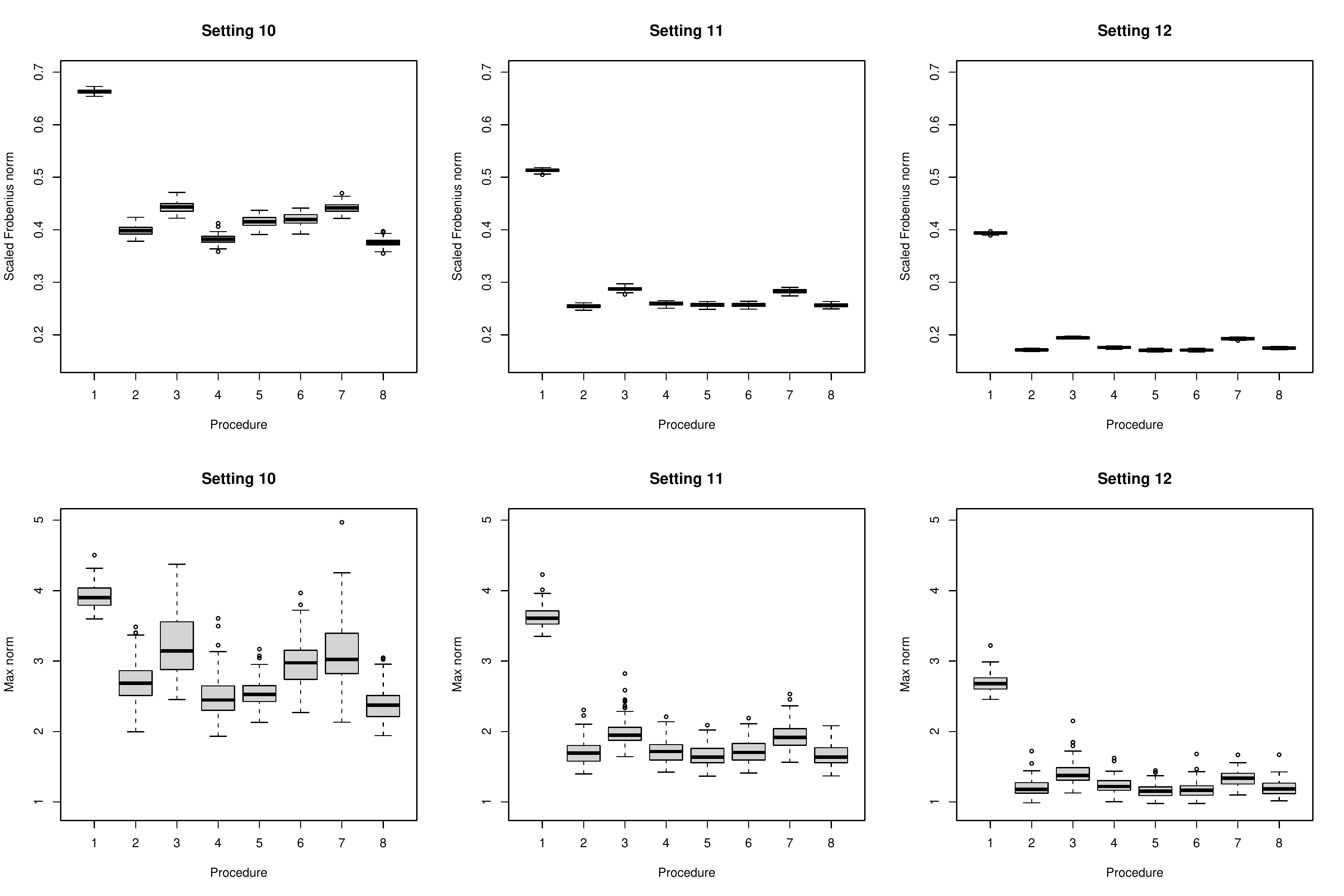}
    \caption{Results from Simulation Settings 10-12. The plots can be interpreted similarly as those in Figure~\ref{fig:sim1-3}.}
    \label{fig:sim10-12}
\end{figure}

\begin{figure}
    \centering
    \includegraphics[scale=0.4]{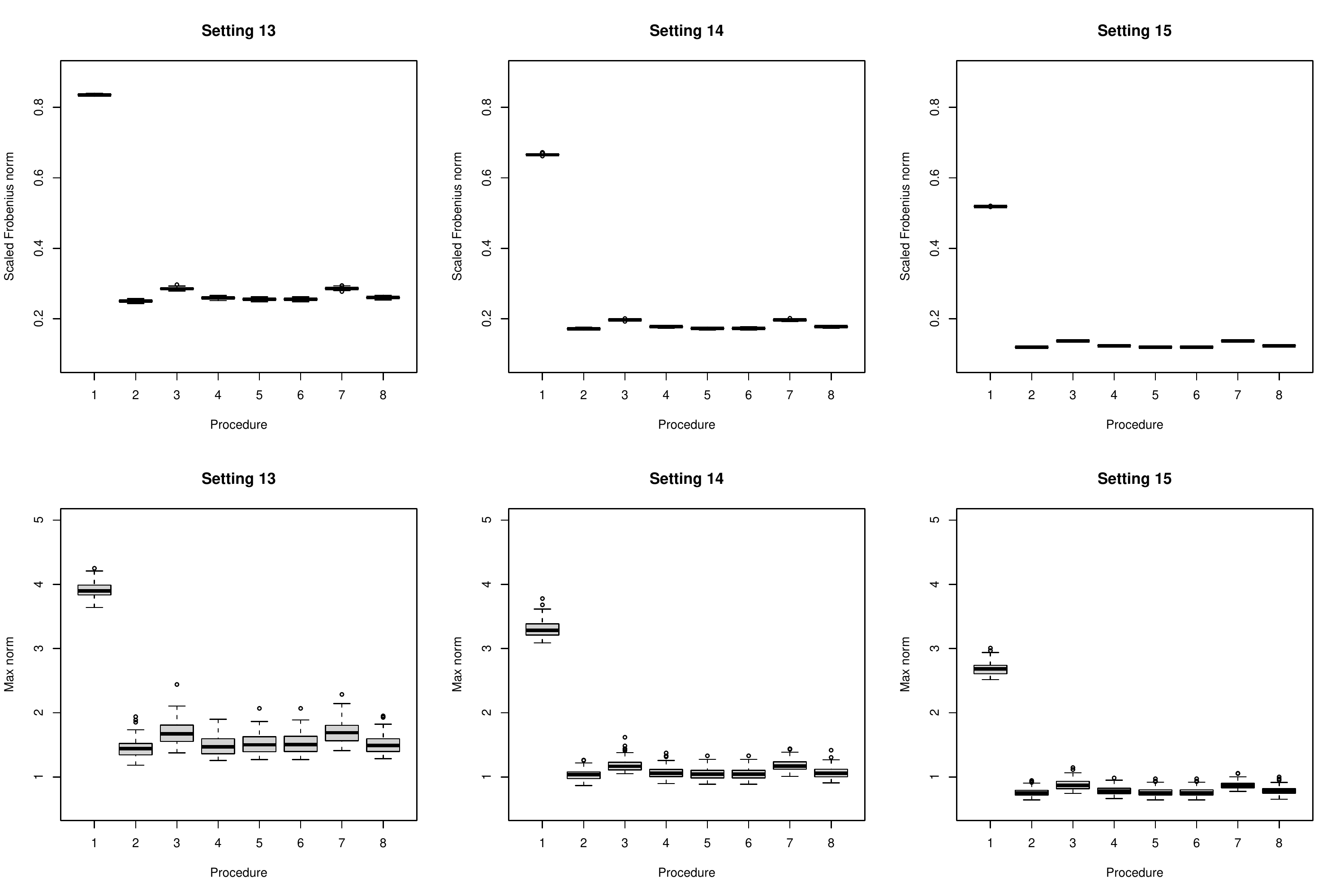}
    \caption{Results from Simulation Settings 13-15. The plots can be interpreted similarly as those in Figure~\ref{fig:sim1-3}.}
    \label{fig:sim13-15}
\end{figure}

\begin{figure}
    \centering
    \includegraphics[scale=0.4]{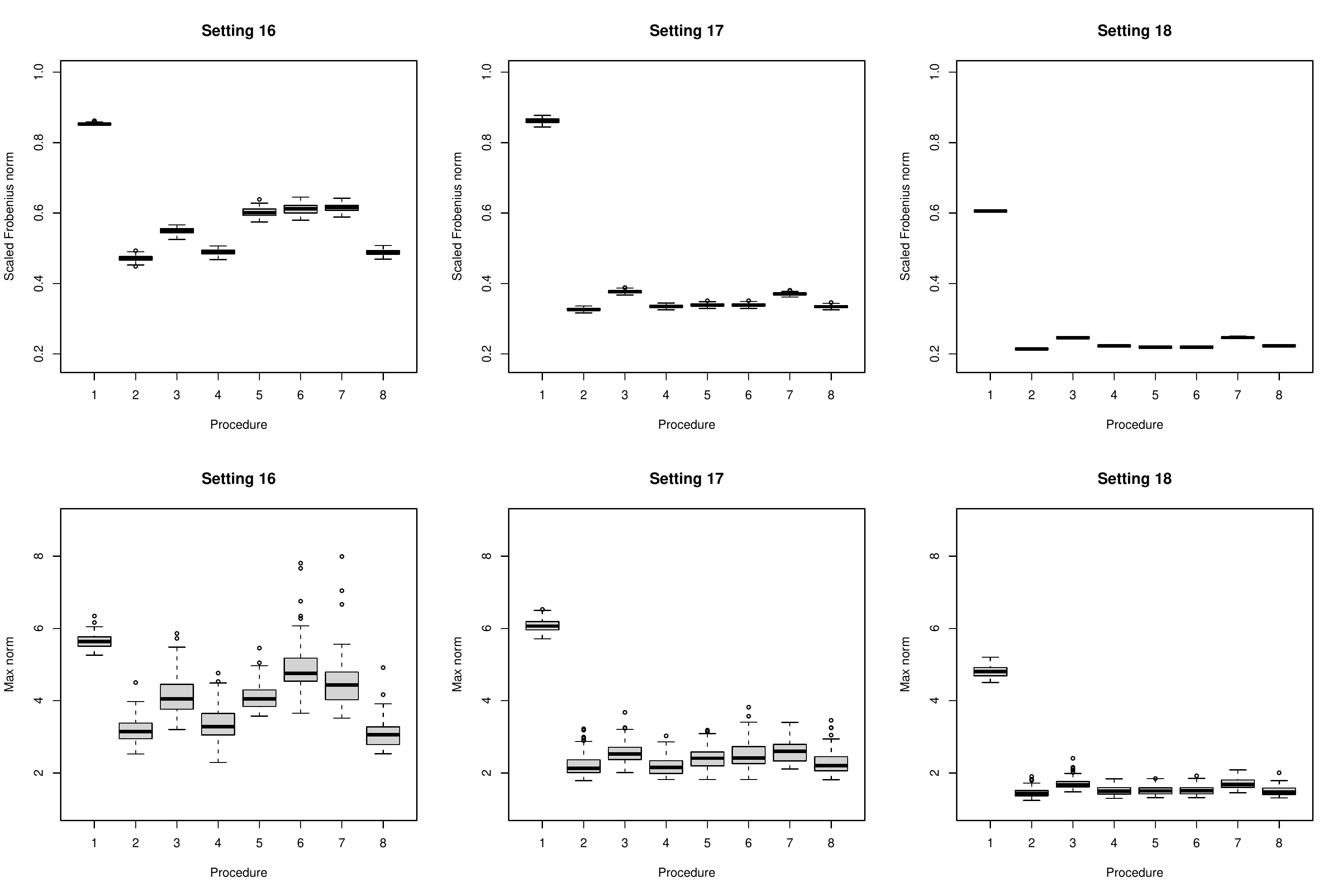}
    \caption{Results from Simulation Settings 16-18. The plots can be interpreted similarly as those in Figure~\ref{fig:sim1-3}.}
    \label{fig:sim16-18}
\end{figure}

\begin{figure}
    \centering
    \includegraphics[scale=0.4]{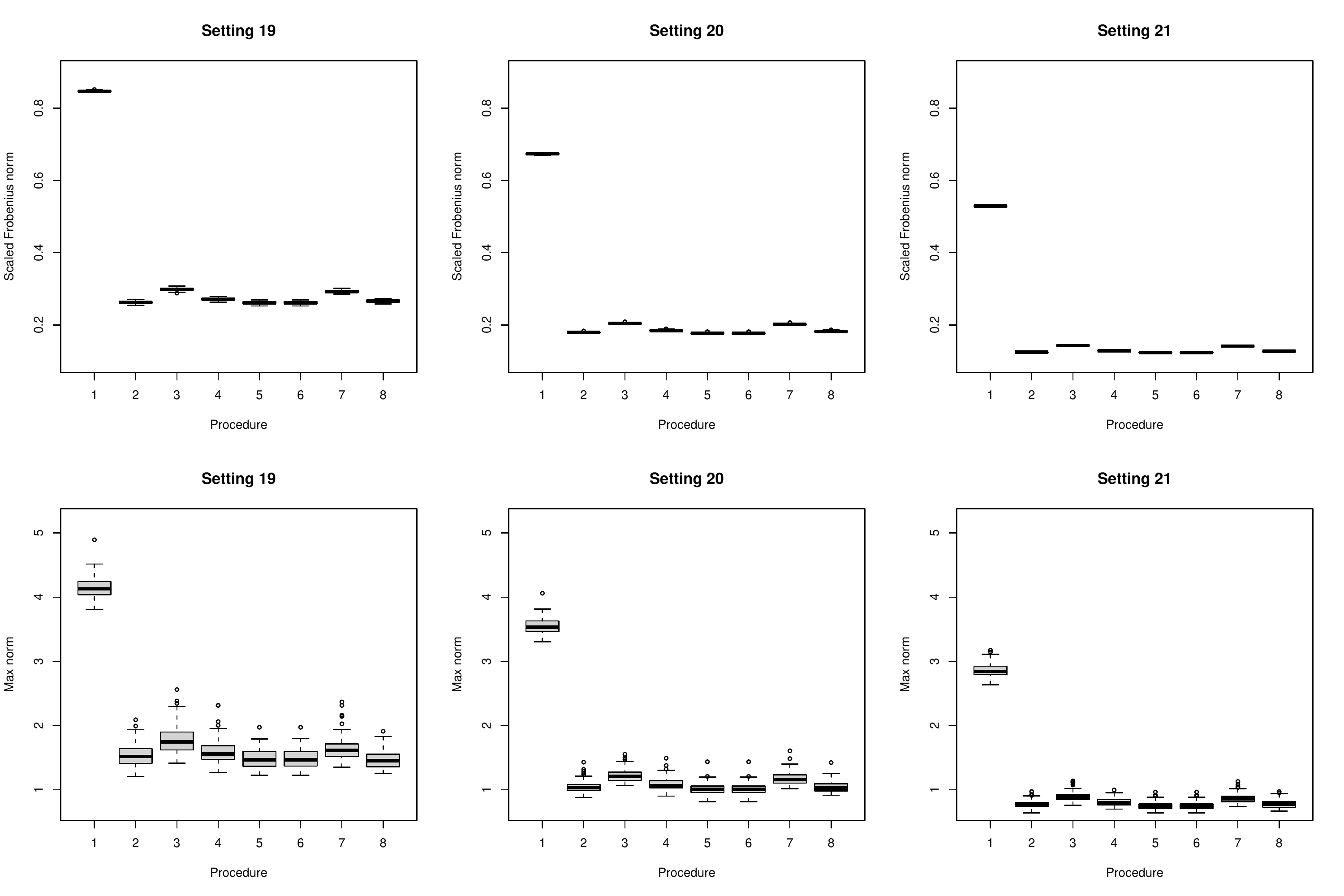}
    \caption{Results from Simulation Settings 19-21. The plots can be interpreted similarly as those in Figure~\ref{fig:sim1-3}.}
    \label{fig:sim19-21}
\end{figure}

\begin{figure}
    \centering
    \includegraphics[scale=0.4]{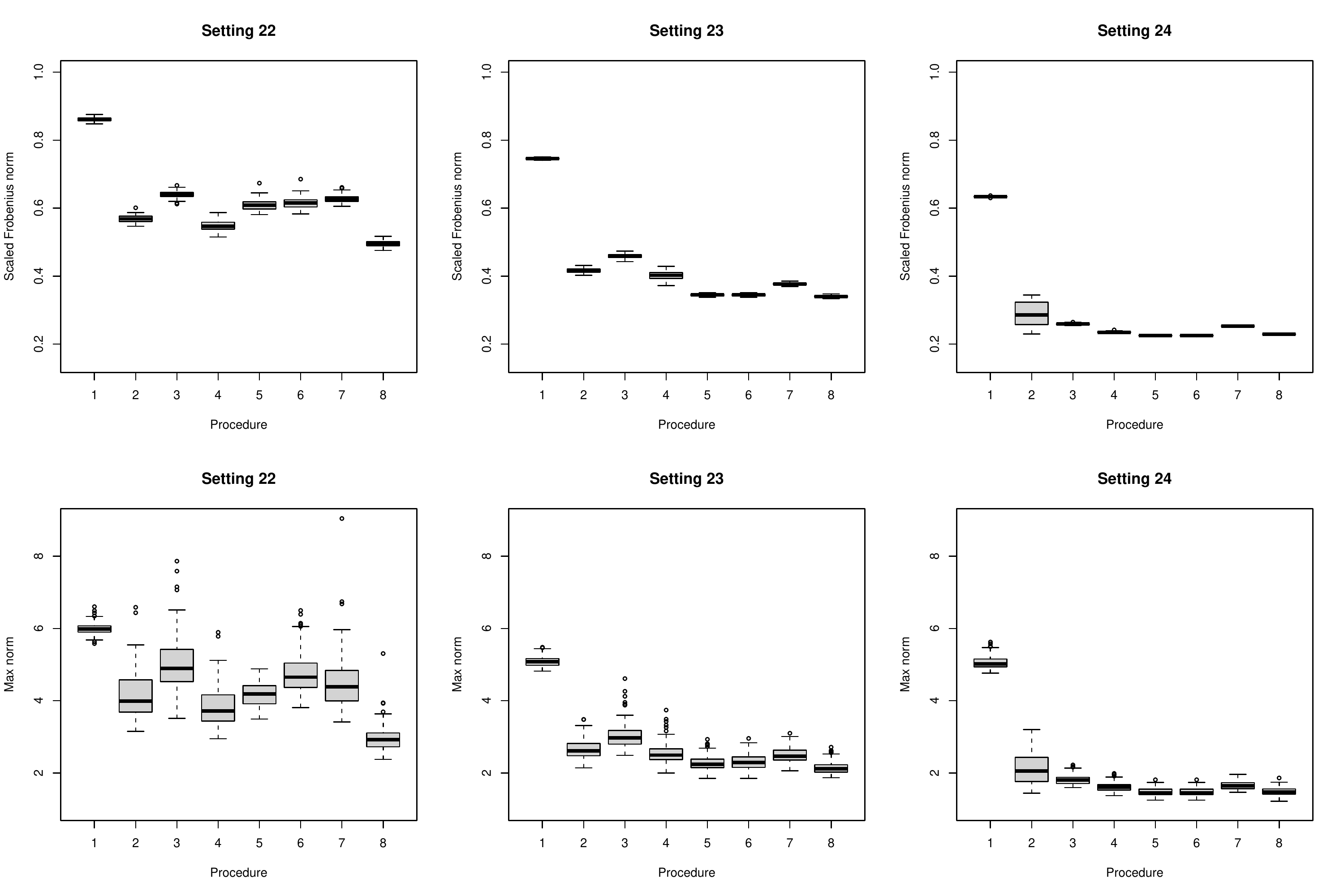}
    \caption{Results from Simulation Settings 22-24. The plots can be interpreted similarly as those in Figure~\ref{fig:sim1-3}.}
    \label{fig:sim22-24}
\end{figure}

\bibliographystyle{plainnat}
 \bibliography{ref}

\end{document}